\documentclass[11pt]{article}

\usepackage[top=50mm, bottom=50mm, left=50mm, right=50mm]{geometry}
\usepackage{lineno}
\usepackage{amssymb}
\usepackage{amsmath}
\usepackage{amsthm}
\usepackage{epsfig}
\usepackage{graphicx}
\usepackage{graphics}
\usepackage{float}
\usepackage{subfigure}
\usepackage{multirow}
\usepackage{color}
\usepackage{lineno}
\usepackage{fullpage}
\usepackage[normalem]{ulem} 
\usepackage{makeidx}
\usepackage{xspace}
\usepackage{wrapfig}
\makeindex

\newtheorem{lemma}{Lemma}

\definecolor{darkred}{rgb}{1, 0.1, 0.3}
\definecolor{darkblue}{rgb}{0.1, 0.1, 1}
\definecolor{darkgreen}{rgb}{0,0.6,0.5}

\newcommand {\mm}[1] {\ifmmode{#1}\else{\mbox{\(#1\)}}\fi}

\usepackage{amsmath}
\usepackage{amsthm}
\usepackage{amssymb}
\usepackage{amsmath,amsfonts,bm}

\def\eqref#1{equation~\ref{#1}}
\def\1{\bm{1}}

\DeclareMathAlphabet{\mathsfit}{\encodingdefault}{\sfdefault}{m}{sl}
\SetMathAlphabet{\mathsfit}{bold}{\encodingdefault}{\sfdefault}{bx}{n}

\newtheorem{definition}{Definition}

\newtheorem{proposition}{Proposition}
\newtheorem{assumption}{Assumption}

\usepackage{hyperref}
\hypersetup{colorlinks=true,allcolors=[rgb]{0,0.25,0.65}}
\usepackage{url}
\usepackage{graphicx}
\usepackage{caption}
\usepackage{multirow}
\begin{document}

\title{\textbf{Empirical or Invariant Risk Minimization?\\ A Sample Complexity Perspective}}
 
\author{
Kartik Ahuja\footnote{Currently at Mila - Quebec AI Institute. Work done while at IBM Research.} \and Jun Wang\footnote{Department of Computer Science,
Rensselaer Polytechnic Institute, New York.} \and Amit Dhurandhar\footnote{IBM Research, Thomas J. Watson Research Center, Yorktown Heights, New York.} \and  Karthikeyan Shanmugam \footnote{Google Research, India.} \and Kush R. Varshney $^{\ddagger}$
}

\date{}

%\begin{document}
\maketitle

\begin{abstract}
Recently, invariant risk minimization (IRM) was proposed as a promising solution to address out-of-distribution (OOD) generalization. However, it is unclear when IRM should be preferred over the widely-employed empirical risk minimization (ERM) framework. In this work, we analyze both these frameworks from the perspective of sample complexity, thus taking a firm step towards answering this important question. We find that depending on the type of data generation mechanism, the two approaches might have very different finite sample and asymptotic behavior. For example, in the covariate shift setting we see that the two approaches not only arrive at the same asymptotic solution, but also have similar finite sample behavior with no clear winner. For other distribution shifts such as those involving confounders or anti-causal variables, however, the two approaches arrive at different asymptotic solutions where IRM is guaranteed to be \emph{close} to the desired OOD solutions in the finite sample regime for polynomial generative models, while ERM is biased even asymptotically.  We further investigate how different factors --- the number of environments, complexity of the model, and IRM penalty weight ---  impact the sample complexity of IRM in relation to its distance from the OOD solutions. 
\end{abstract}

% \newpage

\section{Introduction}

A recent study shows that models trained to detect COVID-19 from chest radiographs rely on spurious factors such as the source of the data rather than the lung pathology \cite{degrave2020ai}. This is just one of many alarming examples of spurious correlations failing to hold outside a  specific training distribution.  In one commonly cited example, \cite{beery2018recognition} trained a convolutional neural network (CNN) to classify camels from cows. In the training data,  most pictures of the cows had green pastures, while most pictures of camels were in the desert. 
The CNN picked up the spurious correlation and associated green pastures with cows thus failing to classify cows on beaches.

Recently, \cite{arjovsky2019invariant} proposed a framework called invariant risk minimization (IRM) to address the problem of models inheriting spurious correlations. They showed that when data is gathered from multiple environments, one can learn to exploit invariant causal relationships, rather than relying on varying spurious relationships, thus learning robust predictors. More recent work suggests that empirical risk minimization (ERM) is still state-of-the-art on many problems requiring OOD generalization \cite{gulrajani2020search}. This gives rise to a fundamental question: when is IRM better than ERM (and vice versa)? In this work, we seek to answer this question through a systematic comparison of the sample complexity of the two approaches under different types of train and test distributional mismatches. 
% We compare the two approaches in terms of their sample complexity required to achieve a desired level of OOD generalization behavior.

The distribution shifts $\left(\mathbb{P}^{\mathsf{train}}(X,Y) \not= \mathbb{P}^{\mathsf{test}}(X,Y) \right)$ that we consider informally stated satisfy an \textit{invariance condition} -- there exists a representation $\Phi^{*}$ of the covariates such that  $\mathbb{P}^{\mathsf{train}}(Y|\Phi^{*}(X)) = \mathbb{P}^{\mathsf{test}}(Y|\Phi^{*}(X))=\mathbb{P}(Y|\Phi^{*}(X))$. A special case of this occurs when $\Phi^{*}$ is identity -- $\mathbb{P}^{\mathsf{train}}(X) \not= \mathbb{P}^{\mathsf{test}}(X)$ but  $\mathbb{P}^{\mathsf{train}}(Y|X) = \mathbb{P}^{\mathsf{test}}(Y|X)$ --  such a shift is known as a \textit{covariate-shift} \cite{gretton2009covariate}. In many other settings $\Phi^{*}$ may not be identity (denoted as $\mathsf{I}$),  examples include settings with confounders or anti-causal variables \cite{pearl2009causality} where covariates appear spuriously correlated with the label and $\mathbb{P}^{\mathsf{train}}(Y|X) \not= \mathbb{P}^{\mathsf{test}}(Y|X)$.   We use causal Bayesian networks to illustrate these shifts  in Figure \ref{fig1:dist_shift}. Suppose $X^e=[X_1^e,X_2^e]$ represents the image, where $X_1^e$ is the shape of the animal and $X_2^e$ is the background color, $Y^e$ is the label of the animal, and $e$ is the index of the environment/domain. In Figure \ref{fig1:dist_shift}a) $X_2^e$ is independent of $(Y^e, X_1^e)$, it represents the covariate shift case ($\Phi^{*} = \mathsf{I}$). In Figure \ref{fig1:dist_shift}b) $X_2^e$ is spuriously correlated with $Y^e$ through the confounder $\varepsilon^e$. In Figure \ref{fig1:dist_shift}c) $X_2^e$ is spuriously correlated with $Y^e$ as it is anti-causally related to $Y^e$. In both Figure \ref{fig1:dist_shift}b) and c) $\Phi^{*} \not= \mathsf{I}$; $\Phi^{*}$ is a block diagonal matrix that selects $X_1^{e}$.
% In the language of causality \cite{pearl2009causality}, if $X$ form the causal parents of $Y$, then this property $\mathbb{P}^{\mathsf{train}}(Y|X) = \mathbb{P}^{\mathsf{test}}(Y|X)$ always holds provided we do not intervene on $Y$. 

Our setup assumes we are given data from multiple training environments satisfying the invariance condition, i.e., $\mathbb{P}(Y|\Phi^{*}(X))$ is the same across all of them. Ideally, we want to learn and predict using  $\mathbb{E}[Y|\Phi^{*}(X)]$; this predictor has a desirable OOD behavior as we show later where we prove min-max optimality with respect to (w.r.t.) unseen test distributions satisfying the invariance condition. Our goal is to analyze and compare  ERM and IRM's ability to learn $\mathbb{E}[Y|\Phi^{*}(X)]$ from finite training data acquired from a fixed number of training environments. Our  analysis has two parts.

\begin{figure}
\begin{center}
% \centering
%\begin{minipage}{.5\textwidth} % \centering
%   \includegraphics[trim=1in 0.5in 1in 0.8in, width=2.2in]{Gmodels-shifts.png}
\includegraphics[width=4in]{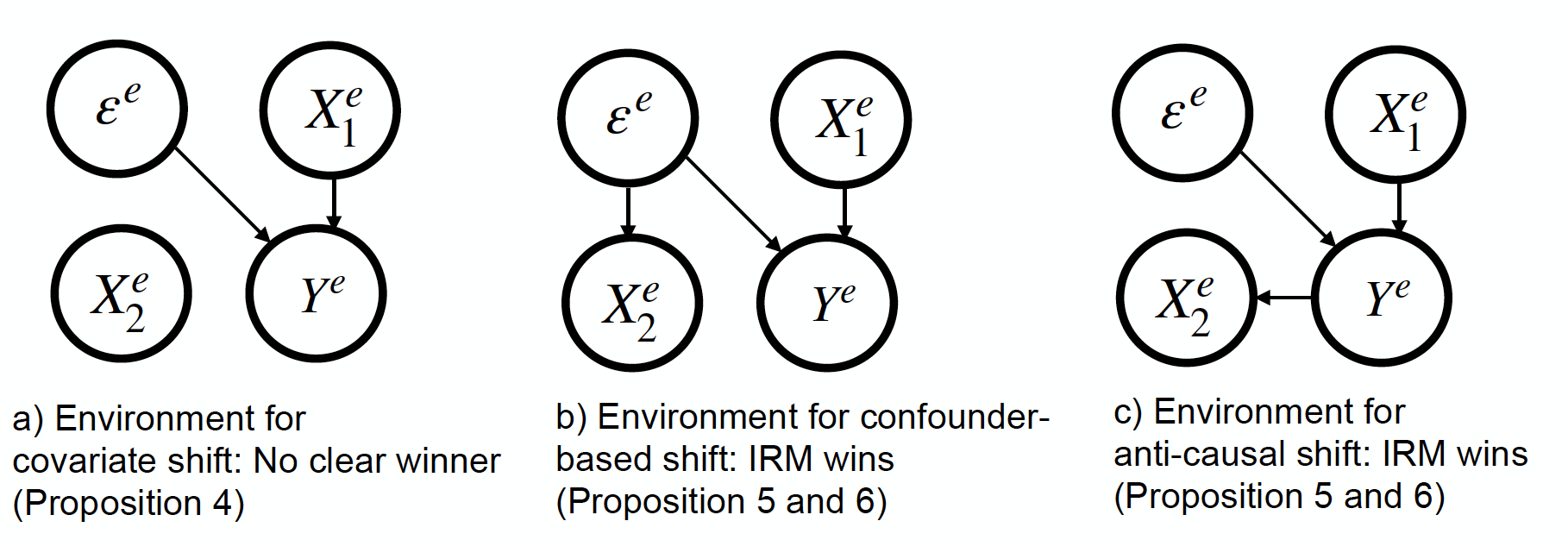}
%   \vspace{-0.7em}
  \caption{\small{Causal Bayesian networks for different distribution shifts.}}
   \label{fig1:dist_shift}
\end{center}
\end{figure}
%\begin{itemize}
    \textbf{1) Covariate shift case ($\Phi^{*} = \mathsf{I}$):}  ERM and IRM achieve the same asymptotic solution  $\mathbb{E}[Y|X]$. We prove (Proposition \ref{prop4: emr_irm_real}) that the sample complexity for both the methods is similar thus there is no clear winner between the two in the finite sample regime.  For the setup in Figure \ref{fig1:dist_shift}a), both ERM and IRM learn a model that only uses $X_1^e$.
    
    \textbf{2) Confounder/Anti-causal variable case ($\Phi^{*} \not= \mathsf{I}$):} We consider a family of structural equation models (linear and polynomial) that may contain confounders and/or anti-causal variables. For the class of models we consider, the asymptotic solution of ERM is biased and not equal to the desired $\mathbb{E}[Y|\Phi^{*}(X)]$. We prove that IRM can learn a solution that is within $\mathcal{O}(\sqrt{\epsilon})$ distance from $\mathbb{E}[Y|\Phi^{*}(X)]$ with a sample complexity that increases as $\mathcal{O}(\frac{1}{\epsilon^2})$ and increases polynomially in the complexity of the model class (Proposition \ref{prop5: irm_cfd_chd}, \ref{prop6:pol_inf});  $\epsilon$ (defined later) is the slack in IRM constraints.    For the setup in Figure \ref{fig1:dist_shift}b) and c), IRM gets close to only using $X_1^e$, while ERM even with infinite data (Proposition 17 in the supplement) continues to use $X_2^e$.  We summarize the results in Table \ref{table1}.
%\end{itemize}

\cite{arjovsky2019invariant} proposed the colored MNIST (CMNIST) dataset; comparisons on it showed how ERM-based models exploit spurious factors (background color). The CMNIST dataset relied on anti-causal variables. Many supervised learning datasets may not contain anti-causal variables (e.g. human labeled images). Therefore, we propose and analyze three new variants of CMNIST in addition to the original one that map to different real-world settings: i) covariate shift based CMNIST (CS-CMNIST): relies on selection bias to induce spurious correlations, ii) confounded CMNIST (CF-CMNIST): relies on confounders to induce spurious correlations, iii) anti-causal CMNIST (AC-CMNIST): this is the original CMNIST proposed by \cite{arjovsky2019invariant}, and iv) anti-causal and confounded (hybrid) CMNIST (HB-CMNIST): relies on confounders and anti-causal variables to induce spurious correlations. On the latter three datasets, which belong to the $\Phi^{*} \not= \mathsf{I}$ class described above, IRM has a much better OOD behavior than ERM, which performs poorly regardless of the data size. However,  IRM and ERM have a similar performance on CS-CMNIST  with no clear winner. These results are consistent with our theory and are also validated in regression experiments. 

\begin{table}[t]
\caption{Summary of (empirical) IRM vs.\ ERM for finite hypothesis class $\mathcal{H}_{\Phi}$.  $\epsilon$: slack in IRM constraints, $\nu$: approximation w.r.t optimal risk, $\delta$: failure probability, $\mathcal{E}_{tr}$: set of training environments, $n$: data dimension, $p$: degree of the generative polynomial, $L,L'$: bound on loss \& its gradients.}
\label{table1}
\begin{center}
\begin{tabular}{||c|c|c|c||}
\hline 
  \multicolumn{1}{||c}{\bf Assumptions}&\multicolumn{1}{|c}{\bf Method} &\multicolumn{1}{|c|}{\bf Sample complexity} &\multicolumn{1}{c|}{\bf OOD} 
\\ \hline \hline

         \multirow{4}{16em}{\emph{Covariate shift case:} $\mathbb{E}^{e}[Y^{e}|X^{e}]$ is invariant
         (Proposition \ref{prop4: emr_irm_real})} & ERM & $\frac{8L^2}{\nu^2} \log\big(\frac{2|\mathcal{H}_{\Phi}|}{\delta}\big)$  & Yes    \\
 && &   \\
          &  IRM   & $\max\big\{ \frac{8L^2}{\nu^2} 
 \log\big(\frac{4|\mathcal{H}_{\Phi}|}{\delta}\big), \frac{16L'^4}{\epsilon^2}\log\big(\frac{2}{\delta}\big)\big\} $&  Yes \\
  && &   \\
 \hline
  \multirow{4}{16em}{\emph{Confounder/Anti-causal variable case:} $\mathbb{E}^{e}[Y^{e}|\Phi^{*}(X^{e})]$ is invariant, \;\; 
 Linear, Polynomial models,
 $|\mathcal{E}_{tr}|= \mathcal{O}(n^{p})$
 (Proposition \ref{prop5: irm_cfd_chd}, \ref{prop6:pol_inf},17)} & ERM  &  $\frac{8L^2}{\nu^2} \log\big(\frac{2|\mathcal{H}_{\Phi}|}{\delta}\big)$ & No  \\
 &  & &  \\ 
        &   IRM &  $\frac{16L'^{4}}{\epsilon^2}\log\big(\frac{2|\mathcal{H}_{\Phi}|}{\delta}\big)$ & Yes  \\
         && &   \\
 \hline
\end{tabular}
\end{center}
\end{table}

\section{Related Works}

\textbf{IRM based works.} Following the original work  IRM from \cite{arjovsky2019invariant}, there have been several interesting works. \cite{teney2020unshuffling, krueger2020out, ahuja2020invariant, chang2020invariant} --- is an incomplete representative list that build new methods inpired from IRM to address the OOD generalization problem.  \cite{arjovsky2019invariant}  prove OOD guarantees for linear models with access to infinite data from finite environments.  We generalize these results in several ways. We provide a first finite sample analysis of IRM. We characterize the impact of hypothesis class complexity, number of environments, weight of IRM penalty on the sample complexity and its distance from the OOD solution for linear and polynomial models.

\textbf{Theory of domain generalization and domain adaption.} Following the seminal works \cite{ben2007analysis,ben2010theory}, there have been many interesting works over the years ---  \cite{muandet2013domain, ajakan2014domain,zhao2019learning, albuquerque2019adversarial} \linebreak \cite{li2017deeper, piratla2020efficient, matsuura2020domain, deng2020representation, david2010impossibility, pagnoni2018pac} is an incomplete representative list (see \cite{redko2019advances} for further references) --- that build the theory of domain adaptation and  generalization and construct new methods based on it. While many of these works develop bounds on loss over the target domain using train data and unlabeled target data, some \cite{ben2012hardness,david2010impossibility, pagnoni2018pac} analyze the finite sample (PAC) guarantees for domain adaptation under covariate shifts. These works  \cite{ben2012hardness,david2010impossibility, pagnoni2018pac} access unlabeled data from a target domain, which we do not. Instead, we have  data from multiple training domains (as in domain generalization). In these works, the guarantees are w.r.t.\ a specific target domain, while we provide (for linear and polynomial models)  worst-case guarantees w.r.t.\ all the unseen domains  satisfying the invariance condition. Also, we consider a larger family of distribution shifts including covariate shifts. The above two categories are not exhaustive -- e.g., there are some recent works that characterize how some inductive biases favor extrapolation \cite{xu2021how} and can be better for OOD generalization.

\section{Sample Complexity of Invariant Risk Minimization}

\subsection{Invariant Risk Minimization} 

\label{secn: IRM} We start with some background on IRM \cite{arjovsky2019invariant}.
 Consider a dataset $D = \{D^{e}\}_{e\in \mathcal{E}_{tr}}$, which is a collection of datasets $D^{e}=\{(x_i^e, y_i^e,e)\}_{i=1}^{n_e}\}$ obtained from a set of training environments $\mathcal{E}_{tr}$, where $e$ is the index of the environment, $i$ is the index of the data point in the environment, $n_e$ is the number of points from environment, $x_i^{e} \in \mathcal{X}\subseteq \mathbb{R}^{n}$ is the feature value  and $y_i^{e}\in \mathcal{Y}\subseteq \mathbb{R}$ is the corresponding label. Define a  probability distribution $\{\pi^{e}\}_{e\in \mathcal{E}_{tr}}$, $\pi^{e}$ is the probability that a training data point is from environment $e$. Define a probability distribution of points conditional on environment $e$ as $\mathbb{P}^{e}$, $(X^{e},Y^{e}) \sim \mathbb{P}^{e}$. Define the joint distribution $\bar{\mathbb{P}}$, $(X^{e},Y^{e},e) \sim \bar{\mathbb{P}}, d\bar{\mathbb{P}}(X^{e},Y^{e},e) = \pi^{e} d\mathbb{P}^e(X^{e}, Y^{e})$. $D$ is a collection of i.i.d. samples from $\bar{\mathbb{P}}$. Define a predictor $f:\mathcal{X} \rightarrow \mathbb{R}$ and  the space $\mathcal{F}$  of all the possible maps from $\mathcal{X} \rightarrow \mathbb{R}$.  Define the risk achieved by $f$ in environment $e$ as $R^e(f) = \mathbb{E}^{e}\big[\ell\big(f(X^e), Y^e\big)\big]$, where $\ell$ is the loss, $f(X^e)$ is the predicted value, $Y^e$ is the corresponding label and $\mathbb{E}^{e}$ is the expectation conditional on environment $e$. The overall expected risk  across  the training environments is  $R(f) = \sum_{e\in \mathcal{E}_{tr}}\pi^{e}R^{e}(f)$. We are interested in two settings: regression (square loss) and binary-classification (cross-entropy loss). In the main body, our focus is regression (square loss) and we mention wherever the results extend to binary-classification (cross-entropy). We discuss these extensions in the supplement.

\textbf{OOD generalization problem.} We want to construct a predictor $f$ that performs well across many unseen environments $\mathcal{E}_{all}$, where $ \mathcal{E}_{all} \supseteq \mathcal{E}_{tr}$.  For $o\in \mathcal{E}_{all} \backslash \mathcal{E}_{tr}$, the distribution $\mathbb{P}^{o}$ can be very different from the train environments.  Next we  state the OOD problem. 
\begin{equation}
    \min_{f\in \mathcal{F}} \max_{e \in \mathcal{E}_{all}} R^{e}(f)
    \label{eqn1: ood}
\end{equation}

 The above problem is very challenging to solve since we only have access to data from training environments $\mathcal{E}_{tr}$ but are required to find the robust solution over all environments $\mathcal{E}_{all}$. Next, we make assumptions on $\mathcal{E}_{all}$ and characterize the optimal solution to \eqref{eqn1: ood}.

% The objective is to show how the above problem cannot be solved by existing approaches and why IRM is needed.

\begin{assumption} \textbf{Invariance condition.}
\label{assm1:ood_cond_envmt}
There exists a representation $\Phi^{*}$ that transforms  $X^{e}$ to $Z^{e}=\Phi^{*}(X^{e})$ and $\forall e,o \in \mathcal{E}_{all}, \forall z \in \Phi^{*}(\mathcal{X})$ satisfies $ \mathbb{E}^{e}[Y^{e}|Z^{e}=z]=\mathbb{E}^{o}[Y^{o}|Z^{o}=z]$.
Also, $\forall e\in \mathcal{E}_{all}, \forall z\in \Phi^{*}(\mathcal{X})$, $\mathsf{Var}^{e}[Y^{e}|Z^{e}=z]=  \xi^2$, where $\mathsf{Var}^{e}$ is the conditional variance.
\end{assumption}

 The above assumption is inspired from causality \cite{pearl2009causality}.  $\Phi^{*}$ acts as the causal feature extractor and from the definition of causal features, it follows that  $\mathbb{E}^{e}[Y^{e}|Z^{e}=z]$ does not vary across environments. When a human labels a cow she uses $\Phi^{*}$ to extract causal features from the pixels to identify cow while ignoring the background.
The first part of the above assumption encompasses a large class of distribution shifts including standard covariate shifts \cite{gretton2009covariate}. Covariate shift assumes $\forall e,o\in\mathcal{E}_{all},\forall x \in \mathcal{X}$, $\mathbb{P}(Y^{e}|X^{e}=x)$ and $ \mathbb{P}(Y^{o}|X^{o}=x)$ are equal thus implying $\mathbb{E}^{e}[Y^{e}|X^{e}=x] = \mathbb{E}^{o}[Y^{o}|X^{o}=x]$. Therefore, for covariate shifts, $\Phi^{*}$ is identity in Assumption \ref{assm1:ood_cond_envmt}. A simple instance illustrating  Assumption \ref{assm1:ood_cond_envmt} with $\Phi^{*}=\mathsf{I}$ is when $Y^{e}\leftarrow g(X^{e})+ \varepsilon^{e}$, where $\mathbb{E}^{e}[\varepsilon^{e}]=0, \mathbb{E}^{e}[(\varepsilon^{e})^2]=\sigma^2$, $\varepsilon^{e} \perp X^{e}$.   Using Assumption \ref{assm1:ood_cond_envmt}, we define the invariant map $m:\Phi^{*}(\mathcal{X})\rightarrow \mathbb{R}$ as follows 
\begin{equation} 
\forall z \in \Phi^{*}(\mathcal{X}), \; m(z)=  \mathbb{E}^{e}[Y^{e}|Z^{e}=z],\; \text{where}\; Z^{e}=\Phi^{*}(X^{e})
\label{m_inv_model}
\end{equation}

\begin{assumption}
\label{assm2:ood_cond_envmt}
\textbf{Existence of an environment where the invariant representation is sufficient.}  $\exists$ an environment $e \in \mathcal{E}_{all}$ such that $Y^{e} \perp X^{e} | Z^{e}$ 
\end{assumption}

Assumption \ref{assm2:ood_cond_envmt} states there exists an environment where the information that $X^{e}$ has about $Y^{e}$  is also contained in $Z^{e}$. Define a composition $m\circ \Phi^{*}$, $\forall x\in \mathcal{X}$,  $m\circ \Phi^{*}(x) =\mathbb{E}^{e}[Y^{e}|Z^{e}=\Phi^{*}(x)]$.

\begin{proposition}
\label{prop1: ood}
 If $\ell$ is the square loss, and Assumptions \ref{assm1:ood_cond_envmt} and \ref{assm2:ood_cond_envmt} hold, then  $m\circ \Phi^{*}$  solves the OOD problem (\eqref{eqn1: ood}).
\end{proposition}
The proofs of all the propositions are in the supplement. A similar result holds for the cross-entropy loss (discussion in supplement). For the rest of the paper, we focus on learning $m\circ \Phi^{*}$ as it solves the OOD problem. For covariate shifts $\Phi^{*}=\mathsf{I}$, $m(x)=\mathbb{E}^{e}[Y^{e}|X^{e}=x]$ is the OOD solution. In \cite{arjovsky2019invariant}, a proof connecting $m\circ \Phi^{*}$ and OOD was not stated. Recently, in \cite{koyama2020out}, a result similar to Proposition \ref{prop1: ood} was shown but with a few differences.  The authors assume  conditional probabilities are invariant  unlike our assumption that only requires conditional expectations and variances to be invariant. However, their result applies to more losses. 
 $m\circ \Phi^{*}$ is the target we want to learn. \cite{arjovsky2019invariant} proposed IRM since standard min-max optimization over the training environments $\mathcal{E}_{tr}$  and  ERM fail to learn $m\circ \Phi^{*}$ in many cases. The authors in \cite{arjovsky2019invariant} identify a crucial property of $m\circ \Phi^{*}$ and use it to define an object called invariant predictor that we define next.
% If $\exists $ an environment $e$ in $\mathcal{E}_{tr}$ satisfying the Assumption 2, then solving a min-max optimization over the training environments $    \min_{f \in \mathcal{H}_{w\circ \Phi}} \max_{e \in \mathcal{E}_{tr}} R^{e}(f)$ is sufficient to solve \eqref{eqn1: ood}. However, in general we cannot assume to have access to data from such an environment. 

\textbf{Invariant predictor and IRM optimization.}  Define a representation map $\Phi:\mathcal{X}\rightarrow \mathcal{Z}$ from feature space to representation space $\mathcal{Z} \subseteq \mathbb{R}^{q}$.  Define a classifier map, $w:\mathcal{Z}\rightarrow \mathbb{R}$ from representation space to real values. Define $\mathcal{H}_{\Phi}$ and $\mathcal{H}_{w}$ as the  spaces of representations and classifiers respectively. A data representation $\Phi$ elicits an invariant predictor $w\circ \Phi$ across environments $e \in \mathcal{E}_{tr}$ if there is a classifier $w$  that achieves the minimum risk simultaneously for all the environments, i.e.,
$\forall e \in \mathcal{E}_{tr},w \in \arg\min_{\bar{w} \in \mathcal{H}_{w}} R^{e}(\bar{w}\circ \Phi)$.   Observe that if we we transform the data with representation $\Phi^{*}$ then $m$ will achieve the minimum risk simultaneously in all the environments. If $\Phi^{*}\in \mathcal{H}_{\Phi}$ and $m\in \mathcal{H}_{w}$, then $m\circ \Phi^{*}$ is an invariant predictor.  IRM selects the invariant predictor with least sum risk across environments (results presented later can be adapted if invariant predictor was selected based on the worst-case risk over the environments as well)  as follows:
\begin{equation}
    \begin{split}
        & \min_{\Phi \in \mathcal{H}_{\Phi},w  \in \mathcal{H}_{w}}R(w\circ \Phi) = \sum_{e \in \mathcal{E}_{tr}}\pi^{e} R^{e}(w\circ \Phi) \\ 
        & \text{s.t.}\;w \in \arg\min_{\bar{w} \in \mathcal{H}_{w}} R^{e}(\bar{w}\circ \Phi),\;\forall e \in \mathcal{E}_{tr}
    \end{split}
    \label{eqn: IRM}
\end{equation}

 From the above discussion we know $m\circ \Phi^{*}$ is a feasible solution to \eqref{eqn: IRM}. It is also the ideal solution we want IRM to find since it solves \eqref{eqn1: ood}. Later in Propositions \ref{prop4: emr_irm_real}, \ref{prop5: irm_cfd_chd}, and \ref{prop6:pol_inf}, we show that IRM actually solves equation \eqref{eqn1: ood}. For the setups in Proposition   \ref{prop5: irm_cfd_chd}, and \ref{prop6:pol_inf}, conventional ERM based approaches fail thus justifying the need for above formulation.

\subsection{Sample Complexity of Gradient Constraint Formulation of IRM} 
In \cite{arjovsky2019invariant}, a gradient constrained alternate (derived below in \eqref{eqn: IRM_grad_cons}) to \eqref{eqn: IRM} was proposed, which focuses on linear and scalar classifiers ($\mathcal{Z}=\mathbb{R}$, $\Phi:\mathcal{X}\rightarrow \mathbb{R}$, $\mathcal{H}_w =\mathbb{R}$). In this case, the composite predictor $w \circ \Phi$ is a multiplication of $w$ and $\Phi$ written as $w\cdot \Phi$. (For binary-classification predictor's output $w\cdot\Phi(x)$ represents logits.) From the definition of invariant predictors and  $\mathcal{H}_w =\mathbb{R}$ it follows that if $\forall \bar{w}\in \mathbb{R}, \; R^{e}(1\cdot \Phi)\leq R^{e}(\bar{w}\cdot \Phi)$, then $\Phi$ is an invariant predictor. For square and cross-entropy losses, $R^{e}(w\cdot \Phi)$ is convex in $w$.   Therefore, a gradient constraint $\|\nabla_{w| w=1.0} R^{e}(w \cdot \Phi) \| =0$ is equivalent to the condition that $\forall \bar{w}\in \mathbb{R}, \; R^{e}(1\cdot \Phi)\leq R^{e}(\bar{w}\cdot \Phi)$, which implies $\Phi$ is an invariant predictor. Recall that IRM aims is to search among invariant predictors and find one that minimizes the risk. We state this as a gradient constrained optimization as follows

\begin{equation}
    \begin{split}
        & \min_{\Phi \in \mathcal{H}_{\Phi}}R(\Phi)\\
        & \text{s.t.}\;\big\|\nabla_{w| w=1.0} R^{e}(w \cdot \Phi) \big\| = 0,\;\forall e \in \mathcal{E}_{tr}
    \end{split}
    \label{eqn: IRM_grad_cons}
\end{equation}

 We propose an  $\epsilon$ approximation of the above with $\epsilon$  slack in the constraint. Define $R^{'}(\Phi) = \sum_{e\in \mathcal{E}_{tr}} \pi^{e} \|\nabla_{w|w=1.0} R^{e}(w\cdot \Phi)\|^2$ and a set  $\mathcal{S}^{\mathsf{IV}}(\epsilon) =\{\Phi\;|\;R^{'}(\Phi) \leq \epsilon, \Phi\in \mathcal{H}_{\Phi}\}$. Note that $R^{'}$ is very similar to the penalty defined in \cite{arjovsky2019invariant}. The $\epsilon$ approximation of \eqref{eqn: IRM_grad_cons} is

\begin{equation}
 \min_{\Phi \in \mathcal{S}^{\mathsf{IV}}(\epsilon)} R(\Phi) 
    \label{eqn: IRM_cons_swap}
\end{equation}

% In the above formulation $\pi^{e}>0$ and $\sum_{e}\pi^{e} = 1$. 
If $\epsilon=0$, then  \eqref{eqn: IRM_grad_cons} and \eqref{eqn: IRM_cons_swap} are equivalent.  In all the optimizations so far, the expectations are computed w.r.t.\ the distributions $\mathbb{P}^{e}$, which are unknown. Therefore, we develop an empirical version of  \eqref{eqn: IRM_cons_swap} below (in \eqref{eqn: EIRM_cons_swap}) and call it empirical IRM (EIRM). We replace $R$ and $R^{'}$ with empirical estimators $\hat{R}$ and $\hat{R}^{'}$ respectively.  For $R$ we use a simple plugin estimator (sample mean of loss across all the samples in $D$). For $R^{'}$ we construct a new estimator that enables the use of standard concentration inequalities. Define a set  $\hat{\mathcal{S}}^{\mathsf{IV}}(\epsilon) =\{\Phi\;|\;\hat{R}^{'}(\Phi) \leq \epsilon, \Phi\in \mathcal{H}_{\Phi}\}$. 

\begin{equation}
  \min_{\Phi \in \hat{\mathcal{S}}^{\mathsf{IV}}(\epsilon)} \hat{R}(\Phi) 
    \label{eqn: EIRM_cons_swap}
\end{equation} 

If we replaced  $\hat{\mathcal{S}}^{\mathsf{IV}}(\epsilon)$ with $\mathcal{H}_{\Phi}$ in \eqref{eqn: EIRM_cons_swap}, then we get the standard ERM. ERM aims to solve  $\min_{\Phi \in \mathcal{H}_{\Phi}} \hat{R}(\Phi)$. The sample complexity analysis of ERM aims to understand the distance between the empirical solutions and the expected solutions as a function of the number of samples. Similarly, we seek to understand the relationship between solutions of \eqref{eqn: EIRM_cons_swap} and \eqref{eqn: IRM_cons_swap}.

\begin{assumption}
\label{assm3: bounded loss and gradient} \textbf{Bounded loss and bounded gradient of the loss.} $\exists\;L<\infty$, $L'<\infty$ such that 
   $\forall \Phi \in \mathcal{H}_{\Phi}, \forall x \in \mathcal{X}, \forall y \in \mathcal{Y}, |\ell(\Phi(x), y)| \leq L, |\frac{\partial \ell(w \cdot \Phi(x),y )}{\partial w}|_{w=1.0}| \leq L^{'}$. 
\end{assumption}

If every $\Phi$ in the hypothesis class $\mathcal{H}_{\Phi}$ is bounded by $M$ and the label space $\mathcal{Y}$ is bounded, then for both square  and cross-entropy loss,  $\ell(\Phi(\cdot), \cdot)$ and $\frac{\partial \ell(w \cdot\Phi(\cdot), \cdot)}{\partial w}|_{w=1.0}$ are bounded.  Define $\kappa = \min_{\Phi \in \mathcal{H}_{\Phi}} |R^{'}(\Phi)-\epsilon|$; $\kappa$ measures how close any penalty can get to the boundary $\epsilon$. $\kappa$ quantifies how good the finite sample approximation $\hat{R}^{'}$ need to be in order to get $\hat{\mathcal{S}}^{\mathsf{IV}}(\epsilon) = \mathcal{S}^{\mathsf{IV}}(\epsilon)$.
Define $\nu$ to quantify the approximation w.r.t.\ optimal risk.

\begin{proposition} \label{prop2: scomp_irm}
For every $\nu>0$ and $\delta\in (0,1)$, if  $\mathcal{H}_{\Phi}$ is a finite hypothesis class, Assumption \ref{assm3: bounded loss and gradient} holds, $\kappa >0$, and if the  number of samples $|D|$ is greater than  $\max\big\{\frac{16L'^{4}}{\kappa^2}, \frac{8L^2}{\nu^2}\big\}\log\big(\frac{4|\mathcal{H}_{\Phi}|}{\delta}\big)$, then with a probability at least $1-\delta$,  every solution $\hat{\Phi}$ to  EIRM (\eqref{eqn: EIRM_cons_swap})  is a $\nu$ approximation of IRM, i.e.  $\hat{\Phi}\in \mathcal{S}^{\mathsf{IV}}(\epsilon)$, $R(\Phi^{*}) \leq R(\hat{\Phi}) \leq R(\Phi^{*}) + \nu$, where $\Phi^{*}$ is a solution to IRM (\eqref{eqn: IRM_cons_swap}).
\end{proposition}

 \textbf{Proof Sketch.} The standard analysis in learning theory on ERM or regularized/constrained ERM typically relies on linearly separable loss functions. In such cases, we can use standard plug-in estimators and analyze their behavior using  concentration inequalities. In our setting, $R^{'}$ is a weighted sum of squares of expectation and thus it is not linearly separable. We develop a new way of expressing $R^{'}$ that allows us to make it linearly separable.  Next, in order to ensure $R(\Phi^{*}) \leq R(\hat{\Phi}) \leq R(\Phi^{*}) + \nu$ we first need to guarantee that the set of invariant predictors is exactly recovered, i.e., $\hat{\mathcal{S}}^{\mathsf{IV}}(\epsilon) = \mathcal{S}^{\mathsf{IV}}(\epsilon)$ (exact recovery is typically not required in existing constrained analysis such as \cite{woodworth2017learning} \cite{agarwal2018reductions}).  We show that if the number of samples grow as $\frac{1}{\kappa^2}$ even the closest points on either side of the boundary of the set $\mathcal{S}^{\mathsf{IV}}(\epsilon)$ are correctly discriminated, which guarantees exact recovery of $\mathcal{S}^{\mathsf{IV}}(\epsilon)$. Once the exact set is recovered, beyond this  we use standard learning theory tools to ensure  $R(\Phi^{*}) \leq R(\hat{\Phi}) \leq R(\Phi^{*}) + \nu$.

The above result holds for both square and cross-entropy loss. For ease of exposition, we use the standard setting of finite hypothesis class and extend all the results to infinite hypothesis classes in the supplement (summary of insights from the extension are in Section \ref{secn: non-realizable}). Next, we state a standard result on ERM's sample complexity. Define a $\Phi^{+}$ such that $ \Phi^{+}\in \arg\min_{\Phi \in \mathcal{H}_{\Phi}} R(\Phi)$

\begin{proposition}\cite{shalev2014understanding} \label{prop3: scomp_erm}
For every $\nu>0$ and $\delta\in (0,1)$, if $\mathcal{H}_{\Phi}$ is a finite hypothesis class, Assumption \ref{assm3: bounded loss and gradient} holds, and if the  number of samples  $|D|$ is greater than $ \frac{8L^2}{\nu^2}\log\big(\frac{2|\mathcal{H}_{ \Phi}|}{\delta}\big)$,  then   with  a probability at least $1-\delta$, every solution $\Phi^{\dagger}$ to ERM  is an $\nu$ approximation of expected risk minimization, i.e., 
 $R(\Phi^{+}) \leq R(\Phi^{\dagger}) \leq  R(\Phi^{+}) + \nu$. 
\end{proposition}

\textbf{Proposition \ref{prop2: scomp_irm} vs.\ \ref{prop3: scomp_erm}}   Since $\kappa\leq \epsilon$, the sample complexity of EIRM grows at least as $\mathcal{O}(\max\{\frac{1}{\epsilon^2}, \frac{1}{\nu^2}\})$. Let us look at the two terms inside $\max$- i) $\frac{1}{\nu^2}$ growth term is similar to ERM, it  ensures $\nu$ approximate optimality in the overall risk $R$, ii) $\frac{1}{\epsilon^2}$  growth ensures the IRM penalty $R^{'}$ is less than $\epsilon$. A direct comparison of sample complexities in Propositions \ref{prop2: scomp_irm} and \ref{prop3: scomp_erm}  suggests that the sample complexity of EIRM is higher than ERM, which is not the complete picture. The two approaches may not converge to the same solutions and IRM may converge to a solution with better OOD behavior than one achieved by ERM.  Therefore, a fair comparison is only possible when we also study the OOD properties of the solutions achieved by the two approaches, which is the subject of the next section.

\subsection{OOD Performance: ERM vs.\ IRM}
We divide the comparisons based on distributional shift assumptions that decide whether ERM and IRM arrive at the same asymptotic solutions or not. 

\subsubsection{ Covariate shift} \label{secn: realizable}

\begin{assumption} \textbf{Invariance w.r.t.\ all the features.}
\label{assm4: invariant_cexp}
 $\forall e,o\in \mathcal{E}_{all}$ and  $\forall x\in \mathcal{X}$,
$\mathbb{E}[Y^{e}|X^{e}=x]  =\mathbb{E}[Y^{o}|X^{o}=x]$.  $\forall e\in \mathcal{E}_{all}$, $X^{e}\sim \mathbb{P}^{e}_{X^{e}}$ and  the support of $\mathbb{P}^{e}_{X^{e}}$ is equal to  $\mathcal{X}$.
\end{assumption}

 As stated earlier, the first part of the above assumption follows from standard covariate shift assumptions \cite{gretton2009covariate} and is a special case of the first part of the Assumption \ref{assm1:ood_cond_envmt} with $\Phi^{*}$ set to $\mathsf{I}$. If $\Phi^{*}=\mathsf{I}$, then $m$ (\eqref{m_inv_model}), which simplifies to $m(x) = \mathbb{E}[Y^{e}|X^{e}=x] $, solves the OOD problem \eqref{eqn1: ood}. A generative model that satisfies the above Assumption \ref{assm4: invariant_cexp} is given as
\begin{equation}Y^{e} \leftarrow g(X^{e}) + \varepsilon^{e},~ \mathbb{E}[\varepsilon^e]=0,~ \varepsilon^e \perp X^{e},~ \mathbb{E}[(\varepsilon^e)^2] = \sigma^2
\label{cov_shift_model}
\end{equation}
In the above model $X^{e}$ is the cause, $Y^{e}$ is the effect, and $g$ a general non-linear function (it satisfies Assumption \ref{assm4: invariant_cexp} with $m=g$). Next, we compare ERM and IRM's ability to learn $m$ under covariate shifts. In Figure \ref{fig1:dist_shift} a), we show 
Define $\tilde{\kappa} = \min_{\Phi_1,\Phi_2  \in \mathcal{H}_{\Phi}, \Phi_1\not=\Phi_2}|R(\Phi_1)-R(\Phi_2)|$, which measures the minimum separation between the risks of any two distinct hypothesis in $\mathcal{H}_{\Phi}$.

\begin{proposition} \label{prop4: emr_irm_real}Let $\ell$ be the square loss.  For every $\nu>0$ and $\delta\in (0,1)$, if $\mathcal{H}_{\Phi}$ is a finite  hypothesis class, $m\in \mathcal{H}_{\Phi}$, Assumptions \ref{assm3: bounded loss and gradient}, \ref{assm4: invariant_cexp} hold, and

$\bullet$  if the  number of samples $|D|$ is greater than
$\max\big\{\frac{8L^2}{\nu^2} \log(\frac{4|\mathcal{H}_{\Phi}|}{\delta}),\frac{16L'^{4}}{\epsilon^2}\log(\frac{2}{\delta}) \big\}$, then  with a probability at least $1-\delta$, every solution $\hat{\Phi}$ to EIRM (\eqref{eqn: EIRM_cons_swap}) satisfies  $R(m) \leq R(\hat{\Phi}) \leq R(m) + \nu$. If also $\nu<\tilde{\kappa}$, then $\hat{\Phi} = m$.

$\bullet$ if the number of samples $|D|$ is greater than
 $\frac{8L^2}{\nu^2}\log(\frac{2|\mathcal{H}_{ \Phi}|}{\delta})$, then with a probability at least $1-\delta$, every solution $\Phi^{\dagger}$ to ERM satisfies $R(m) \leq R(\Phi^{\dagger}) \leq R(m) + \nu$.  If also $\nu<\tilde{\kappa}$, then $\Phi^{\dagger} = m$. 
\end{proposition}
\textbf{Implications of Proposition \ref{prop4: emr_irm_real}.} ERM and EIRM both asymptotically achieve the ideal OOD solution; the above proposition helps compare them in a finite sample regime. The second term inside the $\max$ for EIRM, $\frac{16L'^{4}}{\epsilon^2}\log(\frac{2}{\delta})$, does not depend on the size of the hypothesis class. Hence, for large hypothesis classes, the sample complexity of EIRM equals $\frac{8L^2}{\nu^2} \log(\frac{4|\mathcal{H}_{\Phi}|}{\delta})$. Consequently, the sample complexity of EIRM and ERM differs  by a constant $\frac{8L^2}{\nu^2}\log(2)$. Thus we  conclude, for large hypothesis classes, both ERM and EIRM have  similar sample complexity. Next, we contrast the sample complexity of EIRM in Proposition \ref{prop4: emr_irm_real}, $\mathcal{O}(\frac{1}{\nu^2})$, to Proposition \ref{prop2: scomp_irm}, $\mathcal{O}(\max\{\frac{1}{\epsilon^2}, \frac{1}{\nu^2}\})$;  the additional covariate shift assumption in Proposition \ref{prop4: emr_irm_real} helps get to a lower sample complexity of $\mathcal{O}(\frac{1}{\nu^2})$. In Proposition \ref{prop4: emr_irm_real}, we assumed square loss, but a similar result extends to cross-entropy loss as well.

\subsubsection{Distributional Shift with Confounders and (or) Anti-Causal Variables}\label{secn: non-realizable}

In this section, we consider more general models than \eqref{cov_shift_model}, which only contained cause $X^{e}$ and effect $Y^{e}$.  We also allow  confounders and anti-causal variables. However, we restrict $g$ to polynomials. We  start with linear models from \cite{arjovsky2019invariant}.

\begin{assumption}\label{assm6: linear_model}
 \begin{equation}
    \begin{split}
    &    e \sim \mathsf{Categorical}(\{\pi^{o}\}_{o\in \mathcal{E}_{tr}}), \; \forall o \in \mathcal{E}_{tr},\pi^{o}>0  \\
    &    Y^{e} \leftarrow \gamma^{\mathsf{T}}(Z_1^{e}) + \varepsilon^{e},~\varepsilon^{e} \perp Z_1^{e}, ~ \mathbb{E}[\varepsilon^{e}] = 0,~  \mathbb{E}[(\varepsilon^{e})^2] = \sigma^2, ~|\varepsilon^{e}| \leq \varepsilon^{\mathsf{sup}}   \\ 
      &  X^{e} \leftarrow S(Z_1^{e}, Z_2^{e})
    \end{split}
\end{equation}
We assume that $Z_1$ component of $S$ is invertible, i.e. $\exists\; \tilde{S}$ such that $\tilde{S}(S(Z_1,Z_2))=Z_1$, and $\gamma \not =0 $.  $\forall e \in \mathcal{E}_{tr}, \pi^{e} \geq \frac{\pi^{\mathsf{min}}}{|\mathcal{E}_{tr}|}>0$. Define $\Sigma^e = \mathbb{E}[X^{e}X^{e, \mathsf{T}}]$.  $\forall e \in \mathcal{E}_{tr}$, $\Sigma^e$ is positive definite. The support of distribution of $Z^{e}=(Z_1^e,Z_2^e)$, $\mathbb{P}^{e}_{Z^{e}}$, is bounded and the norm of $S$, $\|S\|=\sigma_{\mathsf{max}}(S)$ (maximum singular value of $S$), is also bounded.  
\end{assumption}
 In the above model, $Z_1^e$ is the cause of $X^e$ and $Y^{e}$ but may not be directly observed. $Z_2^e$ may be arbitrarily correlated with $Z_1^e$ and $\epsilon^e$. We observe a scrambled transformation of $(Z_1^e,Z_2^e)$ in $X^e$. If $Z_2^{e}$ is an effect of $Y^e$ ($Z_2^{e} \leftarrow Y^{e} + N^{e}$), then $Z_2^{e}$ is an anti-causal variable. If $H^{e}$ causes both $\varepsilon^{e}$ ($\varepsilon^{e} \leftarrow  H^{e}+\bar{N}^e)$ and  $Z_2^e$ ($Z_2^{e} \leftarrow H^{e} + \tilde{N}^{e}$), then $H^{e}$ is a confounder. In both these cases, $Z_2^e$ is spuriously correlated with the label $Y^{e}$. Consequently, the standard ERM based models  estimate $Z_2^e$ from $X^e$ and end up being biased w.r.t the desired OOD model, which does not use $Z_2^e$.  If  Assumptions \ref{assm6: linear_model}, \ref{assm2:ood_cond_envmt} hold, then a linear model  $\tilde{S}^{\mathsf{T}}\gamma$ ($X^{e} \xrightarrow{\mathsf{predict}} \gamma^{\mathsf{T}}\tilde{S}X^{e}=\gamma^\mathsf{T}Z_1^e$) solves the OOD problem in \eqref{eqn1: ood} ($\Phi^{*}=\tilde{S}$, $m=\gamma^{\mathsf{T}}$ in Proposition \ref{prop1: ood}) and it relies only on $Z_1^e$. In the supplement (Proposition 17),  we prove that ERM based models do not recover  $\tilde{S}^{\mathsf{T}}\gamma$.

For each environment $e\in \mathcal{E}_{tr}$, define  $c^e = \mathbb{E}^{e}[X^{e}\varepsilon^{e}]$. Also, define a matrix $Q(x)$, where each column of the matrix $Q(x)$ is $\Sigma^e x - c^e$. In the next theorem, we will analyze the number of samples needed for EIRM to reach within $\sqrt{\epsilon}$ radius of the OOD optimal solution. 

% \begin{assumption} \textbf{Linear general position.}
% \label{assm7: linear_gen_posn}
%  A set of training environments $\mathcal{E}_{tr}$ is said to lie in a linear general position of degree $r$ for some $r\in \mathbb{N}$ if $|\mathcal{E}_{tr}| > n-r + n/r$ and for all non-zero $x\in \mathbb{R}^{n}$ 
% \begin{equation}
%     \mathsf{dim}\Big(\mathsf{span}\Big\{\Sigma^e x-c^e\Big\}_{e\in \mathcal{E}_{tr}}\Big) > n-r
% \end{equation}
% where $\mathsf{span}$ is the linear span, $\mathsf{dim}$ is the dimension, and recall $n$ is dimension of $X^{e}$. This assumption checks for diversity in the environments and holds almost everywhere \cite{arjovsky2019invariant}. 
% \end{assumption}

\begin{assumption} \label{assm8:regularity conditions on envmts and hphi} \textbf{Inductive bias.} $\mathcal{H}_{\Phi}$ is a  finite set of linear models  parametrized by $\Phi \in \mathbb{R}^{n}$ (output $\Phi^{\mathsf{T}}X^{e}$).  $\tilde{S}^{\mathsf{T}}\gamma \in \mathcal{H}_{\Phi}$.   $\exists \; \omega>0$ s.t. $\forall \Phi \in \mathcal{H}_{\Phi},\Omega>\|\Phi\|^2 > \omega $.  
\end{assumption}
Informally stated, the above assumption requires  the OOD optimal predictor $\tilde{\mathcal{S}}^{\mathsf{T}}\gamma$ to lie in the interior of the search space and not on the boundary. If Assumptions \ref{assm6: linear_model}, \ref{assm8:regularity conditions on envmts and hphi} hold, then  Assumption \ref{assm3: bounded loss and gradient} holds. Hence, we can use the bounds $L$ and $L'$ on $\ell$  and $\frac{\partial \ell(w\cdot \Phi(\cdot), \cdot)}{\partial w}|_{w=1.0}$ respectively in our next result. 
% \begin{assumption}
% \label{assumption_singular_value}
% For all $\|x\|^2\geq \epsilon >0$, where $\epsilon$ is the approximation threshold from the constraint in \eqref{eqn: IRM_cons_swap}, the minimum eigenvalue of $Q(x)Q(x)^{\mathsf{T}}$, $\lambda_{\mathsf{min}}(Q(x)Q(x)^{\mathsf{T}}) \geq \eta \epsilon >0$, where $\eta \geq \frac{1}{4\pi^{\mathsf{min}}\omega} $.
% \end{assumption}

\begin{assumption}
\label{assumption_singular_value}
For all $\|x\|^2\geq \epsilon >0$, where $\epsilon$ is the approximation threshold from the constraint in \eqref{eqn: IRM_cons_swap}, the minimum eigenvalue of $Q(x)Q(x)^{\mathsf{T}}$, $\lambda_{\mathsf{min}}(Q(x)Q(x)^{\mathsf{T}}) \geq \frac{1}{4\pi^{\mathsf{min}}\omega} $.
\end{assumption}

The above assumption implies that the linear general position assumption from \cite{arjovsky2019invariant} holds for $r=1$ (see the linear general position assumption restated in the supplement, Assumption \ref{assm7: linear_gen_posn}).

\begin{proposition} \label{prop5: irm_cfd_chd} Let $\ell$ be the square loss. Given $\epsilon\in (0,1)$  and a $\delta \in (0,1)$, if Assumptions  \ref{assm6: linear_model},  \ref{assm8:regularity conditions on envmts and hphi}, \ref{assumption_singular_value} hold and if the number of data points $|D|$ is greater than $ \frac{16L'^{4}}{\epsilon^2}\log\big(\frac{2|\mathcal{H}_{\Phi}|}{\delta}\big)$, then  with a probability at least $1-\delta$, every solution $\hat{\Phi}$ to EIRM (\eqref{eqn: EIRM_cons_swap} with $\frac{\epsilon}{2}$) satisfies  $\|\hat{\Phi}- (\tilde{S}^{\mathsf{T}}\gamma\big)\|^2 \leq  \epsilon$.
\end{proposition}
\textbf{Proof Sketch.} In learning theory it is common to analyze the concentration of empirical risks around the expected risks. In our case, we have a target ideal solution to  \eqref{eqn1: ood} ($\tilde{S}^{\mathsf{T}}\gamma$) and we want our empirical solutions to concentrate around that. A direct finite sample approximation of  \eqref{eqn: IRM_grad_cons} is hard to analyze. Therefore, we introduce an intermediate problem in  \eqref{eqn: IRM_cons_swap} and then develop a finite sample approximation of it in  \eqref{eqn: EIRM_cons_swap}.  We first show that solving \eqref{eqn: IRM_cons_swap} leads to solutions in the neighborhood of the target. To show this we use Assumption \ref{assumption_singular_value}. Next, we connect \eqref{eqn: EIRM_cons_swap} and \eqref{eqn: IRM_cons_swap} using our new estimator for  $R^{'}$ and Hoeffding's inequality.

\textbf{Implications of Proposition \ref{prop5: irm_cfd_chd}.}
 \textbf{1. Convergence rate of ERM vs.\ EIRM:} Recall that $\epsilon$ is the slack  on IRM penalty $R^{'}$. If $\epsilon$ is sufficiently small and the data grows as $\mathcal{O}(\frac{1}{\epsilon^2})$, every solution $\hat{\Phi}$ to EIRM (\eqref{eqn: EIRM_cons_swap}) is in $\sqrt{\epsilon}$ radius of the  OOD solution, i.e., $\|\hat{\Phi} - \tilde{S}^{\mathsf{T}}\gamma\| = \mathcal{O}(\sqrt{\epsilon})$. We contrast these rates to ones in the covariate shift setting (Section \ref{secn: realizable}). Let $\mathbb{E}[Y^{e}|X^{e}=x] = \Psi^{\mathsf{T}}x$. If the data grows as $\mathcal{O}(\frac{1}{\nu^2})$, then both ERM and EIRM solution converge to $\Psi$ as $\|\hat{\Phi} - \Psi\| = \mathcal{O}(\sqrt{\nu})$ (from Proposition \ref{prop4: emr_irm_real}). This shows that EIRM works in more settings (Proposition \ref{prop4: emr_irm_real}, \ref{prop5: irm_cfd_chd}) than ERM while matching the convergence rate of ERM.

 \textbf{2. Comparison with Proposition \ref{prop2: scomp_irm}:} Lastly, we contrast  sample complexity of EIRM in Proposition \ref{prop5: irm_cfd_chd}, $\mathcal{O}(\frac{1}{\epsilon^2})$, to Proposition \ref{prop2: scomp_irm}, $\mathcal{O}(\max\{\frac{1}{\epsilon^2}, \frac{1}{\nu^2}\})$; the additional distributional assumptions in Proposition \ref{prop5: irm_cfd_chd}  help arrive at a lower sample complexity of $\mathcal{O}(\frac{1}{\epsilon^2})$. The bound in Proposition  \ref{prop2: scomp_irm}, $\mathcal{O}(\max\{\frac{1}{\epsilon^2}, \frac{1}{\nu^2}\})$, is larger than the one in Proposition \ref{prop4: emr_irm_real}, $\mathcal{O}(\frac{1}{\nu^2})$, and Proposition \ref{prop5: irm_cfd_chd}, $\mathcal{O}(\frac{1}{\epsilon^2})$, but is more general as it is agnostic to the distributional assumptions.

\textbf{A simple illustration summarizing Propositions \ref{prop4: emr_irm_real}, \ref{prop5: irm_cfd_chd}:} Set $S$ to identity in Assumption \ref{assm6: linear_model}. Recall $Z_1^e$ and $Z_2^e$ from Assumption \ref{assm6: linear_model}. Since $S$ is identity $X^e$ can be written as $[X_1^{e}, X_2^{e}]$, where $X_1^e=Z_1^e$ and $X_2^e=Z_2^e$. If $X_2^e \perp \varepsilon^e$, then  $\mathbb{E}[Y^e|X^e]$ is invariant and Assumption \ref{assm4: invariant_cexp} holds. This corresponds to the setup in Figure \ref{fig1:dist_shift}a).  We can now use Proposition \ref{prop4: emr_irm_real} and deduce that ERM and IRM have same sample complexity and end up learning the ideal model that only uses  the causal features $X_1^e$. If $X_2^e \leftarrow \varepsilon^e+ N^e$, then this corresponds to the setup Figure \ref{fig1:dist_shift}b), $X_1^e$ is the cause and $X_2^e$ is spuriously correlated  with label $Y^e$ through the confounder $\varepsilon^e$. If $X_2^e \leftarrow Y^e+ N^e$, then this corresponds to the setup in Figure \ref{fig1:dist_shift}c), $X_1^e$ is the cause and $X_2^e$ is anti-causally related to the label $Y^e$. In both these cases, the ideal OOD solution that solves \eqref{eqn1: ood} will only exploit $X_1^e$ to make predictions. From Propositions \ref{prop5: irm_cfd_chd}, it follows that IRM when fed with $\mathcal{O}(\frac{1}{\epsilon^2})$ samples, it is in $\sqrt{\epsilon}$ radius of the target OOD solution, while ERM is asymptotically biased and exploits $X_2^e$ (Proposition 17).

We define a polynomial version of the model in Assumption \ref{assm6: linear_model}. We only need to change $Y^{e} \leftarrow \gamma^{\mathsf{T}}(Z_1^{e}) + \varepsilon^{e}$ to  $Y^{e} \leftarrow \gamma^{t}\zeta_{p}(Z_1^{e}) + \varepsilon^{e}$. $\zeta_p$ is a polynomial feature map of degree $p$ defined as $\zeta_p:\mathbb{R}^{c} \rightarrow \mathbb{R}^{c^{'}}$, where $c$ is the dimension of the input $Z_1^e$, 
$\zeta_p(W) = \big(W, W\otimes W,\dots,  (W \otimes W ... \text{p times} \otimes W)\big) = \big((W^{\otimes i})_{i=1}^{p}\big)$ and $\otimes$ is the Kronecker product. Also, $c^{'} = \sum_{i=1}^{p}c^{i}$.  Can we directly use the analysis from the linear case by transforming $X^{e}$ appropriately? No, we first need to find an appropriate transformation for the scrambling matrix $S$ that satisfies the conditions (invertibiltiy) in Assumption \ref{assm6: linear_model} while maintaining a linear relationship between transformations of $X^{e}$ and $Z^{e}$. We  present the main result informally below (details are in the supplement).

\begin{proposition} \label{prop6:pol_inf} (Informal statement) Let $\ell$ be the square loss. Given $\epsilon\in (0,1)$  and a $\delta \in (0,1)$, if Assumptions similar to Proposition \ref{prop5: irm_cfd_chd} hold and  $|D|\geq \frac{16L'^{4}}{\epsilon^2}\log(\frac{2|\mathcal{H}_{\Phi}|}{\delta})$, then  with a probability at least  $1-\delta$, every solution $\hat{\Phi}$ to EIRM (\eqref{eqn: EIRM_cons_swap} with $\frac{\epsilon}{2}$) satisfies  $\|\hat{\Phi} - \bar{\tilde{S}}^{\mathsf{T}}\gamma\|^2 \leq \epsilon$, where $\bar{\tilde{S}}^{\mathsf{T}}\gamma$ is the OOD optimal solution (defined in the supplement).

\end{proposition}

\textbf{Insights from the polynomial case and infinite hypothesis case.} In the polynomial case, we adapt the  Assumption \ref{assumption_singular_value}, the number of environments $|\mathcal{E}_{tr}|$ are now required to grow as  $\mathcal{O}(n^p)$.  In the infinite hypothesis case, the main change in the results is that we replace $|\mathcal{H}_{\Phi}|$ with an appropriate model complexity metric \cite{shalev2014understanding}. 
% Consider Proposition \ref{prop5: irm_cfd_chd}, a sample complexity of $n\log(n)$ ensures $\|\hat{\Phi}- (\tilde{S}^{\mathsf{T}}\gamma\big)\|^2\leq \epsilon$  in contrast to $ \log\big(|\mathcal{H}_{\Phi}|\big)\Big)$ in the finite hypothesis case.  
We showed the benefits of IRM   for polynomial models, other extensions such as non-linear $S$ are an important future work. In the supplement, we provide a dialogue explaining how our work fits in the big picture.

\section{Experiments}

In this section, we discuss classification experiments (regression experiments with similar qualitative findings are in the supplement). We  introduce three new variants of the colored MNIST (CMNIST) dataset in \cite{arjovsky2019invariant}. We divide the training data in MNIST digits into two environments ($e=1,2$) equally and
the testing data in MNIST digits is assigned to another environment ($e=3$).   $X_g^e$: gray scale image of the digit, $Y_g^e$: label of the gray scale digit (digits $\geq 5$ have $Y_g^e = 1$ and digits $<5$ have $Y_g^e= 0$). $X^e$: final colored image and $Y^e$: final label are generated as follows. Define Bernoulli variables $G$, $N$, $N^e$ that take a value $1$ with probability $\theta$, $\beta$ and $\beta^e$ and $0$ otherwise. Define a color variable $C^{e}$, where $C^{e}=0$ is red and $C^{e}=1$ is green. Let $\oplus$ denotes xor operation.
\begin{equation}
    \begin{split}
        & Y_g^e \leftarrow \mathsf{L}(X_g^e), \mathsf{L}: \text{Human labeling},\\
        &Y^{e} \leftarrow \mathsf{L}(X_g^e)\oplus N,\; \text{Corrupt the original labels with noise}  \\ 
        & C^e \leftarrow G(Y^{e} \oplus N^e) + (1-G)(N\oplus N^e),\; \text{Use} \;G\; \text{to select b/w anti-causal or confounded} \\
        & X^{e} \leftarrow T(X_g,C^e), T: \text{transformation to color the image}
    \end{split}
    \label{cmnist_acausal_cfd_sem}
\end{equation}

If the probability $\theta=1$, then $G=1$ and that gives us back the original CMNIST in \cite{arjovsky2019invariant}, which we call anti-causal CMNIST (AC-CMNIST). If $\theta=0$, then we get  confounded colored MNIST (CF-CMNIST). If $0<\theta<1$, we get a hybrid dataset (HB-CMNIST). The above model in \eqref{cmnist_acausal_cfd_sem} has features of the model in Assumption \ref{assm6: linear_model}, where $X_g^e$, $C^e$, $\mathsf{L}$, $T$ take the role of $Z_1^e$, $Z_2^e$, $\gamma$, $S$. We set the noise $N$ parameter $\beta=0.25$, and the parameter for $N^e$ in the three environments $[\beta^1,\beta^2,\beta^3] =[0.1,0.2,0.9]$. Color is spuriously correlated with the label; $\mathbb{P}(C^e=1|Y^e=1)$ varies drastically in the three environments ($[0.9,0.8,0.1]$ for AC-CMNIST). In the variants of CMNIST we discussed, $\mathbb{P}(Y^e|X^e)$ varies across the environments. We now define a covariate shift based CMNIST (CS-CMNIST) ($\mathbb{P}(Y^e|X^e)$ is invariant). We use selection bias to induce spurious correlations. Generate a color $C^{e}$ uniformly at random. Select the pair $(X_g^e,C^{e})$ with probability  $1-\psi^{e}$ if the label $Y_g^{e}$ and $C^e$ are the same, else select them  with a probability $\psi^{e}$. If the pair is selected, then color the image $X^{e} \leftarrow T(X_g,C^e)$ and $Y^{e}\leftarrow Y^{e}_{g}$. Selection probability $\psi^e$ for the three environments are $[\psi^{1},\psi^{2},\psi^{3}]= [0.1,0.2,0.9]$. Due to the selection bias, color is spuriously correlated,  $\mathbb{P}(C^e=1|Y^e=1)$ varies drastically $[0.9,0.8,0.1]$. We provide the graphical models for the CMNIST variants and computations of $\mathbb{P}(C^e|Y^e)$, $\mathbb{P}(Y^e|X^e)$ in the supplement.

\begin{figure}[!h]
\centering
  \includegraphics[width=.3\textwidth]{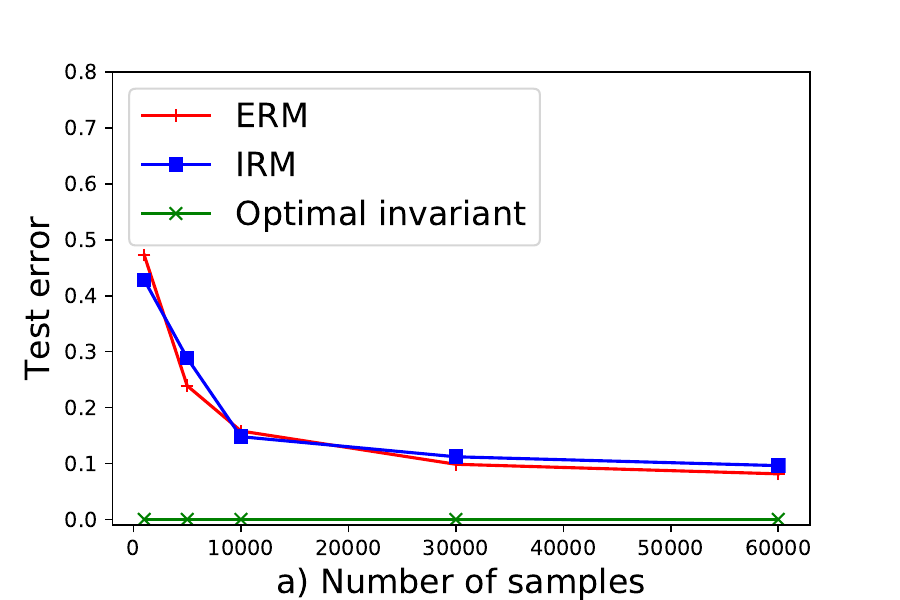}
  \includegraphics[width=.3\textwidth]{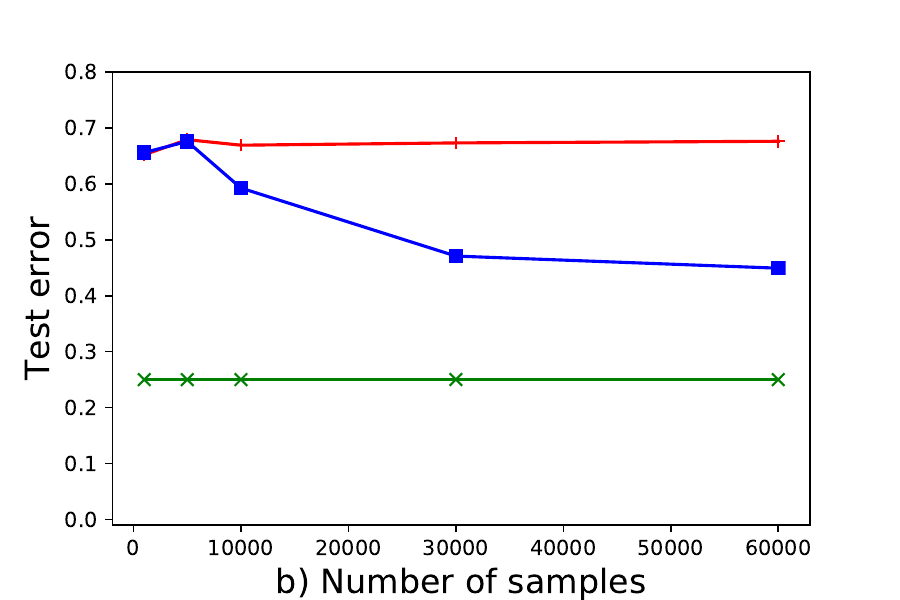}\\
  \includegraphics[width=.3\textwidth]{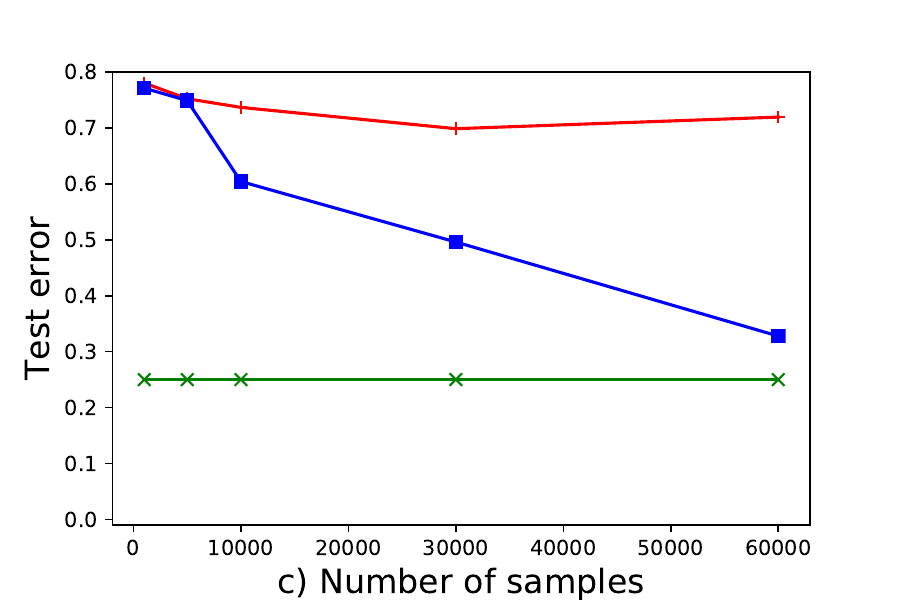}
  \includegraphics[width=.3\textwidth]{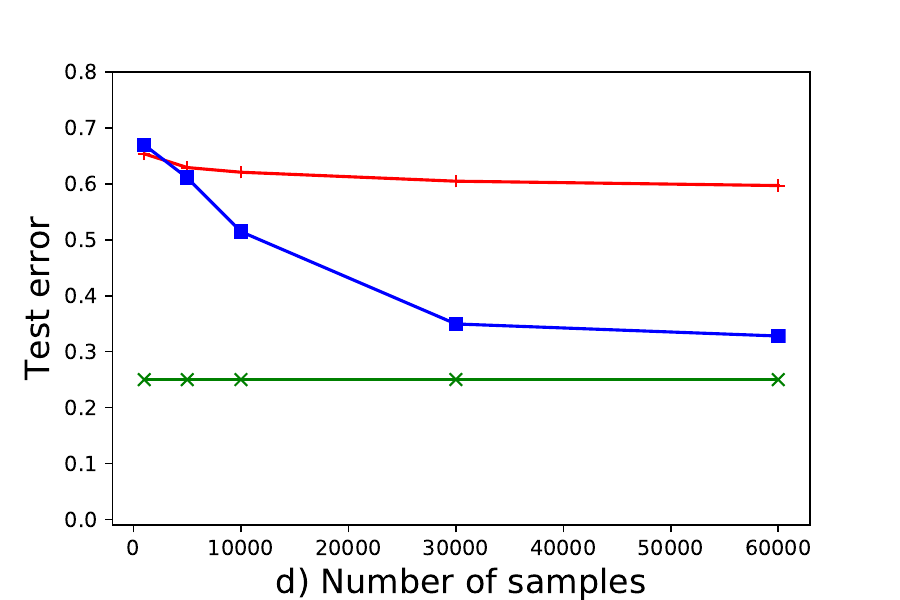}
\caption{\small{Comparisons: a) CS-CMNIST, b) CF-CMNIST, c) AC-CMNIST and d) HB-CMNIST.}}
\label{fig:cmnistvar}
\end{figure}

\subsection{Results}

We use the first two environments ($e=1,2$) to train and third environment ($e=3$) to test. Other details of the training (models, hyperparameters, etc.) are in the supplement. For each of the above datasets, we run the experiments for different amounts of training data from $1000$ to up to $60000$ samples (10 trials for each data size).  In Figure \ref{fig:cmnistvar}, we compare the models trained using IRM and ERM in terms of the classification error on the test environment $e=3$ (a poor  performance indicates model exploits the color) for varying number of train samples. We also provide the performance of the ideal hypothetical optimal invariant model. Observe that except for in the covariate shift setting where IRM and ERM are similar as seen in Figure \ref{fig:cmnistvar}a (as predicted from Proposition \ref{prop4: emr_irm_real}), IRM outperforms ERM in the remaining three datasets (as predicted from Proposition \ref{prop5: irm_cfd_chd}) as seen in Figure \ref{fig:cmnistvar}b-d. We further validate this claim through the regression experiments provided in the supplement.  In CF-CMNIST, IRM achievs an error  of $0.45$, which is much better than error  of ERM $(0.7)$ but is marginally better than a random guess. This suggests that confounder induced spurious correlations are harder to mitigate and may need more samples than needed in anti-causal case (AC-CMNIST).

\section{Conclusion}

We presented a sample complexity analysis of IRM to answer the question: when is IRM better than ERM (and vice-versa)? For distribution shifts such as the covariate shifts, we proved that both IRM and ERM have similar sample complexity and arrive at the desired OOD solution asymptotically.  For distribution shifts involving confounders and (or) anti-causal variables and polynomial generative models, we proved that IRM is guaranteed to achieve the desired OOD solution while ERM can be asymptotically biased. We proposed new variants of original colored MNIST dataset from \cite{arjovsky2019invariant}, which are more comprehensive and better capture how spurious correlations occur in reality. %We validate predictions from theory on these datasets and also on regression datasets. 
To the best of our knowledge, we believe this to be the first work that provides a rigorous characterization of impact of factors such as model complexity, number of environments on the sample complexity and its distance from the OOD solution in distribution shifts that go beyond covariate shifts. 

\section{Acknowledgements} 
This work was supported in part by the Rensselaer-IBM AI Research Collaboration (part of the IBM AI Horizons Network).

\section{Errata}

In the previous version of this work, there was an error in the proof of Proposition \ref{prop5: irm_cfd_chd}. We are extremely thankful to Advait Parulekar and Sanjay Shakkottai for bringing this to our attention and engaging in an insightful discussion.  In the current version, we have fixed the error, the main findings of the work (highlighted in Table \ref{table1}) have not changed.

\section{Supplement}

\subsection{Dialogue on IRM Continued}

[\emph{In the original IRM paper by \cite{arjovsky2019invariant}, we left two graduate students \textsc{Eric} and \textsc{Irma} strolling along the Parisian streets on a warm summer evening. A lot has transpired since then. \textsc{Eric} is now holed up inside his room at Cit\'e Universitaire and \textsc{Irma} has returned to her parents' home in Provence. The two are talking through a Zoom call.}]

\begin{itemize}
    \item \textsc{Eric}: After reading \cite{arjovsky2019invariant} and \cite{gulrajani2020search}, I was not sure of how to understand the settings where IRM is beneficial over ERM and vice-versa? From  \cite{gulrajani2020search}, I understood that ERM continues to be the state-of-the-art if sufficient care is taken in performing model selection.
    \item \textsc{Irma}:  But after reading this manuscript I understand that there are settings where no matter how one selects the model, ERM is bound to be at a disadvantage.
    \item \textsc{Eric}: What you say seems to contradict my understanding of \cite{gulrajani2020search}.
    \item \textsc{Irma}: Actually, they ...
\end{itemize}

[\emph{\textsc{Irma}'s audio and video freeze for a minute, leaving \textsc{Eric} to ponder this conundrum for a little bit on his own. \textsc{Irma} stops her video and continues only with audio and the conversation is able to resume.}]

\begin{itemize}
    \item \textsc{Irma}: Sorry, how much did you hear?
    \item \textsc{Eric}: You were just starting to explain why ERM can be at a disadvantage, but I didn't hear anything after that.
    \item \textsc{Irma}: Okay, let me start my explanation again. There is no contradiction. The current manuscript says that in covariate shift settings, there maybe no clear winner. Many of the datasets considered in \cite{gulrajani2020search} are perhaps similar to covariate shift settings. Also, there could be another reason why IRM did not outperform ERM in \cite{gulrajani2020search} and I will come to it in a bit.
    \item \textsc{Eric}: Yes, you are right! The datasets are human labeled images if I recall correctly. In these datasets it's safe to assume $\mathbb{P}(Y^e|X^e)$ is close to invariant across domains (as long as the cohort of humans in the different domains are not very different). What happens when $\mathbb{P}(Y^e|X^e)$ varies a lot? 
    \item \textsc{Irma}: Yes, when $\mathbb{P}(Y^e|X^e)$ varies a lot, IRM can be at an advantage over ERM. In this manuscript, I learned that when data generation involves confounders or anti-causal variables, it is possible to show that IRM methods converge to desirable OOD solutions with good convergence rates. 
    % \item ERMON: What are the settings when $\mathbb{P}(Y^e|X^e)$ varies a lot? 
    % \item IRMIE: When there are confounders in the dataset that make the label and the covariates appear to be correlated, then $\mathbb{P}(Y^e|X^e)$ varies a lot. ..
    
    \item \textsc{Eric}: Interesting! I should read that. The experiment on colored MNIST by \cite{arjovsky2019invariant} contained anti-causal variables right? Is that the reason why IRM performed better than ERM? 
    \item \textsc{Irma}: Yes, you are right. In the current manuscript, the authors also develop other variants of colored MNIST (covariate shift based, confounder based, anti-causal based, and hybrid of confounder and anti-causal). IRM retains its advantage on all but the first dataset. I liked that IRM has advantage over ERM in confounded datasets as I believe many real datasets can have confounders generating spurious correlations. 
    \item \textsc{Eric}: Wow! These different variants of CMNIST sound exciting. Going back to the covariate shift case, do you think that we can perhaps come up with a finer criterion to say when IRM is better than ERM and vice-versa?
    \item \textsc{Irma}: Yes, actually I have been thinking about this problem and would be happy to share my initial thoughts. Identifying a representation $\Phi^{*}$ that leads to  invariant conditional distributions is crucial to the success of IRM. A subtle factor that is implicitly assumed is representations obtained from multiple domains should overlap. 
    \item \textsc{Eric}: Can you clarify what you mean by overlap?
    \item \textsc{Irma}: Sure! Consider the dataset from domain 1, say photographs of birds, and domain 2, say sketches of birds. Imagine I have access to the oracle representation $\Phi^{*}$. Now I pass the data from domain 1 and 2 through $\Phi^{*}$ to get the representations. If these representations from the two domains live in very different parts of the representation space, then we cannot hope that IRM will offer any advantage. This is the other explanation, which I was going to come to, why IRM did not outperform ERM in \cite{gulrajani2020search}. 
    \item \textsc{Eric}: How about when there is a complete overlap?
    \item \textsc{Irma}: Yes, if there is a strong or complete overlap in the representations from the two domains, then it is possible that IRM can help. In fact I believe in such settings, if the ERM models trained on the two domains disagree a lot, then that can be a strong indication towards the models heavily exploiting spurious correlations. In such cases, I believe IRM can again offer some advantage. 
    \item \textsc{Eric}: I will try to put these ideas down on paper and meet you again for a discussion. 
    \item \textsc{Irma}: Great! I truly hope that it can be in person at the caf\'e in Palais-Royal where we first started this conversation. So long for now!
    
    % You are right, so this is how I would summarize my understanding. The current manuscript says that in settings with covariate shifts there is no clear winner. I believe most of the experiments in \cite{gulrajani2020search} operated with datasets that were generated with covariate shift assumptions. 
\end{itemize}

[\emph{The students end their call.}]

\subsection{Supplementary Materials for Experiments}

In this section, we cover the supplementary materials for the experiments.  The code to reproduce the results presented in this work can be found at \url{https://github.com/IBM/OoD}.

\subsubsection{Classification}
We first describe the model and other training details. 

\textbf{Choice of $\mathcal{H}_{\Phi}$ and other training details} 
We use the same architecture for both ERM and IRM. We choose the architecture that was used in \cite{arjovsky2019invariant}: 2 layer MLP. The first two layer consists of 390 hidden nodes, and the output layer has two nodes (for the two classes). We use ReLU activation in each layer, a regularization weight of 0.0011 is used for each layer. We use a learning rate of 4.9e-4, batch size of 512 for both ERM and IRM. We use 1000 gradient steps for IRM. As was done in the original IRM work \cite{arjovsky2019invariant}, we use a threshold on steps (190) after which a large penalty is imposed for violating the IRM constraint.  We use the train domain validation set procedure described in \cite{gulrajani2020search} to select the penalty value from the set 
$\{1e4, 3.3e4, 6.6e4, 1e5\}$ (with 4:1 train-validation split). With the same learning rate, we observed ERM was slower at learning than IRM. To ensure ERM always converges we set the number of epochs to a very high value $100$ ($118k$ steps).

\textbf{A. Covariate shift based for CMNIST}

We provide the generative model for CS-CMNIST below. 

\begin{equation}
    \begin{split}
        & Y_g^e \leftarrow \mathsf{L}(X_g^e), \; C^e\leftarrow \mathsf{Uniform}(\{0,1\})  \\ 
        & U^e \leftarrow \mathsf{Bernoulli}\Big((C^e \oplus Y_g^e)\psi^e + \big(1-(C^e \oplus Y_g^e)\big)\big(1- \psi^e)\big)\Big)\\
        & X^e \leftarrow T(X^g,C^e) |U^e =1, Y^e \leftarrow Y_g^e |U^e=1
    \end{split}
    \label{cmnist_selbias_sem}
\end{equation}
\textbf{A.1. Compute $\mathbb{P}(C^e|Y^e)$ and $\mathbb{P}(Y^e|X^e)$.}

We compute $\mathbb{P}(C^e|Y^e)$ and $\mathbb{P}(Y^e|X^e)$ for the covariate shift based CMNIST described by \eqref{cmnist_selbias_sem}.  $\mathbb{P}(C^e|Y^e)$ helps us understand how the spurious correlations vary across the environments. $\mathbb{P}(Y^e|X^e)$ helps us understand if the covariate shift condition is satisfied or not. 
Compute the probability $\mathbb{P}(C^e|Y^e=1) = \mathbb{P}(C^e|Y_g^e=1, U^e=1)$ as follows.
\begin{equation}
    \begin{split}
      &  \mathbb{P}(C^{e}=1|Y_g^{e}=1, U^{e}=1) =  \frac{\mathbb{P}(C^{e}=1,Y_g^{e}=1| U^{e}=1)}{\mathbb{P}(C^{e}=1,Y_g^{e}=1| U^{e}=1) + \mathbb{P}(C^{e}=0,Y_g^{e}=1| U^{e}=1)} \\
       & \mathbb{P}(C^{e}=1,Y_g^{e}=1| U^{e}=1) = \frac{\mathbb{P}(C^{e}=1,Y_g^{e}=1, U^{e}=1)}{\sum_{a,b} \mathbb{P}(C^{e}=a,Y_g^{e}=b, U^{e}=1)}   = \frac{1}{2}(1-\psi^{e}) \\
         &\mathbb{P}(C^{e}=0,Y_g^{e}=1| U^{e}=1)  =\frac{\mathbb{P}(C^{e}=0,Y_g^{e}=1, U^{e}=1)}{\sum_{a,b} \mathbb{P}(C^{e}=a,Y_g^{e}=b, U^{e}=1)} = \frac{1}{2}\psi^{e} \\
         &\mathbb{P}(C^{e}=1|Y_g^{e}=1, U^{e}=1) = (1-\psi^{e})
    \end{split}
\end{equation}

% Compute the probability $\mathbb{P}(C^e|Y^e=1) = \mathbb{P}(C^e|Y_g^e=1, S=1)$ for environment 1.
% \begin{equation}
%     \begin{split}
%       &  \mathbb{P}(Z_c=1|Y_g=1, S=1) =  \frac{\mathbb{P}(Z_c=1,Y_g=1| S=1)}{\mathbb{P}(Z_c=1,Y_g=1| S=1) + \mathbb{P}(Z_c=0,Y_g=1| S=1)} \\
%       & \mathbb{P}(Z_c=1,Y_g=1| S=1) = \frac{\mathbb{P}(Z_c=1,Y_g=1, S=1)}{\sum \mathbb{P}(Z_c=a,Y_g=b, S=1)}   = 0.45 \\
%          &\mathbb{P}(Z_c=0,Y_g=1| S=1)  =\frac{\mathbb{P}(Z_c=0,Y_g=1, S=1)}{\sum \mathbb{P}(Z_c=a,Y_g=b, S=1)} = 0.05 \\
%          &\mathbb{P}(Z_c=1|Y_g=1, S=1) = 0.9
%     \end{split}
% \end{equation}
% Similarly, $ \mathbb{P}(Z_c=1|Y_g=1, S^e=1) = 0.8$ for environment 2 and $0.1$ for test environment. 

From the above simplification we gather that $\mathbb{P}(C^{e}=1|Y^{e}=1)$ is $0.9$, $0.8$ and $0.1$ in the three environments. 
Define $p_l=P(Y_g^e=1|X_g^{e})$. In MNIST data it is reasonable to assume deterministic labeling, i.e. $p_l=1$ or $p_l=0$. From the way the data is constructed we can assume that $T$ is invertible, i.e. from colored image $X^{e}$ we can get back the grayscale image and the color $X_g^{e}, C^{e}$. We now move to computing $\mathbb{P}(Y^{e}|X^{e})$. We assume $Y_g^e=1$ in the simplifcation below. 

\begin{equation}
    \begin{split}
       & \mathbb{P}(Y_g^{e}=1,C^{e}=0,U^e=1|X_g^{e}) = \mathbb{P}(Y_g^{e}=1|X_g^{e}) \mathbb{P}(C^{e}=0)\mathbb{P}(U^e=1|Y_g^{e}=1, C^{e}=0)  \\& = 0.5p_l\psi_e\\ 
       &\mathbb{P}(Y_g^{e}=0,C^{e}=0,U^e=1|X_g^{e}) = \mathbb{P}(Y_g^{e}=0|X_g^{e}) \mathbb{P}(C^{e}=0)\mathbb{P}(U^e=1|Y_g^{e}=1, C^{e}=0)   \\&=0.5 (1-p_l) (1-\psi_e)\\
               &\mathbb{P}(Y^{e}=1|X^{e})  =   \mathbb{P}(Y_g^{e}=1|X_g^{e},C^{e}=0,U^e=1)  = \\& \frac{ \mathbb{P}(Y_g^{e}=1,C^{e}=0,U^e=1|X_g^{e})}{\mathbb{P}(Y_g^{e}=1,C^{e}=0,U^e=1|X_g^{e}) + \mathbb{P}(Y_g^{e}=0,C^{e}=0,U^e=1|X_g^{e})} \\ 
      & \mathbb{P}(Y^{e}=1|X^{e}) = \frac{p_l \psi_e}{2p_l\psi_e + 1-p_l-\psi_e} 
    \end{split}
\end{equation}
Observe that under standard assumption of deterministic labeling $p_l=1$, $\mathbb{P}(Y^{e}=1|X^{e})$ remains invariant and for high values of $p_l$ it is relatively stable across the environments. 

% \begin{equation}
%     \begin{split}
%       & \mathbb{P}(Y_g^{e}=1,Z_c=0,S^e=1|X_g) = \mathbb{P}(Y_g=1|X_g) \mathbb{P}(Z_c=0)\mathbb{P}(S^e=1|Y_g=1, Z_c=0)  =p_l\times 0.5\times p_e\\ 
%       &\mathbb{P}(Y_g=0,Z_c=0,S^e=1|X_g) = \mathbb{P}(Y_g=0|X_g) \mathbb{P}(Z_c=0)\mathbb{P}(S^e=1|Y_g=1, Z_c=0)  = (1-p_l)\times0.5 \times (1-p_e)\\
%               &\mathbb{P}(Y^{e}=1|X^{e})  =   \mathbb{P}(Y_g=1|X_g,Z_c=0,S^e=1)  = \frac{ \mathbb{P}(Y_g=1,Z_c=0,S^e=1|X_g)}{\mathbb{P}(Y_g=1,Z_c=0,S^e=1|X_g) + \mathbb{P}(Y_g=0,Z_c=0,S^e=1|X_g)} \\ 
%       & \mathbb{P}(Y^{e}=1|X^{e}) = \frac{p_l p_e}{2p_lp_e + 1-p_l-p_e}  = 0.916\;\; \text{(For environment 1)} \\
%       & \mathbb{P}(Y^{e}=1|X^{e}) =  0.961\;\; \text{(For environment 2)} \\
%           & \mathbb{P}(Y^{e}=1|X^{e}) = 0.999\;\; \text{(For environment 2)} 
%     \end{split}
% \end{equation}

\textbf{A.2 Graphical model for covariate shift based CMNIST.}
In Figure \ref{fig:sb_cmnist}, we provide the graphical model for covariate shift based CMNIST described in \eqref{cmnist_selbias_sem}.
\begin{figure}
    \centering
    \includegraphics[trim=0 0 0 0in , width=6in]{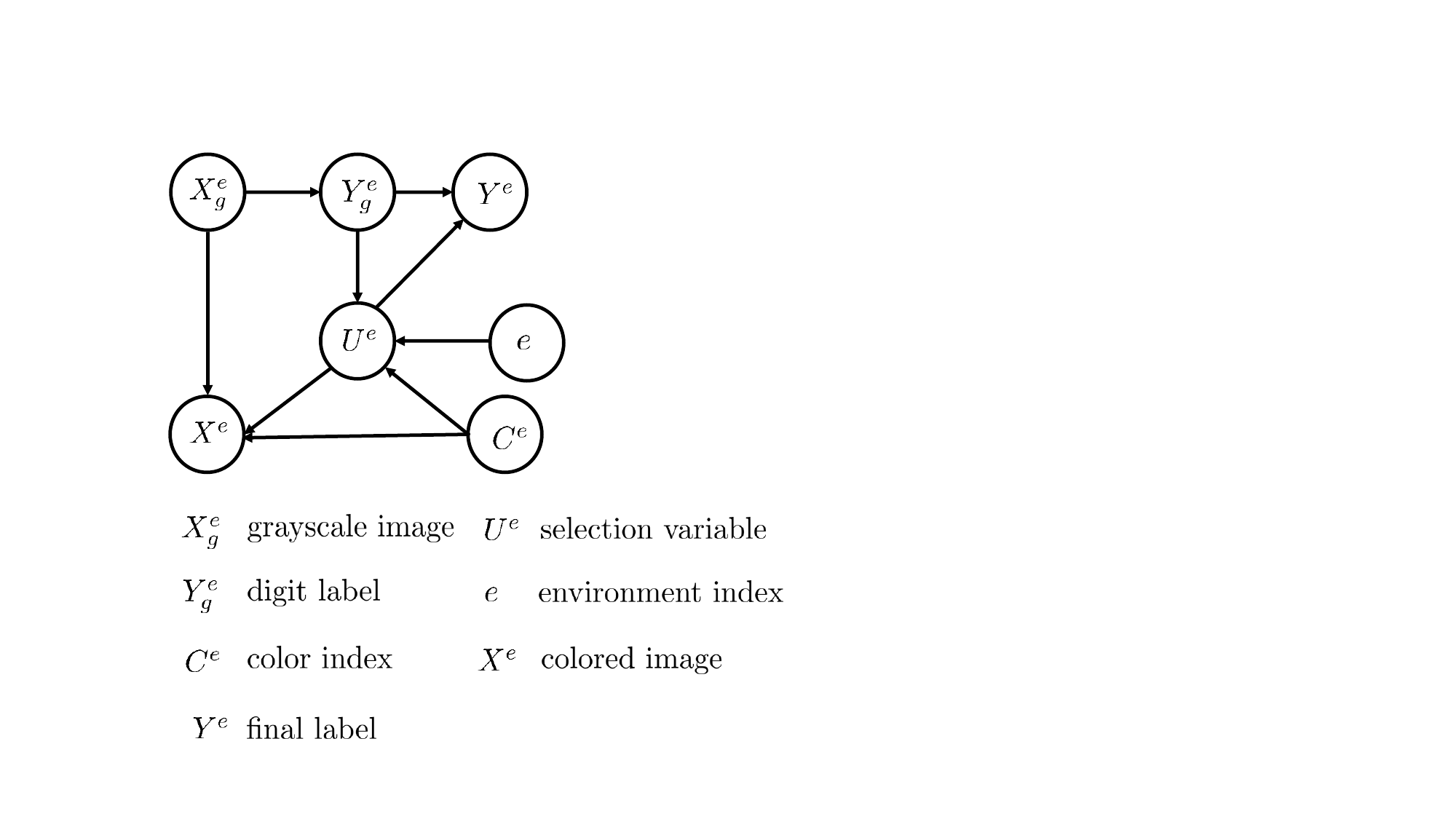}
    \caption{Graphical model for CS-CMNIST}
    \label{fig:sb_cmnist}
\end{figure}

\textbf{A.3. Results with numerical values and standard errors.}
In Table \ref{table_results_cmnits_sb}, we provide the numerical values for the results showed in the Figure \ref{fig:cmnistvar} a) along with the standard errors.
% \begin{table}[h!]
% \centering
% \def\arraystretch{1.5}
%  \begin{tabular}{||c c c||} 
%  \hline
% Method & Number of samples  & Test error \\ [0.5ex] 
%  \hline\hline
%  ERM &  $1000$ &  $45.55 \pm 1.06$ \\  
%   \hline
%  IRMv1 &  $1000$ & $41.17 \pm 1.37$\\ 
%  \hline
%   ERM &  $5000$ &   $23.31 \pm 0.47$\\ \hline 
%  IRMv1 &  $5000$ &  $22.97 \pm 3.01$\\ 
%  \hline
%   ERM &  $10000$ &   $16.93 \pm 0.55$\\ \hline 
%  IRMv1 &  $10000$ &   $15.11 \pm 0.28$\\ \hline 
%   ERM &  $30000$ &   $10.20 \pm 0.19$\\ \hline 
%  IRMv1 &  $30000$ &   $10.27 \pm 0.64$\\ \hline 
%     ERM &  $60000$ &   $7.89 \pm 0.36$\\ \hline 
%  IRMv1 &  $60000$ &   $10.33 \pm 0.46$\\ \hline 
% \end{tabular}
%  \caption{Comparison of ERM vs IRMv1: CS-CMNIST}
% % \end{center}
% \label{table_results_cmnits_sb}
% \end{table}

\begin{table}[h!]
\centering
\def\arraystretch{1.5}
 \begin{tabular}{||c c c||} 
 \hline
Method & Number of samples  & Test error \\ [0.5ex] 
 \hline\hline
 ERM &  $1000$ &  $47.27 \pm 0.72$ \\  
  \hline
 IRM &  $1000$ & $42.84 \pm 0.94$\\ 
 \hline
  ERM &  $5000$ &   $23.91 \pm 1.23$\\ \hline 
 IRM &  $5000$ &  $28.85 \pm 5.21$\\ 
 \hline
   ERM &  $10000$ &   $15.80 \pm 0.63$\\ \hline 
 IRM &  $10000$ &   $14.80 \pm 0.30$\\ \hline 
   ERM &  $30000$ &   $9.86 \pm 0.32$\\ \hline 
 IRM &  $30000$ &   $11.20 \pm 0.64$\\ \hline 
    ERM &  $60000$ &   $8.15 \pm 0.37$\\ \hline 
 IRM &  $60000$ &   $9.62 \pm 0.67$\\ \hline 
\end{tabular}
 \caption{Comparison of ERM vs IRM: CS-CMNIST}
% \end{center}
\label{table_results_cmnits_sb}
\end{table}
\textbf{B. Colored MNIST with anti-causal variables (AC-CMNIST)}

\textbf{B.1. Compute $\mathbb{P}(C^{e}|Y^e)$ and $\mathbb{P}(Y^{e}|X^{e})$.}
We first compute  $\mathbb{P}(C^{e}|Y^e)$.
 $\mathbb{P}(C^{e}=1|Y^e=1) = \mathbb{P}(Y^e \oplus N^e =1|Y^e=1) = \mathbb{P}(N^e=0|Y^{e}=1) = \beta^{e}$. Therefore,  $\mathbb{P}(C^{e}=1|Y^e=1)$ is $0.9,0.8$ and $0.1$ in environments $1,2$, and $3$ respectively.
Next, we compute $\mathbb{P}(Y^{e}|X^{e})$. We assume $Y_g^{e}=1$ in the simplfication below. We also assume deteministic labeling. In the simplfication that follows we use $\beta=0.25$. 

\begin{equation}
    \begin{split}
     &\mathbb{P}(Y^e=1,C^{e}=0|X_g^{e}) =\\& \mathbb{P}(Y^e=1,C^{e}=0|X_g^{e}) =  \mathbb{P}(Y^{e}=1|Y_g^{e}=1) \mathbb{P}(C^{e}=0|Y^{e}=1) = 0.75\beta^{e} \\ 
          &\mathbb{P}(Y^e=0,C^{e}=0|X_g^{e}) = \\& \mathbb{P}(Y^e=1,C^{e}=0|X_g^{e}) = \mathbb{P}(Y^{e}=0|Y_g^{e}=1) \mathbb{P}(C^{e}=0|Y^{e}=0) = 0.25(1-\beta^{e}) \\ 
    & \mathbb{P}(Y^{e}=1|X^{e}) = \mathbb{P}(Y^e=1|X_g^{e},C^{e}=0) = \frac{3\beta^{e}}{2\beta^{e}+1}  \\
          & \mathbb{P}(Y^{e}=1|X^{e}) =  0.25\;\; \text{(For environment 1)} \\
      & \mathbb{P}(Y^{e}=1|X^{e}) =  0.428 \;\; \text{(For environment 2)} \\
          & \mathbb{P}(Y^{e}=1|X^{e}) = 0.964\;\; \text{(For environment 2)} 
    \end{split}
\end{equation}
\textbf{B.2 Graphical model for anti-causal CMNIST}
In Figure \ref{fig:cfd_cmnist_combined}, we provide the graphical model for AC-CMNIST described in \eqref{cmnist_acausal_cfd_sem} (for $G=1$).

% \begin{figure}
%     \centering
%     \includegraphics[trim={0, 2in, 0 ,3in }, width=6in]{colored_mnist_child_gm.pdf}
%     \caption{Anti-causal colored MNIST}
%     \label{fig:chd_cmnist}
% \end{figure}

\textbf{B.3. Results with numerical values and standard errors.}
In Table \ref{table_results_cmnits_chd}, we provide the numerical values for the results showed in the Figure \ref{fig:cmnistvar} c) along with the standard errors.

\begin{table}[h!]
\centering
\def\arraystretch{1.5}
 \begin{tabular}{||c c c||} 
 \hline
Method & Number of samples  & Test error \\ [0.5ex] 
 \hline\hline
 ERM &  $1000$ &   $78.04 \pm 0.70$\\  
  \hline
 IRM &  $1000$ & $ 77.12 \pm 1.00$\\ 
 \hline
  ERM &  $5000$ &   $75.23 \pm 1.04$\\ \hline 
 IRM &  $5000$ &  $74.89 \pm 1.08$\\ 
 \hline
   ERM &  $10000$ &   $74.68 \pm 1.23$\\ \hline 
 IRM &  $10000$ &   $60.42 \pm 1.72$\\ \hline 
   ERM &  $30000$ &   $69.87 \pm 0.39$\\ \hline 
 IRM &  $30000$ &   $49.61 \pm 2.57$\\ \hline 
    ERM &  $60000$ &   $71.96 \pm 0.55$\\ \hline 
 IRM &  $60000$ &   $32.78 \pm 2.70$\\ \hline 
\end{tabular}
 \caption{Comparison of ERM vs IRM: AC-CMNIST}
% \end{center}
\label{table_results_cmnits_chd}
\end{table}

\textbf{C. Colored MNIST with confounded variables (CF-CMNIST)}

\textbf{C.1. Compute $\mathbb{P}(C^{e}|Y^e)$ and $\mathbb{P}(Y^{e}|X^{e})$.}
We start by computing $\mathbb{P}(C^{e}|Y^e)$.  Recall that $C^{e} = (N\oplus N^{e}) $.   In the simplfication that follows we use $\beta=0.25$. 

\begin{equation}
    \begin{split}
        \mathbb{P}(N=0|Y^{e}=1) = \frac{\mathbb{P}(N=0, Y^{e}=1)}{\mathbb{P}(N=0, Y^{e}=1) + \mathbb{P}(N=1, Y^{e}=1)} = \frac{0.75*0.5}{0.75*0.5 + 0.25*0.5} = 0.75
    \end{split}
\end{equation}
We use the above to compute $\mathbb{P}(C^{e}=0|Y^e=1) = \mathbb{P}(N=0,N^{e}=0|Y^{e}=1) + \mathbb{P}(N=1,N^{e}=1|Y^{e}=1)  = (1-\beta^{e})0.75 + 0.25\beta^{e} = 0.75-0.5\beta^{e}$. For environment 1, 2, 3 the above probability $\mathbb{P}(C^{e}=0|Y^e=1)$  is $0.7$, $0.65$ and $0.30$ respectively.  Next, we compute the probability $\mathbb{P}(Y^{e}|X^{e})$. Suppose $Y_g^{e}=1$ for the calculation below. We also assume deterministic labeling.

\begin{equation}
    \begin{split}
     &\mathbb{P}(Y^e=1,C^{e}=1|X_g^{e}) = \\& \mathbb{P}(Y^e=1,C^{e}=1|X_g^{e}) = \mathbb{P}(Y^{e}=1|Y_g^{e}=1) \mathbb{P}(C^{e}=1|N=0) = 0.75\beta^{e} \\ 
          &\mathbb{P}(Y^e=0,C^{e}=1|X_g^{e}) = \\& \mathbb{P}(Y^e=1,C^{e}=1|X_g^{e}) = \mathbb{P}(Y^{e}=0|Y_g^{e}=1) \mathbb{P}(C^{e}=1|N=1) = 0.25(1-\beta^{e}) \\ 
    & \mathbb{P}(Y^{e}=1|X^{e}) = \mathbb{P}(Y^e=1|X_g^{e},C^{e}=1) = \frac{3\beta^{e}}{2\beta^{e}+1}  \\
          & \mathbb{P}(Y^{e}=1|X^{e}) =  0.25\;\; \text{(For environment 1)} \\
      & \mathbb{P}(Y^{e}=1|X^{e}) =  0.428 \;\; \text{(For environment 2)} \\
          & \mathbb{P}(Y^{e}=1|X^{e}) = 0.964\;\; \text{(For environment 2)} 
    \end{split}
\end{equation}

\textbf{C.2 Graphical model for confounded CMNIST.}
In Figure \ref{fig:cfd_cmnist_combined}, we provide the graphical model for confounded CMNIST described in \eqref{cmnist_acausal_cfd_sem} (for $G=0$).
% \begin{figure}
%     \centering
%     \includegraphics[trim={0, 2in, 0 ,3in }, width=6in]{colored_mnist_cfdd_gm.pdf}
%     \caption{Confounder based colored MNIST}
%     \label{fig:cfd_cmnist}
% \end{figure}

\textbf{C.3. Results with numerical values and standard errors.}
In Table \ref{table_results_cmnits_cfd}, we provide the numerical values for the results showed in the Figure \ref{fig:cmnistvar} b) along with the standard errors.
\begin{table}[h!]
\centering
\def\arraystretch{1.5}
 \begin{tabular}{||c c c||} 
 \hline
Method & Number of samples  & Test error \\ [0.5ex] 
 \hline\hline
 ERM &  $1000$ &   $ 65.21 \pm 0.64$\\  
  \hline
 IRM &  $1000$ & $65.60 \pm 0.53$\\ 
 \hline
  ERM &  $5000$ &   $67.91 \pm 0.40$\\ \hline 
 IRM &  $5000$ &  $67.59 \pm 1.14$\\ 
 \hline
   ERM &  $10000$ &   $66.92 \pm 0.30$\\ \hline 
 IRM &  $10000$ &   $59.26 \pm 1.62$\\ \hline 
   ERM &  $30000$ &   $67.32 \pm 0.28$\\ \hline 
 IRM &  $30000$ &   $47.01 \pm 1.37$\\ \hline 
    ERM &  $60000$ &  $67.62 \pm 0.31$ \\ \hline 
 IRM &  $60000$ &   $44.92 \pm 0.84$\\ \hline 
\end{tabular}
 \caption{Comparison of ERM vs IRM: CF-CMNIST}
% \end{center}
\label{table_results_cmnits_cfd}
\end{table}

\textbf{D. Colored MNIST with anti-causal variables and confounded variables (HB-CMNIST)}

\textbf{D.1. Compute $\mathbb{P}(C^{e}|Y^e)$ and $\mathbb{P}(Y^{e}|X^{e})$.} 

We start by computing $\mathbb{P}(C^{e}|Y^e)$. Recall that $C^{e}  = G(Y^{e}\oplus N^{e}) + (1-G)(N\oplus N^{e})$. $G=1$ with probability $\theta$ and $0$ otherwise.  
\begin{equation}
    \mathbb{P}(C^{e}=1 |Y^{e}=1) = \frac{\mathbb{P}(C^{e}=1,Y^{e}=1)}{\mathbb{P}(C^{e}=1,Y^{e}=1) + \mathbb{P}(C^{e}=0,Y^{e}=1)}
\end{equation}

\begin{equation}
     \mathbb{P}(C^{e}=1 , Y^{e}=1) = \theta(1-\beta^{e}) + (1-\theta)(0.25 + 0.5\beta^{e}) 
\end{equation}
\begin{equation}
     \mathbb{P}(C^{e}=0 , Y^{e}=1) = \theta(\beta^{e}) + (1-\theta)(0.75 - 0.25\beta^{e}) 
\end{equation}
We used $\theta=0.8$ in the experiments. $\mathbb{P}(Y^{e}|X^{e})$ can be computed on the same lines as was shown for anti-causal and confounded model and it varies significantly across the environments.

\textbf{D.2 Graphical model for confounded CMNIST.}

In Figure \ref{fig:cfd_cmnist_combined}, we provide the graphical model for confounded CMNIST described in \eqref{cmnist_acausal_cfd_sem} (for  $0<\theta<1$).
% \begin{figure}
%     \centering
%     \includegraphics[trim={0, 2in, 0 ,3in }, width=6in]{colored_mnist_cfdd_chd_gm.pdf}
%     \caption{Anti-causal and confounded colored MNIST}
%     \label{fig:cfd_cfd_cmnist}
% \end{figure}

\textbf{D.3. Results with numerical values and standard errors.}
In Table \ref{table_results_cmnits_cfd_chd}, we provide the numerical values for the results showed in the Figure \ref{fig:cmnistvar} d) along with the standard errors.
\begin{table}[h!]
\centering
\def\arraystretch{1.5}
 \begin{tabular}{||c c c||} 
 \hline
Method & Number of samples  & Test error \\ [0.5ex] 
 \hline\hline
 ERM &  $1000$ &   $65.38 \pm 0.56$\\  
  \hline
 IRM &  $1000$ &   $66.98 \pm 0.61$\\ 
 \hline
  ERM &  $5000$ &   $62.92 \pm 0.99$\\ \hline 
 IRM &  $5000$ &  $61.09\pm 1.74$\\ 
 \hline
   ERM &  $10000$ &   $62.08 \pm 1.19$\\ \hline 
 IRM &  $10000$ &   $51.46 \pm 1.22$\\ \hline 
   ERM &  $30000$ &   $60.48 \pm 0.42$\\ \hline 
 IRM &  $30000$ &   $34.96 \pm 1.47$\\ \hline 
    ERM &  $60000$ &   $59.70 \pm 0.72$\\ \hline 
 IRM &  $60000$ &   $32.80 \pm 0.55$\\ \hline 
\end{tabular}
 \caption{Comparison of ERM vs IRM: HB-CMNIST}
% \end{center}
\label{table_results_cmnits_cfd_chd}
\end{table}
% \subsection{Classification:  supplementary results}
% Before discussing the supplementary results,  we also present the graphical models for the four types of models used below in Figure \ref{fig:sb_cmnist}, \ref{fig:cfd_cmnist}, \ref{fig:chd_cmnist}, \ref{fig:cfd_cfd_cmnist}. In Tables \ref{table_results_cmnits_sb}, \ref{table_results_cmnits_cfd}, \ref{table_results_cmnits_chd}, \ref{table_results_cmnits_cfd_chd}. we present the mean test error and standard error in the means for the figures shown in the main text. 

\begin{figure}
    \centering
    \includegraphics[trim=0 10em 0 0, width=5in]{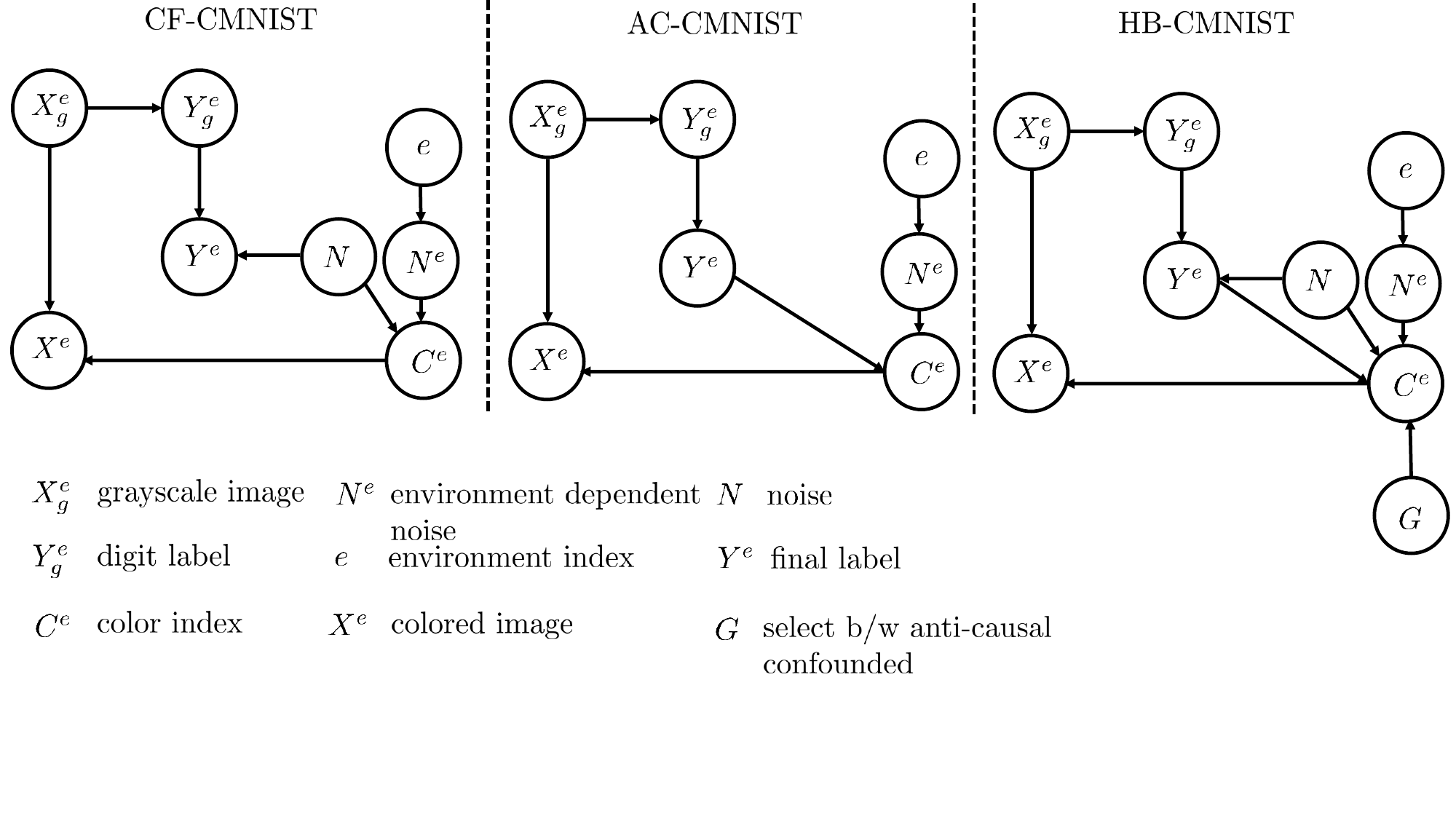}
    \caption{Graphical model for CF-CMNIST, AC-CMNIST, HB-CMNIST}
    \label{fig:cfd_cmnist_combined}
\end{figure}

\subsubsection{Regression} 
 We use the same structure for the generative model as described by \cite{arjovsky2019invariant}. We work with the four different variants on the same lines as CMNIST (covariate shift based, confounded, anti-causal, hybrid). The comparisons in \cite{arjovsky2019invariant} were for anti-causal and hybrid models. The general model is written as

\begin{equation}
    \begin{split}
      &  H^{e} \leftarrow \mathcal{N}(0, \sigma_e^{2}\mathsf{I}_{s}) \\ 
      &  X_1^{e} \leftarrow \mathcal{N}(0, \sigma_e^{2}\mathsf{I}_{s}) + W_{h \rightarrow 1}H^{e} \\ 
      &  Y^{e} \leftarrow W_{1\rightarrow y}^{e} X_{1}^{e} +  \mathcal{N}(0, \sigma_e^2)  + W_{h\rightarrow y} H^{e} \\
      & X_{2}^{e} \leftarrow W_{y\rightarrow 2}Y^{e} + \mathcal{N}(0, \mathsf{I}_s) + W_{h\rightarrow 2}H^{e}
    \end{split}
\end{equation}
$H^{e}$ is the hidden confounder, $X^{e} = [X_1^{e}$, $X_2^{e}]$ is the observed covariate vector, $Y^{e}$ is the label. Different $W$'s correspond to the weight vectors that multiply with the covariates and the confounders. The four datasets differ in the weight $W$ vectors and we describe them below. $\sigma_e$ is environment dependent standard deviation, we use two training environments with $\sigma_1=0.2$ and $\sigma_2=2.0$ respectively and both environments .  

\begin{itemize}
    \item Covariate shift case (CS-regresion): In this case, we  fix $W_{h \rightarrow 2}$, $W_{h\rightarrow y}$, and $W_{y\rightarrow 2}$ to zero and we draw each entry of  $W_{1\rightarrow y}$ from $ \frac{1}{s}\mathcal{N}(0,1)$ and set $W_{h \rightarrow 1}$ to identity.
    \item Confounded variable case (CF-regression): Set $W_{y\rightarrow 2}$ to zero and we draw each entry in $W_{1\rightarrow y}, W_{h \rightarrow 1}, W_{h \rightarrow 2}, W_{h\rightarrow y}$ from $ \frac{1}{s}\mathcal{N}(0,1)$.
    \item Anti-causal variable case (AC-regression): Set
$W_{h\rightarrow y}$,$W_{h\rightarrow 1}$ and $W_{h\rightarrow 2}$ to zero and draw each entry of  $W_{1\rightarrow y}$ and $W_{y\rightarrow 2}$ from $\frac{1}{s} \mathcal{N}(0,1)$.

\item Hybrid confounded and anti-causal variable case (HB-regression): Draw each
$W_{h\rightarrow y}$,$W_{h\rightarrow 1}$, $W_{h\rightarrow 2}$, $W_{1\rightarrow y}$ and $W_{y\rightarrow 2}$ from $\frac{1}{s} \mathcal{N}(0,1)$.
\end{itemize}
We also present the graphical models for the four types of models used  in Figure \ref{fig:graph_model_regn}. 

\begin{figure}
    \centering
    \includegraphics[trim={0, 3in, 0 ,0 }, width=6in]{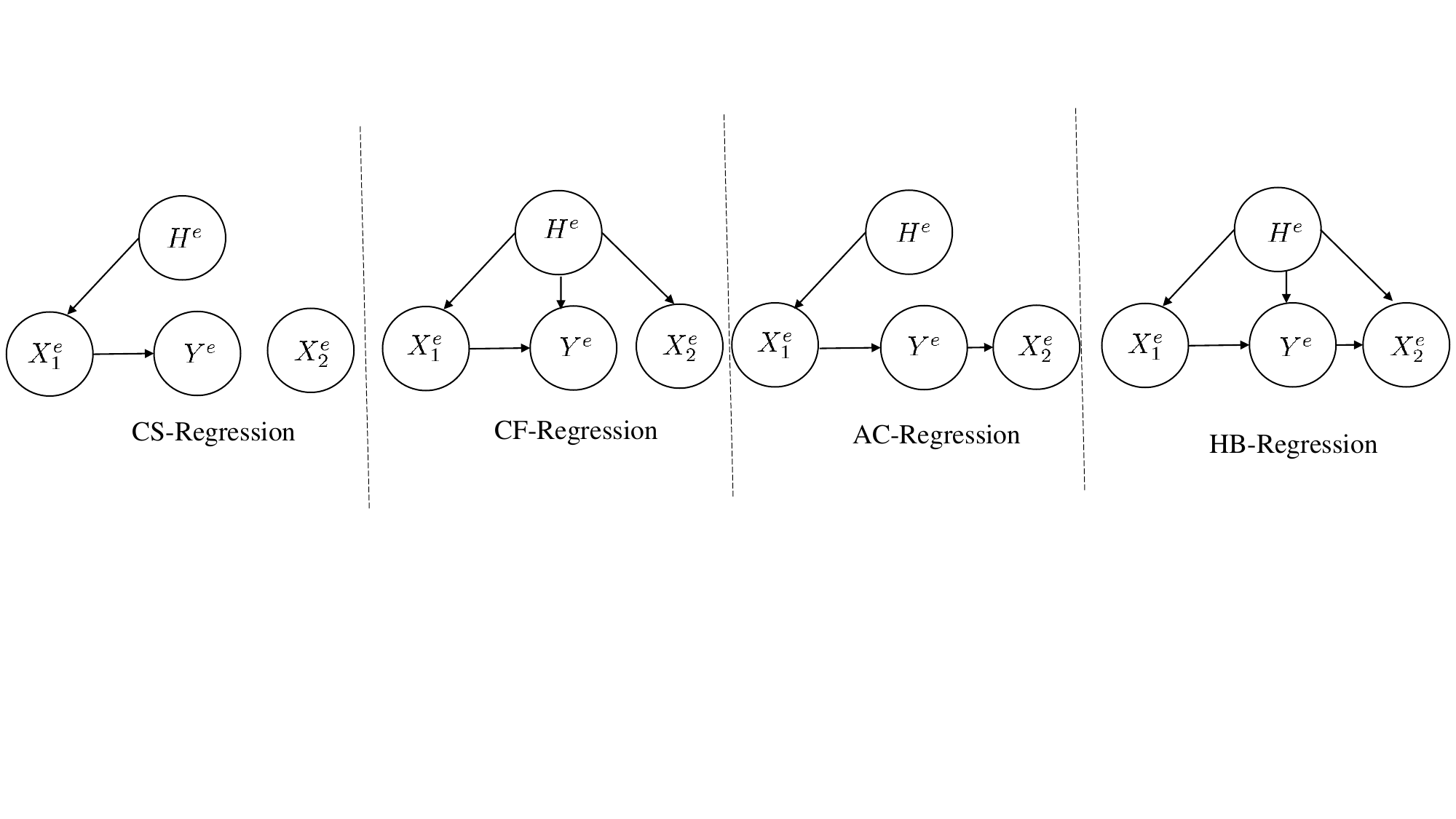}
    \caption{Graphical models for the different regression datasets}
    \label{fig:graph_model_regn}
\end{figure}

\subsubsection{Choice of $\mathcal{H}_{\Phi}$ and other training details}
We use a linear model that takes as input $X^{e}$. For ERM we use standard linear regression from sklearn. For IRM, we use 50k gradient steps with learning rate 1e-3, batch size is equal to the size of the training data. We use the train domain validation set procedure described by \cite{gulrajani2020search} to select the penalty value from the set 
$\{0, 1e-5, 1e-4, 1e-3, 1e-2, 1e-1\}$ (with 4:1 train-validation split). We average the results over $25$ trials.

\subsubsection{Results} We discuss results for the case when the length of the covariate vector $X^e$ is 10. The desired optimal invariant predictor is $W^{*} = [W_{1\rightarrow y},0]$. We will compare ERM and IRM in terms of the model estimation error, i.e., the distance between the model estimated by the method $\hat{W}$ and the true model given as  $\|\hat{W} - W^{*}\|^2$.
In Figures \ref{fig:dim10_sb}, \ref{fig:dim10_cfd}, \ref{fig:dim10_chd}, \ref{fig:dim10_chd_cfd}, we compare the model estimation error vs. the number of samples when the model. In these comparisons, we see that consistent with the classification experiments and predictions from Proposition \ref{prop4: emr_irm_real} in the covariate shift case (See Figure \ref{fig:dim10_sb}) there is no clear winner between the two approaches. There are gains from using IRM in the other cases (Figures \ref{fig:dim10_cfd}, \ref{fig:dim10_chd}, \ref{fig:dim10_chd_cfd}). However, in the confounder case in Figure \ref{fig:dim10_cfd}, the gains from IRM appear in the low sample regime but are not there in the high sample regime. This is because for this setup the asymptotic bias of ERM is also very small. In addition to the Figures \ref{fig:dim10_sb}, \ref{fig:dim10_cfd}, \ref{fig:dim10_chd}, \ref{fig:dim10_chd_cfd}, we   provide the tables (see Table \ref{table_suplement_reg10_real}, \ref{table_suplement_reg10_cfd}, \ref{table_suplement_reg10_chd}, \ref{table_suplement_reg10_cfd_chd}) with the numerical values for the mean model estimation (and the standard error) error shown in the figures.

% We also carried out the experiments for the case when the length of covariate vector is $50$.  In Figures \ref{fig:dim50_sb}, \ref{fig:dim50_cfd}, \ref{fig:dim50_chd}, \ref{fig:dim50_chd_cfd}, we present the plots for model estimation error vs. the number of samples for the case when the dimension of the covariate vector is $50$. For both 10 and 50 dimensional case, 

\begin{figure}[!h]
\centering
\begin{minipage}{.5\textwidth}
  \centering
  \includegraphics[width=3in]{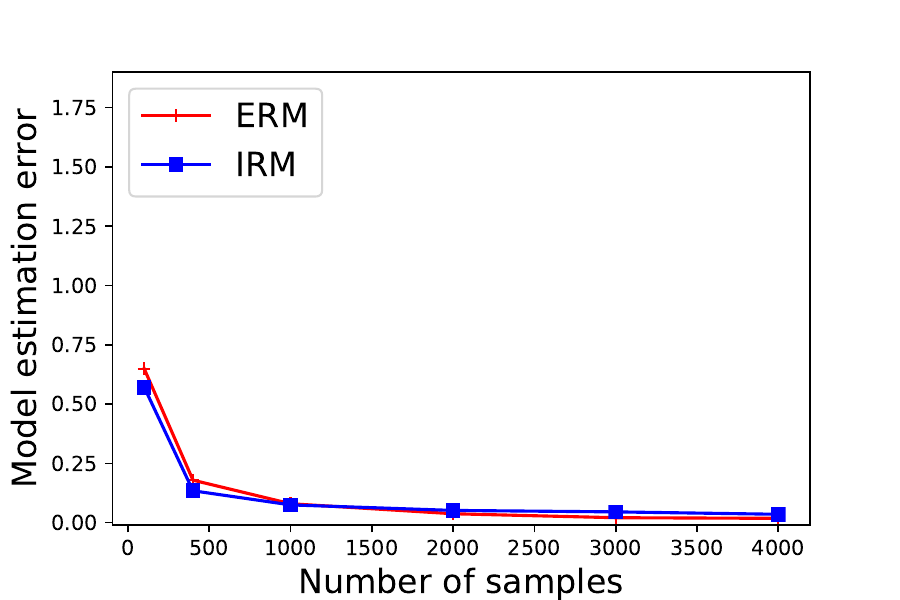}
   \captionof{figure}{Comparisons: $n= 10$ CS-regression}
  \label{fig:dim10_sb}
\end{minipage}%
\begin{minipage}{.5\textwidth}
  \centering
  \includegraphics[width=3in]{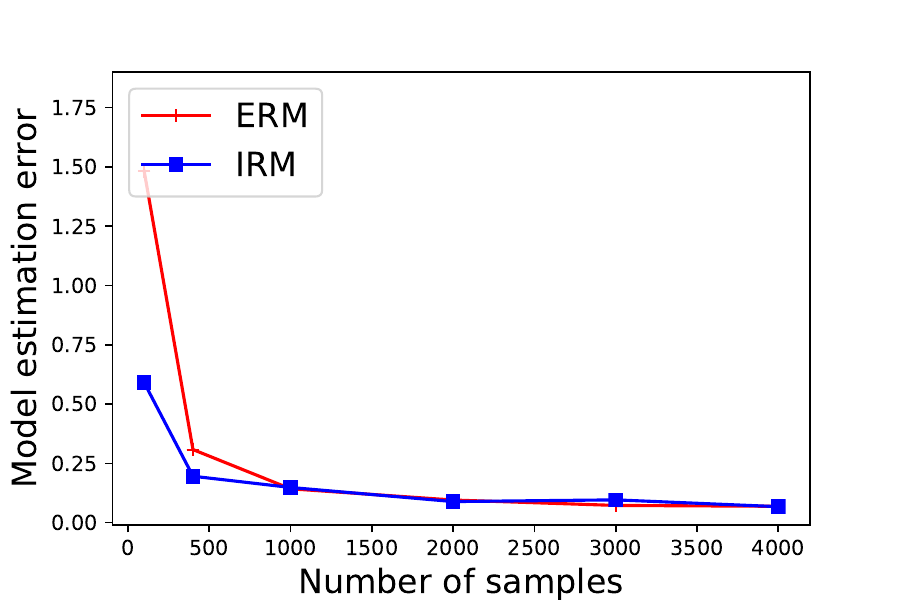}
   \captionof{figure}{Comparisons: $n= 10$ CF-regression}
  \label{fig:dim10_cfd}
\end{minipage}
\end{figure}
\begin{figure}[!h]
\centering
\begin{minipage}{.5\textwidth}
  \centering
  \includegraphics[width=3in]{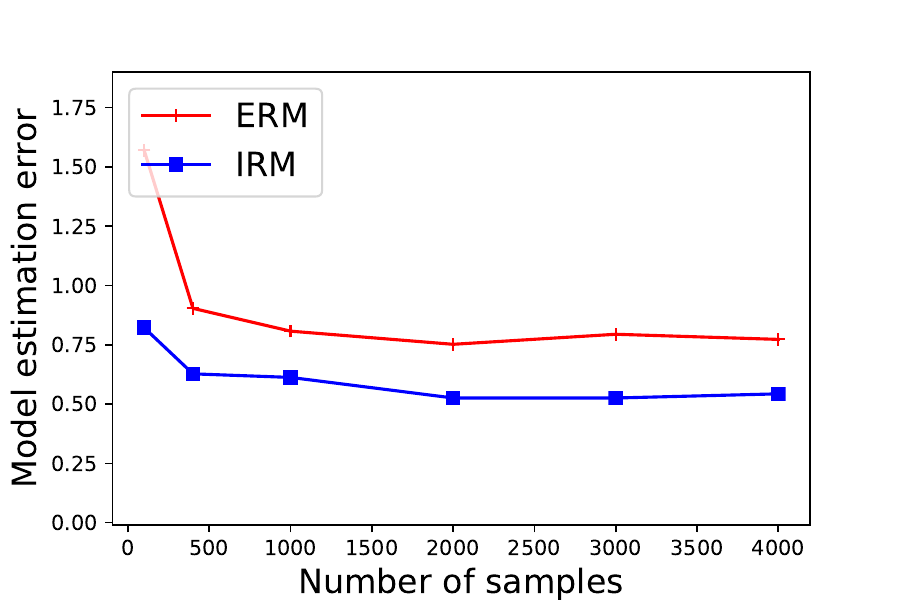}
   \captionof{figure}{Comparisons: $n= 10$ AC-regression}
  \label{fig:dim10_chd}
\end{minipage}%
\begin{minipage}{.5\textwidth}
  \centering
  \includegraphics[width=3in]{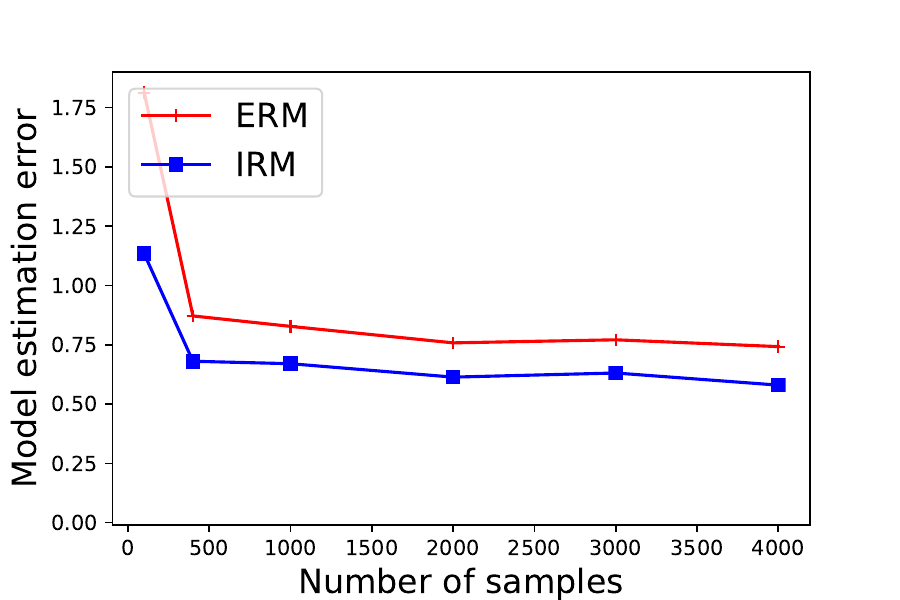}
   \captionof{figure}{Comparisons: $n= 10$ HB-regression}
  \label{fig:dim10_chd_cfd}
\end{minipage}
\end{figure}

% \subsubsection{Regression: supplementary results}
% Before discussing the supplementary results, 

% \begin{figure}[!h]
% \centering
% \begin{minipage}{.5\textwidth}
%   \centering
%   \includegraphics[width=3in]{dim_50_sb_rescaled_new.pdf}
%   \captionof{figure}{Comparisons: $n= 50$ covariate shift}
%   \label{fig:dim50_sb}
% \end{minipage}%
% \begin{minipage}{.5\textwidth}
%   \centering
%   \includegraphics[width=3in]{dim_50_cfd_rescaled_new.pdf}
%   \captionof{figure}{Comparisons: $n= 50$ confounded}
%   \label{fig:dim50_cfd}
% \end{minipage}
% \end{figure}
% \begin{figure}[!h]
% \centering
% \begin{minipage}{.5\textwidth}
%   \centering
%   \includegraphics[width=3in]{dim_50_chd_rescaled_new.pdf}
%      \captionof{figure}{Comparisons: $n= 50$ anti-causal}
%   \label{fig:dim50_chd}
% \end{minipage}%
% \begin{minipage}{.5\textwidth}
%   \centering
%   \includegraphics[width=3in]{dim_50_chd_rescaled_new.pdf}
%       \captionof{figure}{Comparisons: $n= 50$ hybrid}
%   \label{fig:dim50_chd_cfd}
% \end{minipage}
% \end{figure}

\begin{table}[h!]
\centering
\def\arraystretch{1.5}
 \begin{tabular}{||c c c||} 
 \hline
Method & Number of samples  & Model estimation error \\ [0.5ex] 
 \hline\hline
 ERM &  $50$ &   $0.65 \pm 0.06$\\  
  \hline
 IRM &  $50$ & $0.57 \pm 0.10$\\ 
 \hline
  ERM &  $200$ &   $0.18 \pm 0.02$\\ \hline 
 IRM &  $200$ & $0.13 \pm 0.02$\\ 
 \hline
   ERM &  $500$ &   $0.08 \pm 0.005$\\ \hline 
 IRM &  $500$ & $0.08 \pm 0.014$\\ \hline 
   ERM &  $1000$ &   $0.037 \pm 0.004$\\ \hline 
 IRM &  $1000$ & $0.051 \pm 0.007$\\ \hline 
    ERM &  $1500$ &   $0.021 \pm 0.002$\\ \hline 
 IRM &  $1500$ & $0.045 \pm 0.007$\\ \hline 
     ERM &  $2000$ &   $0.018 \pm 0.002$\\ \hline 
 IRM &  $2000$ & $0.035 \pm 0.006$\\ \hline 
\end{tabular}
 \caption{Comparison of ERM vs IRM: $n=10$ CS-regression}
% \end{center}
\label{table_suplement_reg10_real}
\end{table}

\begin{table}[h!]
\centering
\def\arraystretch{1.5}
 \begin{tabular}{||c c c||} 
 \hline
Method & Number of samples  & Model estimation error \\ [0.5ex] 
 \hline\hline
 ERM &  $50$ &   $1.48 \pm 0.15$\\  
  \hline
 IRM &  $50$ & $0.59 \pm 0.10$\\ 
 \hline
  ERM &  $200$ &   $0.30 \pm 0.03$\\ \hline 
 IRM &  $200$ & $0.20 \pm 0.03$\\ 
 \hline
   ERM &  $500$ &   $0.14 \pm 0.01$\\ \hline 
 IRM &  $500$ & $0.15 \pm 0.02$\\ \hline 
   ERM &  $1000$ &   $0.10 \pm 0.02$\\ \hline 
 IRM &  $1000$ & $0.09 \pm 0.01$\\ \hline 
    ERM &  $1500$ &   $0.07 \pm 0.01$\\ \hline 
 IRM &  $1500$ & $0.10 \pm 0.01$\\ \hline 
     ERM &  $2000$ &   $0.07 \pm 0.01$\\ \hline 
 IRM &  $2000$ & $0.07 \pm 0.01$\\ \hline 
\end{tabular}
 \caption{Comparison of ERM vs IRM: $n=10$ CF-regression}
% \end{center}
\label{table_suplement_reg10_cfd}
\end{table}

\begin{table}[h!]
\centering
\def\arraystretch{1.5}
 \begin{tabular}{||c c c||} 
 \hline
Method & Number of samples  & Model estimation error \\ [0.5ex] 
 \hline\hline
 ERM &  $50$ &   $1.57 \pm 0.18$\\  
  \hline
 IRM &  $50$ & $0.82 \pm 0.20$\\ 
 \hline
  ERM &  $200$ &   $0.90 \pm 0.08$\\ \hline 
 IRM &  $200$ & $0.63 \pm 0.08$\\ 
 \hline
   ERM &  $500$ &   $0.81 \pm 0.08$\\ \hline 
 IRM &  $500$ & $0.61 \pm 0.05$\\ \hline 
   ERM &  $1000$ &   $0.75 \pm 0.07$\\ \hline 
 IRM &  $1000$ & $0.53 \pm 0.04$\\ \hline 
    ERM &  $1500$ &   $0.79 \pm 0.07$\\ \hline 
 IRM &  $1500$ & $0.53 \pm 0.05$\\ \hline 
     ERM &  $2000$ &   $0.77 \pm 0.08$\\ \hline 
 IRM &  $2000$ & $0.54 \pm 0.04$\\ \hline 
\end{tabular}
 \caption{Comparison of ERM vs IRM: $n=10$ AC-regression}
% \end{center}
\label{table_suplement_reg10_chd}
\end{table}
\begin{table}[h!]
\centering
\def\arraystretch{1.5}
 \begin{tabular}{||c c c||} 
 \hline
Method & Number of samples  & Model estimation error \\ [0.5ex] 
 \hline\hline
 ERM &  $50$ &   $1.81 \pm 0.19$\\  
  \hline
 IRM &  $50$ & $1.13 \pm 0.18$\\ 
 \hline
  ERM &  $200$ &   $0.87 \pm 0.09$\\ \hline 
 IRM &  $200$ & $0.68 \pm 0.10$\\ 
 \hline
   ERM &  $500$ &   $0.83 \pm 0.09$\\ \hline 
 IRM &  $500$ & $0.67 \pm 0.07$\\ \hline 
   ERM &  $1000$ &   $0.76 \pm 0.07$\\ \hline 
 IRM &  $1000$ & $0.61 \pm 0.05$\\ \hline 
    ERM &  $1500$ &   $0.77 \pm 0.07$\\ \hline 
 IRM &  $1500$ & $0.63 \pm 0.05$\\ \hline 
     ERM &  $2000$ &   $0.74 \pm 0.07$\\ \hline 
 IRM &  $2000$ & $0.58 \pm 0.05$\\ \hline 
\end{tabular}
 \caption{Comparison of ERM vs IRM: $n=10$ HB-regression}
% \end{center}
\label{table_suplement_reg10_cfd_chd}
\end{table}

% \textbf{Illustrating IRM and OOD connection.}   We use a simplified version of the model described by \cite{peters2015causal}. In each environment $e$, the random variable $X^{e}=[X_1^{e},...,X_{n}^{e}]$ corresponds to the feature vector and $Y^{e}$ corresponds to the label. The data for each environment is generated by i.i.d. sampling $(X^{e},Y^{e})$ from the following generative model. Assume a subset $S^{*} \subset \{1,...,n\}$ is causal for the label $Y^{e}$.  For each environment $e\in \mathcal{E}_{all}$, $X^e$ has an arbitrary distribution and $Y^{e} \leftarrow g(X^{e}_{S^{*}}) + \upsilon^{e}$
% where $X^{e}_{S^{*}}$ is the vector $X^{e}$ with indices in $S^{*}$,   $g:\mathbb{R}^{|S^{*}|} \rightarrow \mathbb{R}$ is a function to describe the conditional expectation and  $\mathbb{E}^{e}[\upsilon^e]=0$, $\mathbb{E}^{e}[(\upsilon^e)^{2}]=\sigma^2$, $\upsilon^{e} \perp  X^{e}_{S^{*}}$. We fix the representation $\Phi^{*}(X^{e}) = X^{e}_{S^{*}}$. From Proposition \ref{assm1:ood_cond_envmt}, if $g\circ \Phi^{*} \in \mathcal{H}_{w\circ \Phi}$, then $g\circ \Phi^{*}$ solves the OOD problem in \eqref{eqn1: ood}. Use $\Phi^{*}$ as the representation in the constraint in \eqref{eqn: IRM}, the optimal classifier $w$ among all the functions is $g$. Since the optimal classifier does not vary across environments,  $g$ is an invariant predictor across all $\mathcal{E}_{tr}$. IRM framework hopes that by enforcing the constraint in \eqref{eqn: IRM} predictor $g$ is selected. 
\pagebreak
\subsection{Proofs for the Propositions}

In the results to follow, we will rely on Hoeffding's inequality. We restate the inequality below for convenience. 

\begin{lemma} \label{lemma1: hoeffding} (Hoeffding's inequality). Let $\theta_1,\dots \theta_m$ be a sequence of i.i.d. random variables and assume that for all $i$, $\mathbb{E}[\theta_i]=\mu$ and $\mathbb{P}[a \leq \theta_i\leq b] = 1$. Then, for any $\epsilon>0$
$$\mathbb{P}\bigg[\frac{1}{m}\big|\sum_{i=1}^{m}\theta_i - \mu\big|>\epsilon\bigg] \leq 2 \exp(-2\frac{m \epsilon^2}{(b-a)^2})$$

\end{lemma}

We restate the propositions from the main body. Next, we prove Proposition \ref{prop1: ood} from the main body of the manuscript. 
\begin{proposition}
\label{prop1: ood_supplement}
If $\ell$ is square loss, and Assumptions \ref{assm1:ood_cond_envmt},\ref{assm2:ood_cond_envmt} hold, then  $m\circ \Phi^{*}$  solves the OOD problem (\eqref{eqn1: ood}).

% If Assumptions \ref{assm1:ood_cond_envmt},\ref{assm2:ood_cond_envmt} hold and $\ell$ is square loss, then  $m\circ \Phi^{*}$  solves the OOD problem in \eqref{eqn1: ood}.
\end{proposition}

\begin{proof}
Define a predictor $w\circ \Phi^{*}$, where $\Phi^{*}$ is defined in Assumption \ref{assm1:ood_cond_envmt}. Let us simplify the expression for the risk for this predictor $R^{e}(w\circ \Phi^{*})$ using square loss for $\ell$. Recall $Z^{e} = \Phi^{*}(X^{e})$

\begin{equation}
\begin{split}
   & R^{e}(w\circ \Phi^{*}) = \mathbb{E}^{e}\bigg[\Big(Y^{e}- \mathbb{E}^{e}\big[Y^{e}\big|Z^{e}\big] +\mathbb{E}^{e}\big[Y^{e}\big|Z^{e}\big]  -(w\circ \Phi^{*})(X^{e})\Big)^2\bigg]\\ &= \mathbb{E}^{e}\bigg[\big(Y^{e}- \mathbb{E}^{e}\big[Y^{e}\big|Z^{e}\big]\big)^2\bigg] + \mathbb{E}^{e}\bigg[\big( \mathbb{E}^{e}\big[Y^{e}\big|Z^{e}\big]-(w\circ \Phi^{*})(X^{e})\big)^2\bigg]  + \\ & \;\;\;\;\; 2\mathbb{E}^{e}\bigg[\big(Y^{e}- \mathbb{E}^{e}\big[Y^{e}\big|Z^{e}\big]\big)\big( \mathbb{E}^{e}\big[Y^{e}\big|Z^{e}\big]-(w\circ \Phi^{*})(X^{e})\big)\bigg] \\ 
 &=   \mathbb{E}^{e}\bigg[\big(Y^{e}- m(Z^{e})\big)^2\bigg] + \mathbb{E}^{e}\bigg[\big( m(Z^{e})-w(Z^{e})\big)^2\bigg] + 2\mathbb{E}^{e}\bigg[\big(Y^{e}- m(Z^{e})\big)\big(m(Z^{e})-w(Z^{e})\big)\bigg] \\ 
  &=   \mathbb{E}^{e}\bigg[\big(Y^{e}- m(Z^{e})\big)^2\bigg] + \mathbb{E}^{e}\bigg[\big( m(Z^{e})-w(Z^{e})\big)^2\bigg]     \\ 
  &= \xi^{2} + \mathbb{E}^{e}\bigg[\big( m(Z^{e})-w(Z^{e})\big)^2\bigg] 
\end{split}
\label{proof_prop1:eqn1}
\end{equation}

In the above simplification in \eqref{proof_prop1:eqn1}, we use the following \eqref{proof_prop1:eqn2} and \eqref{proof_prop1:eqn3}, which rely on the law of total expectation. 
\begin{equation}
    \begin{split}
      &  \mathbb{E}^{e}\bigg[\Big(Y^{e}- m(Z^{e})\Big)\Big( m(Z^{e})-w(Z^{e})\Big)\bigg]   =  \mathbb{E}^{e}\bigg[\mathbb{E}^{e}\Big[\Big(Y^{e}- m(Z^{e})\Big)\Big(m(Z^{e}
      )-w(Z^{e})\Big)\Big|Z^{e}\Big]\bigg] \\
      & = \mathbb{E}^{e}\bigg[\Big(\mathbb{E}^{e}\big[Y^{e}\big|Z^{e}\big]- m(Z^{e})\Big)\Big( m(Z^{e})-w(Z^{e})\Big)\bigg] =0
    \end{split}
    \label{proof_prop1:eqn2}
\end{equation}

\begin{equation}
    \begin{split}
        \mathbb{E}^{e}\bigg[\Big(Y^{e}- m(Z^{e})\Big)^2\bigg]  =  \mathbb{E}^{e}\bigg[\mathbb{E}^{e}\Big[\Big(Y^{e}- m(Z^{e})\Big)^2\Big|Z^{e}\Big]\bigg] = \mathbb{E}^{e}\bigg[\mathsf{Var}^{e}\big[Y^{e}\big|Z^{e}\big]\bigg] = \xi^2
    \end{split}
        \label{proof_prop1:eqn3}
\end{equation}
In the last equality in \eqref{proof_prop1:eqn3}, we use the Assumption \ref{assm1:ood_cond_envmt} and obtain $\xi^2$.  Therefore, substituting $w=m$ in \eqref{proof_prop1:eqn1} achieves a risk of $\xi^2$ for all the environments. 
\begin{equation}
    \forall e \in \mathcal{E}_{all},\; R^{e}(m\circ\Phi^{*}) = \xi^2
    \label{proof_prop1:eqn9}
\end{equation}

% \begin{equation}
% \begin{split}
%     R^{e}(w\circ \Phi^{*})
%     \end{split}
%     \label{proof_prop1:eqn4}
% \end{equation}
% This establishes that $m$ is an invariant predictor across all the environments in $\mathcal{E}_{tr}_{all}$ with $\Phi^{*}$ as the representation. The risk achieved by $m\circ\Phi^{*}$ is $\sigma^{2}$ for all the environments and $m\circ \Phi^{*} \in \mathcal{S}^{\mathsf{IV}}\subseteq \mathcal{H}_{w\circ \Phi}$. 
Consider the environment $q$ that satisfies Assumption \ref{assm2:ood_cond_envmt}. Let us simplify the expression for the risk achieved by every predictor $f\in \mathcal{F}$ in the  environment $q$  following the steps similar to \eqref{proof_prop1:eqn1}. 
\begin{equation}
\begin{split}
    &R^{q}(f) = \mathbb{E}^{q}\bigg[\Big(Y^{q}- \mathbb{E}^{q}[Y^{q}|Z^{q}]\Big)^2\bigg] +\\& \mathbb{E}^{q}\bigg[\Big( \mathbb{E}^{q}[Y^{q}|Z^{q}]-f(X^{q})\Big)^2\bigg]   + 2\mathbb{E}^{q}\bigg[\big(Y^{q}- \mathbb{E}^{q}[Y^{q}|Z^{q}]\big)\Big( \mathbb{E}^{q}[Y^{q}|Z^{q}]-f(X^{q})\Big)\bigg] \\ 
 &=   \mathbb{E}^{q}\bigg[\Big(Y^{q}- m(Z^{q})\Big)^2\bigg] + \mathbb{E}^{q}\bigg[\big( m(Z^{q})-f(X^{q})\Big)^2\bigg]   + 2\mathbb{E}^{q}\bigg[\Big(Y^{q}- m(Z^{q})\Big)\Big(m(Z^{q})-f(X^{q})\Big)\bigg] \\ 
 & =   \mathbb{E}^{q}\bigg[\Big(Y^{q}- m(Z^{q})\Big)^2\bigg] + \mathbb{E}^{q}\bigg[\Big( m(Z^{q})-f(X^{q})\Big)^2\bigg]     \\ 
  &= \xi^{2} + \mathbb{E}^{q}\bigg[\Big( m(Z^{q})-f(X^{q})\Big)^2\bigg] 
\end{split}
\label{proof_prop1:eqn5}
\end{equation}

In the above simplification in \eqref{proof_prop1:eqn5}, we use the following \eqref{proof_prop1:eqn6}

\begin{equation}
    \begin{split}
      &  \mathbb{E}^{q}\bigg[\Big(Y^{q}- m(Z^{q})\Big)\Big( m(Z^{q})-f(X^q))\Big)\bigg]   =  \mathbb{E}^{q}\bigg[\mathbb{E}^{e}\Big[\Big(Y^{e}- m(Z^{q})\Big)\Big(m(Z^{q}
      )-f(X^{q})\Big)\Big|Z^{q}\Big]\bigg]\\
      & = \mathbb{E}^{q}\bigg[\Big(\mathbb{E}^{q}\big[Y^{q}\big|Z^{q}\big]- m(Z^{q})\Big)\Big( m(Z^{q})-f(X^{q})\Big)\bigg] =0\; \text{(Last equality follows from the  Assumption \ref{assm2:ood_cond_envmt})}
    \end{split}
    \label{proof_prop1:eqn6}
\end{equation}

Therefore,  for environment $q$ satisfying Assumption \ref{assm2:ood_cond_envmt} from \eqref{proof_prop1:eqn5}, it follows that for all $f\in \mathcal{F}$, $R^{q}(f)\geq \xi^{2}$. Therefore, we can write that 

\begin{equation}
\forall f \in \mathcal{F},\;  \max_{e\in \mathcal{E}_{all}}R^{e}(f)  \geq R^{q}(f) \geq \xi^2
  \label{proof_prop1:eqn8}
\end{equation}

Therefore, from \eqref{proof_prop1:eqn8} it directly follows that 
\begin{equation}
  \min_{f\in \mathcal{F} }  \max_{e\in \mathcal{E}_{all}}R^{e}(f)  \geq \xi^2
  \label{proof_prop1:eqn7}
\end{equation}

We showed in \eqref{proof_prop1:eqn9}, $R^{e}(m\circ \Phi^{*}) = \xi^{2}$ for all the environments.

Hence, $f=m\circ \Phi^{*}$ achieves the RHS of \eqref{proof_prop1:eqn7}. This completes the proof. 
\end{proof}

Some of the proofs that we describe next take a few  intermediate steps to build. Here we give a brief preview of the key ingredients that we developed to build these propositions.
\begin{itemize}
    \item In Proposition \ref{prop2: scomp_irm}, our goal is to carry out a sample complexity analysis of EIRM in the same spirit as ERM. However, there are two key challenges that we are faced with -  i) the IRM penalty $R^{'}$ is not separable (as it is composed of terms involving squares of expectations) and ii) unlike ERM, IRM is a constrained optimization problem. To deal with i), we develop an estimator in the next section that allows us to re-express IRM penalty in a separable fashion.  To deal with ii), we define a parameter $\kappa$ that measures the minimum separation between IRM penalty and $\epsilon$. We show that as long as this separation for all the predictors in the hypothesis class is positive, then we can rely on $\kappa$-representative property from \cite{shalev2014understanding} applied to the new estimator that we build to show that the set of empirical invariant predictors $\hat{\mathcal{S}}^{\mathsf{IV}}(\epsilon)$ are the same as exact invariant predictors $\mathcal{S}^{\mathsf{IV}}(\epsilon)$. 
    \item In Proposition \ref{prop5: irm_cfd_chd}, our goal is to show that approximate OOD can be achieved by IRM in the finite sample regime. This result builds on the infinite sample result from \cite{arjovsky2019invariant}. In Theorem 9 in \cite{arjovsky2019invariant}, it was shown that for linear models (defined in Assumption \ref{assm6: linear_model}) obeying linear general position, if the gradient constraints in the exact gradient constraint IRM \eqref{eqn: IRM_grad_cons} are satisfied, then the OOD solution is achieved. We extend this result  to show that if the constraints in the approximate $\epsilon$ penalty based IRM in \eqref{eqn: IRM_cons_swap} are satisfied, then we are guaranteed to be in the $\sqrt{\epsilon}$ neighborhood of the OOD solution. Note that  this result is again in the infinite sample regime as it proves the approximation for solutions of the problem \eqref{eqn: IRM_cons_swap}, which involves expectations w.r.t true distributions. Next, we exploit similar tools that we introduced to prove Proposition \ref{prop2: scomp_irm} to also prove the finite sample extension. 
    \item In later sections, we show the generalizations to infinite hypothesis classes. In particular, we  focus on parametric model families that are Lipschitz continuous. The extension to infinite hypothesis classes is based on carefully exploiting the covering number based techniques \cite{shalev2014understanding} for the IRM penalty estimator that we introduced.  We also provide generalizations of the results for linear models to polynomial models. To arrive at these results, we exploit some standard properties of tensor products.
\end{itemize}
% Before proving the next Propositions, we need to define the empirical estimator of $R^{'}$ and also provide some lemmas that serve as building blocks. 

\subsubsection{Empirical Estimator of $R^{'}$}
Next, we define an estimator for $R^{'}$. We first simplify $R^{'}$ as follows. 

Observe that $$\nabla_{w|w=1.0} R^{e}(w\cdot\Phi) = \frac{\partial \mathbb{E}^{e}\big[\ell(w\cdot\Phi(X^{e}), Y^{e})\big]}{\partial w}\Big|_{w=1.0}= \mathbb{E}^{e}\bigg[\frac{\partial \ell(w\cdot\Phi(X^{e}), Y^{e})}{\partial w}\Big|_{w=1.0}\bigg]$$ and 
\begin{equation}
    \|\nabla_{w|w=1.0} R^{e}(w\cdot\Phi)\|^2 = \Big(\frac{\partial \mathbb{E}^{e}\big[\ell(w\cdot\Phi(X^{e}), Y^{e})\big]}{\partial w}\Big|_{w=1.0}\Big)^2 = \Big(\mathbb{E}^{e}\bigg[\frac{\partial \ell(w\cdot\Phi(X^{e}), Y^{e})}{\partial w}\Big|_{w=1.0}\bigg]\Big)^2
\end{equation}
In the above simplification, we used Leibniz integral rule and take the derivative inside the expectation.

Also we can write $\mathbb{E}[X]^2 = \mathbb{E}[AB]$, where $A$ and $B$ are independent and identical random variables with same distribution as $X$.  Therefore, we consider two independent data points $(X^e, Y^e)\sim \mathbb{P}^{e}$ and $(\tilde{X}^{e}, \tilde{Y}^{e})\sim \mathbb{P}^{e}$. 
\begin{equation}
    \|\nabla_{w|w=1.0} R^{e}(w\cdot\Phi)\|^2 =  \mathbb{E}^{e}\bigg[\Big(\frac{\partial \ell(w\cdot\Phi(X^{e}), Y^{e})}{\partial w}\Big|_{w=1.0}\Big)\Big(\frac{\partial \ell(w\cdot\Phi(\tilde{X}^{e}), \tilde{Y}^{e})}{\partial w}\Big|_{w=1.0}\Big)\bigg]
    \label{estimate_rp_eqn1}
\end{equation}

In the above the expectation $\mathbb{E}^{e}$ is taken over the joint distribution over pairs of distributions of pairs $(X^e, Y^e),(\tilde{X}^{e}, \tilde{Y}^{e})$ from the same environment $e$.

We write 

\begin{equation}
\begin{split}
&   R^{'}(\Phi) =  \sum_{e\in \mathcal{E}_{tr}} \pi^{e} \|\nabla_{w|w=1.0} R^{e}(w\cdot\Phi)\|^2     \\ & 
= \sum_{e\in \mathcal{E}_{tr}} \pi^{e}  \mathbb{E}^{e}\bigg[\Big(\frac{\partial \ell(w\cdot\Phi(X^{e}), Y^{e})}{\partial w}\Big|_{w=1.0}\Big)\Big(\frac{\partial \ell(w\cdot\Phi(\tilde{X}^{e}), \tilde{Y}^{e})}{\partial w}\Big|_{w=1.0}\Big)\bigg]
\label{estimate_rp_eqn2}
\end{split}
\end{equation}
In the above simplification, we used  \eqref{estimate_rp_eqn1}. Define a joint distribution $\tilde{\mathbb{P}}$  over the tuple $(e,(X^e, Y^e),(\tilde{X}^e, \tilde{Y}^e))$, where $e\sim \{\pi^{o}\}_{o\in \mathcal{E}_{tr}}$, $(X^e, Y^e)\sim \mathbb{P}^{e}$ and $(\tilde{X}^{e}, \tilde{Y}^{e})\sim \mathbb{P}^{e}$. Also, 
\begin{equation}
    \tilde{\mathbb{P}}\Big((e,(X^e, Y^e),(\tilde{X}^e, \tilde{Y}^e))\Big) = \pi^{e}\mathbb{P}^{e}(X^e, Y^e)\mathbb{P}^{e}(\tilde{X}^e, \tilde{Y}^e) 
    \label{estimate_rp_eqn7}
\end{equation} 

% Since the environment itself is drawn at random, define a joint distribution $\tilde{\mathcal{D}}$ over $E$ the environment, and the pair of observations from that environment $((X^e, Y^e),(\tilde{X}^e, \tilde{Y}^e))$. 

We rewrite the above expression \eqref{estimate_rp_eqn2}  in terms of an expectation w.r.t $\tilde{\mathbb{P}}$, which we represent as $\tilde{\mathbb{E}}$ 
follows 
\begin{equation}
     R^{'}(\Phi) = \tilde{\mathbb{E}}\bigg[\Big(\frac{\partial \ell(w\cdot\Phi(X^{e}), Y^{e})}{\partial w}\Big|_{w=1.0}\Big)\Big(\frac{\partial \ell(w\cdot\Phi(\tilde{X}^{e}), \tilde{Y}^{e})}{\partial w}\Big|_{w=1.0}\Big)\bigg]
     \label{estimate_rp_eqn3}
\end{equation}

Define \begin{equation}
   \ell^{'}\Big(h, \big((X^{e}, Y^{e}), (\tilde{X}^{e}, \tilde{Y}^{e})\big)\Big) = \Big(\frac{\partial \ell(w\cdot\Phi(X^{e}), Y^{e})}{\partial w}\Big|_{w=1.0}\Big)\Big(\frac{\partial \ell(w\cdot\Phi(\tilde{X}^{e}), \tilde{Y}^{e})}{\partial w}\Big|_{w=1.0}\Big)
    \label{estimate_rp_eqn4}
\end{equation}

Substitute \eqref{estimate_rp_eqn4} in \eqref{estimate_rp_eqn3} to obtain 
\begin{equation}
     R^{'}(\Phi) = \tilde{\mathbb{E}}\bigg[ \ell^{'}\Big(h, \big((X^{e}, Y^{e}), (\tilde{X}^{e}, \tilde{Y}^{e})\big)\Big)\bigg]
     \label{estimate_rp_eqn5}
\end{equation}
We  construct a simple estimator $\hat{R}^{'}(\Phi)$ by pairing the data points in each environment. For simplicity assume that each environment has even number of points. In environment $e$, which has $n_{e}$ points we construct $\frac{n_e}{2}$ pairs. Define a set of such pairs as 

\begin{equation}
\tilde{D} = \{\{(x_{2i-1}^{e}, y_{2i-1}^{e}), (x_{2i}^{e}, y_{2i}^{e})\}_{i=1}^{\frac{n_e}{2}}\}_{e \in \mathcal{E}_{tr}}
 \label{estimate_rp_eqn6}
\end{equation}

\begin{equation}
     \hat{R}^{'}(\Phi) = \frac{2}{|D|}\sum_{e\in \mathcal{E}_{tr}}\sum_{i=1}^{\frac{n_e}{2}}\ell^{'}\Big(h, \big((X_{2i-1}^{e}, Y_{2i-1}^{e}), (\tilde{X}_{2i}^{e}, \tilde{Y}_{2i}^{e})\big)\Big)
      \label{eqn: rprime_estimator}
\end{equation}
% \begin{equation}
%      \hat{R}^{'}(\Phi) = \frac{2}{|\mathcal{E}_{tr}|n_e}\sum_{e}\sum_{i}\frac{\partial \ell(w.\Phi(X_i^{e}), Y_i^{e})}{\partial w}|_{w=1.0}\frac{\partial \ell(w.\Phi(X_{i+1}^{e}), Y_{i+1}^{e})}{\partial w}|_{w=1.0}
%       \label{eqn: rprime_estimator}
% \end{equation}

There can be other estimators of $R^{'}(\Phi)$ where we separately estimate each term $\|\nabla_{w|w=1.0} R^{e}(w.\Phi)\|^2$ and $\pi^{e}$ in the summation. We rely on the above estimator \eqref{eqn: rprime_estimator} as its separability allows us to use standard concentration inequalities, e.g., Hoeffding's inequality. 

% \begin{equation}
% \begin{split}
% &   R^{'}(\Phi) =  \sum_{e\in \mathcal{E}_{tr}} \pi^{e} \|\nabla_{w|w=1.0} R^{e}(w.\Phi)\|^2     = \sum_{e\in \mathcal{E}_{tr}} \pi^{e}  \Big(\mathbb{E}^{e}[\frac{\partial \ell(w.\Phi(X^{e}, Y^{e})}{\partial w}|_{w=1.0}\frac{\partial \ell(w.\Phi(\tilde{X}^{e}, \tilde{Y}^{e})}{\partial w}|_{w=1.0}]\Big)\\ 
% &  = \sum_{e\in \mathcal{E}_{tr}} \pi^{e} \Big(\mathbb{E}[\frac{\partial \ell(w.\Phi(X^{e}, Y^{e})}{\partial w}]\Big)^2 \\
% & = \sum_{e\in \mathcal{E}_{tr}} \pi^{e} \mathbb{E}_{X^{e},Y^{e}, \tilde{X}^{e}, \tilde{Y}^{e}}\Big[\frac{\partial \ell(w.\Phi(X^{e}), Y^{e})}{\partial w} \frac{\partial \ell(w.\Phi(\tilde{X}^{e}), \tilde{Y}^{e})}{\partial w}\Big]  \\ 
% & = \mathbb{E}_{E, X^{E},Y^{E}, \tilde{X}^{E}, \tilde{Y}^{E}}\Big[\frac{\partial \ell(w.\Phi(X^{e}), Y^{e})}{\partial w} \frac{\partial \ell(w.\Phi(\tilde{X}^{e}), \tilde{Y}^{e})}{\partial w}\Big] 
% \end{split}
% \end{equation}

\subsubsection{$\epsilon$-representative training set for $R$ and $R^{'}$}

We use the definition of $\epsilon$-representative sample from \cite{shalev2014understanding} and state it appropriately for both $R$ and $R^{'}$.
\begin{definition}
A training set $S$ is called $\epsilon$-representative (w.r.t. domain $\mathcal{Z}$, hypothesis $\mathcal{H}$, loss $\ell$ and distribution $\mathcal{D}$) if 
\begin{equation}
    \forall h \in \mathcal{H}, |\hat{R}(h)-R(h)|\leq \epsilon
\end{equation}
where $R(h) = \mathbb{E}_{\mathcal{D}}[\ell(h(X), Y]$ and $(X,Y) \sim \mathcal{D}$. 
\end{definition}

Following the above definition, we apply it to the set of points $\tilde{D}$ defined above in \eqref{estimate_rp_eqn6}.  $\tilde{D}$ is called $\epsilon$-representative w.r.t. domain $\mathcal{X}$, hypothesis $\mathcal{H}_{\Phi}$, loss $\ell^{'}$ (\eqref{estimate_rp_eqn3}) and distribution $\tilde{\mathbb{P}}$ (\eqref{estimate_rp_eqn7}) if 
\begin{equation}
    \forall \Phi \in \mathcal{H}_{\Phi}, |\hat{R}^{'}(\Phi)-R^{'}(\Phi)|\leq \epsilon
\end{equation}
where $R^{'}(\Phi) = \tilde{\mathbb{E}}\bigg[ \ell^{'}\Big(h, \big((X^{e}, Y^{e}), (\tilde{X}^{e}, \tilde{Y}^{e})\big)\Big)\bigg]$ (from \eqref{estimate_rp_eqn4}) and  $(e, (X^{e},Y^{e}), (\tilde{X}^{e}, \tilde{Y}^{e})) \sim \tilde{\mathbb{P}}$. 

% \begin{definition}
% A training set with tuple of data points $\tilde{S}$ is called $\epsilon$-representative (w.r.t. domain $\mathcal{X}$, hypothesis $\mathcal{H}$, loss $\frac{\partial \ell}{\partial w}$ and distribution $\mathcal{D}$) if 
% \begin{equation}
%     \forall h \in \mathcal{H}, |\hat{R}^{'}(h)-R^{'}(h)|\leq \epsilon
% \end{equation}
% where $R^{'}(h) = \mathbb{E}_{\mathcal{D}\times\mathcal{D}}[\frac{\partial \ell(w.h(X),Y)}{\partial w}|_{w=1.0}\frac{\partial \ell(w.h(\tilde{X}),\tilde{Y})}{\partial w}|_{w=1.0}]$
% \end{definition}

% Define $\tilde{D}$ to be a list of pairs of points in $D$; divide $D$ into two halves (assume even number of points in $D$) and pair the first point in first half to the first point in second half, and so on to construct $\tilde{D}$. 

Recall the definition of $\kappa$, $\kappa = \min_{\Phi \in \mathcal{H}_{\Phi}} |R^{'}(\Phi)-\epsilon|$. Next, we show that if $\tilde{D}$ is $\frac{\kappa}{2}$-representative w.r.t $\mathcal{X}$, $\mathcal{H}_{\Phi}$, loss $\ell^{'}$, distribution $\tilde{\mathbb{P}}$ then the set of invariant predictors in \eqref{eqn: EIRM_cons_swap} ($\hat{\mathcal{S}}^{\mathsf{IV}}(\epsilon)$ ) and the set of invariant predictors in \eqref{eqn: IRM_cons_swap} ($\mathcal{S}^{\mathsf{IV}}(\epsilon)$) are equal.

\begin{lemma} \label{lemma2: set_equality}
If $\kappa>0$ and $\tilde{D}$ is $\frac{\kappa}{2}$-representative w.r.t $\mathcal{X}$, $\mathcal{H}_{\Phi}$, loss $\ell^{'}$ and distribution $\tilde{\mathbb{P}}$ , then $\hat{\mathcal{S}}^{\mathsf{IV}}(\epsilon)  = \mathcal{S}^{\mathsf{IV}}(\epsilon)$. 
\end{lemma}

\begin{proof}
First we show $\mathcal{S}^{\mathsf{IV}}(\epsilon) \subseteq \hat{\mathcal{S}}^{\mathsf{IV}}(\epsilon)$. 
From the definition of $\kappa$, $\kappa = \min_{\Phi \in \mathcal{H}_{\Phi}} |R^{'}(\Phi)-\epsilon|$ it follows that $\forall \Phi \in \mathcal{H}_{\Phi} $
\begin{equation}
\begin{split}
    &|R^{'}(\Phi)-\epsilon| \geq \kappa  \implies R^{'}(\Phi) \geq \epsilon + \kappa \; \text{or} \;R^{'}(\Phi) \leq \epsilon - \kappa
    \label{lemma1_proof: eqn1}
\end{split}
\end{equation}

Consider any $\Phi$ in $\mathcal{S}^{\mathsf{IV}}(\epsilon)$. 
\begin{equation}
    R^{'}(\Phi) \leq \epsilon 
\end{equation}
Given the definition of $\kappa$ and \eqref{lemma1_proof: eqn1}, we obtain 
\begin{equation}
    R^{'}(\Phi) \leq \epsilon \implies R^{'}(\Phi) \leq \epsilon-\kappa
\end{equation} 
Therefore, $\mathcal{S}^{\mathsf{IV}}(\epsilon) \subseteq \mathcal{S}^{\mathsf{IV}}(\epsilon-\kappa)$. Also, it follows from the definition of the set $\mathcal{S}^{\mathsf{IV}}(\epsilon)$ that $\mathcal{S}^{\mathsf{IV}}(\epsilon-\kappa) \subseteq \mathcal{S}^{\mathsf{IV}}(\epsilon)$. Hence,

\begin{equation}
\mathcal{S}^{\mathsf{IV}}(\epsilon) =\mathcal{S}^{\mathsf{IV}}(\epsilon-\kappa)
\label{lemma1_proof: eqn4}
\end{equation}

Consider any $\Phi$ in $\mathcal{S}^{\mathsf{IV}}(\epsilon)$
\begin{equation}
\begin{split}
      & R^{'}(\Phi) \leq \epsilon-\kappa\; \text{(From \eqref{lemma1_proof: eqn4})}  \\ 
       &   R^{'}(\Phi) - \hat{R}^{'}(\Phi) + \hat{R}^{'}(\Phi) \leq \epsilon -\kappa \\ 
       &   \hat{R}^{'}(\Phi) \leq \epsilon -\kappa +  |R^{'}(\Phi) - \hat{R}^{'}(\Phi) |\\
\end{split} 
\label{lemma1_proof: eqn2}
\end{equation}

From the definition of $\frac{\kappa}{2}$-representativeness it follows that $|R^{'}(\Phi) - \hat{R}^{'}(\Phi) | \leq \frac{\kappa}{2}$ and substituting this in 
\eqref{lemma1_proof: eqn2} we get
\begin{equation}
    \begin{split}
            &      \hat{R}^{'}(\Phi) \leq \epsilon -\kappa/2 \implies \hat{R}^{'}(\Phi) \leq \epsilon \implies \mathcal{S}^{\mathsf{IV}}(\epsilon) \subseteq \hat{\mathcal{S}}^{\mathsf{IV}}(\epsilon)
    \end{split}
\end{equation}

% Therefore, $\mathcal{S}^{\mathsf{IV}}(\epsilon) \subseteq \hat{\mathcal{S}}^{\mathsf{IV}}(\epsilon)$.

Next we show $\hat{\mathcal{S}}^{\mathsf{IV}}(\epsilon) \subseteq \mathcal{S}^{\mathsf{IV}}(\epsilon)$.

Consider $\Phi \in \hat{\mathcal{S}}^{\mathsf{IV}}(\epsilon)$

\begin{equation}
\begin{split}
   &  \hat{R}^{'}(\Phi) \leq \epsilon \\ 
    & \hat{R}^{'}(\Phi) - R^{'}(\Phi) + R^{'}(\Phi^{'}) \leq \epsilon \\  
    & R^{'}(\Phi) \leq \epsilon + |\hat{R}^{'}(\Phi) - R^{'}(\Phi)| 
\end{split}
\label{lemma1_proof: eqn3}
\end{equation}

From the definition of $\frac{\kappa}{2}$-representativeness it follows that $|R^{'}(\Phi) - \hat{R}^{'}(\Phi) | \leq \frac{\kappa}{2}$ and substituting this in 
\eqref{lemma1_proof: eqn3} we get

\begin{equation}
    R^{'}(\Phi) \leq \epsilon + \frac{\kappa}{2}
\end{equation}

From \eqref{lemma1_proof: eqn1}, it follows that $R^{'}(\Phi) \leq \epsilon + \frac{\kappa}{2} \implies R^{'}(\Phi) \leq \epsilon $. Therefore, $\Phi \in \mathcal{S}^{\mathsf{IV}}(\epsilon)$. This proves the second part $\hat{\mathcal{S}}^{\mathsf{IV}}(\epsilon) \subseteq \mathcal{S}^{\mathsf{IV}}(\epsilon)$ and completes the proof.

\end{proof}

% Recall the definition of $\tilde{\epsilon} = \min\{\frac{\epsilon}{2}, \frac{\kappa}{2}\}$
\begin{lemma} \label{lemma3: rep_imp_eq}
If $\kappa>0$, $D$ is $\frac{\nu}{2}$-representative w.r.t $\mathcal{X}$, $\mathcal{H}_{\Phi}$, loss $\ell$ and distribution $\bar{\mathbb{P}}$ (joint distribution over $(e,X^{e},Y^{e})$ defined in Section \ref{secn: IRM}) and  and $\tilde{D}$ is $\frac{\kappa}{2}$-representative w.r.t $\mathcal{X}$, $\mathcal{H}_{\Phi}$, loss $\ell^{'}$ and distribution $\tilde{\mathbb{P}}$, then every solution $\hat{\Phi}$ to EIRM (\eqref{eqn: EIRM_cons_swap}) satisfies $\hat{\Phi}$ is in $\mathcal{S}^{\mathsf{IV}}(\epsilon)$ and $R(\Phi^{*}) \leq R(\hat{\Phi}) \leq R(\Phi^{*}) + \epsilon$, where $\Phi^{*}$ is the solution of IRM in \eqref{eqn: IRM_cons_swap}.
\end{lemma}

\begin{proof}
Given the condition in the above lemma, we are able to use the previous Lemma \ref{lemma2: set_equality} to deduce that $\hat{\mathcal{S}}^{\mathsf{IV}}(\epsilon)  = \mathcal{S}^{\mathsf{IV}}(\epsilon)$. This makes the set of predictors satisfying the constraints in EIRM \eqref{eqn: EIRM_cons_swap} and IRM \eqref{eqn: IRM_cons_swap} the same. 

$\Phi^{*}$ solves \eqref{eqn: IRM_cons_swap} and $\hat{\Phi}$ solves \eqref{eqn: EIRM_cons_swap}. From $\frac{\nu}{2}$-representativeness we know that $R(\hat{\Phi}) -\frac{\nu}{2}    \leq \hat{R}(\hat{\Phi})$. From the optimality of $\hat{\Phi}$ we know that  $\hat{R}(\hat{\Phi}) \leq \hat{R}(\Phi^{*})$ ($\Phi^{*} \in \hat{\mathcal{S}}^{\mathsf{IV}}(\epsilon)  = \mathcal{S}^{\mathsf{IV}}(\epsilon)$). Moreover, from $\frac{\nu}{2}$-representativeness we know that $\hat{R}(\Phi^{*}) \leq R(\Phi^{*}) + \frac{\nu}{2}$. We combine these conditions as follows.

\begin{equation}
R(\hat{\Phi}) -\frac{\nu}{2}    \leq \hat{R}(\hat{\Phi}) \leq \hat{R}(\Phi^{*}) \leq R(\Phi^{*}) + \frac{\nu}{2}
\label{eqn : iota-rep-opt-cond}
\end{equation}

% In the above the first and third inequality follows $\tilde{\epsilon}$-representativeness. The second inequality follows from optimality of $\hat{\Phi}$ and the fact that $\Phi^{*} \in \hat{\mathcal{S}}^{\mathsf{IV}}(\epsilon)  = \mathcal{S}^{\mathsf{IV}}(\epsilon)$. 
% Therefore,C
Comparing the first and third inequality in the above equations we get $R(\hat{\Phi}) \leq R(\Phi^{*}) + \nu$. From the optimality of $\Phi^{*}$ over the set $\mathcal{S}^{\mathsf{IV}}(\epsilon)$ and since $\hat{\Phi} \in \mathcal{S}^{\mathsf{IV}}(\epsilon)$ it follows that $R(\Phi^{*}) \leq R(\hat{\Phi})$. Hence, $R(\Phi^{*}) \leq R(\hat{\Phi}) \leq R(\Phi^{*}) + \nu$. This completes the proof.

\end{proof}
% \begin{proposition} \label{prop2: scomp_irm_append}
% For a given $\nu>0$, if  $\mathcal{H}_{\Phi}$ is a finite hypothesis class, Assumption \ref{assm3: bounded loss and gradient} holds, $\kappa >0$, and if the  number of samples $|D|$ is greater than  $\max\big\{\frac{16L'^{4}}{\kappa^2}, \frac{8L^2}{\nu^2}\big\}\log\big(\frac{4|\mathcal{H}_{\Phi}|}{\delta}\big)$, then with a probability at least $1-\delta$,  every solution $\hat{\Phi}$ of  EIRM (\eqref{eqn: EIRM_cons_swap})  is in $\mathcal{S}^{\mathsf{IV}}(\epsilon)$ and $R(\Phi^{*}) \leq R(\hat{\Phi}) \leq R(\Phi^{*}) + \nu$, where $\Phi^{*}$ is a solution to IRM (\eqref{eqn: IRM_cons_swap}).
% \end{proposition}
Next, we prove Proposition \ref{prop2: scomp_irm} from the main body of the manuscript. 

\begin{proposition} \label{prop2: scomp_irm_append}
For every $\nu>0,\epsilon>0$ and $\delta\in (0,1)$, if  $\mathcal{H}_{\Phi}$ is a finite hypothesis class, Assumption \ref{assm3: bounded loss and gradient} holds, $\kappa >0$, and if the  number of samples $|D|$ is greater than  $\max\big\{\frac{16L'^{4}}{\kappa^2}, \frac{8L^2}{\nu^2}\big\}\log\big(\frac{4|\mathcal{H}_{\Phi}|}{\delta}\big)$, then with a probability at least $1-\delta$,  every solution $\hat{\Phi}$ of  EIRM (\eqref{eqn: EIRM_cons_swap})  is a $\nu$ approximation of IRM, i.e.  $\hat{\Phi}\in \mathcal{S}^{\mathsf{IV}}(\epsilon)$, $R(\Phi^{*}) \leq R(\hat{\Phi}) \leq R(\Phi^{*}) + \nu$, where $\Phi^{*}$ is a solution to IRM (\eqref{eqn: IRM_cons_swap}).
\end{proposition}
% \begin{proposition}
% For a given $\nu>0$, if if $\mathcal{H}_{\Phi}$ is a finite hypothesis class, $\kappa >0$, Assumption \ref{assm3: bounded loss and gradient} holds, and if the total number of samples $|D|$ is greater than  $\max\{\frac{4L'^{4}}{\frac{\kappa^2}{4}}, \frac{2L^2}{\frac{\nu^2}{4}}\}\log(\frac{4|\mathcal{H}_{\Phi}|}{\delta})$, then with probability $1-\delta$  the output of EIRM \eqref{eqn: EIRM_cons_swap}, $\hat{\Phi}$ is in $\mathcal{S}^{\mathsf{IV}}(\epsilon)$ and $R(\Phi^{*}) \leq R(\hat{\Phi}) \leq R(\Phi^{*}) + \nu$, where $\Phi^{*}$ is the solution of IRM in \eqref{eqn: IRM_cons_swap}.
% \end{proposition}

\begin{proof}

From Lemma \ref{lemma3: rep_imp_eq}, we know that if $D$ is $\frac{\nu}{2}$-representative w.r.t  $\mathcal{X}$, $\mathcal{H}_{\Phi}$, loss $\ell$ and distribution $\mathbb{P}$ and if $\tilde{D}$ is $\frac{\kappa}{2}$-representative $\mathcal{X}$, $\mathcal{H}_{\Phi}$, loss $\ell^{'}$ and distribution $\tilde{\mathbb{P}}$, then the the claim in the above Proposition is true. 

% then it follows that $D$ is  $\frac{\nu}{2}$-representative w.r.t $\ell$ and if $\tilde{D}$ is $\frac{\kappa}{2}$-representative $\ell^{'}$

Define an event $A$: $D$ is $\frac{\nu}{2}$-representative w.r.t $\mathcal{X}$, $\mathcal{H}_{\Phi}$, loss $\ell$ and distribution $\bar{\mathbb{P}}$. 

Define an event $B$: $\tilde{D}$ is $\frac{\kappa}{2}$-representative $\mathcal{X}$, $\mathcal{H}_{\Phi}$, loss $\ell^{'}$ and distribution $\tilde{\mathbb{P}}$.  

Define success as $A\cap B$. 
Next, we show that if $|D|$ is greater than  $\max\big\{\frac{16L'^{4}}{\kappa^2}, \frac{8L^2}{\nu^2}\big\}\log\big(\frac{4|\mathcal{H}_{\Phi}|}{\delta}\big)$, then $\mathbb{P}(A\cap B)$ occurs with a probability at least $1-\delta$. 
$P(A\cap B) = 1- P(A^{c} \cup B^{c}) \geq 1- P(A^{c}) - P(B^{c})$. If we  bound $P(A^{c}) \leq \frac{\delta}{2}$ and $P(B^{c}) \leq \frac{\delta}{2}$, then we know the probability of success is at least $1-\delta$.

We write $$\mathbb{P}(A) = \mathbb{P}\Big(\Big\{D: \forall h \in \mathcal{H}_{\Phi}, |\hat{R}(h) - R(h)| \leq \frac{\nu}{2} \Big\}\Big) = 1-\mathbb{P}\Big(\Big\{D: \exists h \in \mathcal{H}_{\Phi}, |\hat{R}(h) - R(h)| > \frac{\nu}{2} \Big\}\Big)  $$

\begin{equation}
    \begin{split}
        & \mathbb{P}\Big(\Big\{D: \exists h \in \mathcal{H}_{ \Phi}, |\hat{R}(h) - R(h)| >\frac{\nu}{2} \Big\}\Big)   = \mathbb{P}\Big(\bigcup_{h \in \mathcal{H}_{ \Phi}}\Big\{D: |\hat{R}(h) - R(h)| >\frac{\nu}{2}\Big\}\Big) \\
        & \leq \sum_{h\in \mathcal{H}_{\Phi}}\mathbb{P}\Big(\Big\{D: |\hat{R}(h) - R(h)| >\frac{\nu}{2}\Big\} \Big)
    \end{split}
    \label{prop2_proof:eqn1}
\end{equation}
% $\mathbb{P}(\{D: \exists h \in \mathcal{H}_{ \Phi}, |\hat{R}(h) - R(h)| >\frac{\nu}{2} \})  = \mathbb{P}(\cup_{h \in \mathcal{H}_{ \Phi}}) \leq \sum_{h\in \mathcal{H}_{\Phi}}\mathbb{P}(|\hat{R}(h) - R(h)| > \frac{\nu}{2} )$

The loss function is bounded $|\ell(\Phi(\cdot), \cdot) | \leq L$. From Hoeffding's inequality in Lemma \ref{lemma1: hoeffding} it follows that  
\begin{equation}\mathbb{P}\Big(\Big\{D: |\hat{R}(h) - R(h)| >\frac{\nu}{2}\Big\} \Big) \leq 2 \exp\Big(-\frac{|D|\nu^2}{8L^2}\Big)\label{prop2_proof:eqn2}
\end{equation}
Using this expression \eqref{prop2_proof:eqn2} in \eqref{prop2_proof:eqn1}, we get 
\begin{equation}
2|\mathcal{H}_{ \Phi}| \exp\Big(-\frac{|D|\nu^2}{8L^2}\Big) \leq \frac{\delta}{2} \implies|D| \geq  \frac{8L^2}{\nu^2}\log\Big(\frac{4|\mathcal{H}_{ \Phi}|}{\delta}\Big)
\label{prop2_proof:eqn6}
\end{equation}

$$\mathbb{P}(B) = \mathbb{P}\Big(\Big\{\tilde{D}: \forall h \in \mathcal{H}_{\Phi}, |\hat{R}^{'}(h) - R^{'}(h)| \leq \frac{\kappa}{2} \Big\}\Big) = 1-\mathbb{P}\Big(\Big\{\tilde{D}: \exists h \in \mathcal{H}_{\Phi}, |\hat{R}^{'}(h) - R^{'}(h)| > \frac{\kappa}{2} \Big\}\Big)  $$

\begin{equation}
    \begin{split}
        & \mathbb{P}\Big(\Big\{\tilde{D}: \exists h \in \mathcal{H}_{ \Phi}, |\hat{R}^{'}(h) - R^{'}(h)| >\frac{\kappa}{2} \Big\}\Big)   = \mathbb{P}\Big(\bigcup_{h \in \mathcal{H}_{ \Phi}}\Big\{\tilde{D}: |\hat{R}^{'}(h) - R^{'}(h)| >\frac{\kappa}{2} \Big\}\Big) \\
        & \leq \sum_{h\in \mathcal{H}_{\Phi}}\mathbb{P}\Big(\Big\{\tilde{D}: |\hat{R}^{'}(h) - R^{'}(h)| >\frac{\kappa}{2} \Big\} \Big)
    \end{split}
    \label{prop2_proof:eqn3}
\end{equation}

% $\mathbb{P}(\{D: \exists h \in \mathcal{H}_{ \Phi}, |\hat{R}^{'}(h) - R^{'}(h)| > \frac{\kappa}{2}  \})  = \mathbb{P}(\cup_{h \in \mathcal{H}_{ \Phi}}) \leq \sum_{h\in \mathcal{H}_{\Phi}}\mathbb{P}(|\hat{R}^{'}(h) - R^{'}(h)| > \frac{\kappa}{2}  )$

The  gradient of loss function is bounded $|\frac{\partial \ell(h(\cdot), \cdot)}{\partial w}|_{w=1.0} | \leq L'$. From the definition of $\ell^{'}(h(\cdot), \cdot) $ in \eqref{estimate_rp_eqn4}, we can infer that $|\ell^{'}(h(\cdot), \cdot)| \leq L'^2$. Recall that $R^{'}(h) = \tilde{\mathbb{E}}\Big[\ell^{'}\Big(h, \big((X^e,Y^e), (\tilde{X}^e, \tilde{Y}^e)\big)\Big)\Big]$

From Hoeffding's inequality in Lemma \ref{lemma1: hoeffding} it follows that  
\begin{equation}
\mathbb{P}\Big(\Big\{\tilde{D}: |\hat{R}^{'}(h) - R^{'}(h)| >\frac{\kappa}{2} \Big\}  \Big) \leq 2 \exp\Big(-\frac{|\tilde{D}|\kappa^2 }{8L'^4}\Big) = 2 \exp\Big(-\frac{|D|\kappa^2 }{16L'^4}\Big) 
\label{prop2_proof:eqn4}
\end{equation}

Using the above \eqref{prop2_proof:eqn4} in \eqref{prop2_proof:eqn3} we get 
\begin{equation}
2 |\mathcal{H}_{\Phi}|\exp\Big(-\frac{|D|\kappa^2 }{16L'^4}\Big) \leq \frac{\delta}{2} \implies |D| \geq  \frac{16L'^4}{\kappa^2}\log\Big(\frac{4|\mathcal{H}_{\Phi}|}{\delta}\Big)
\label{prop2_proof:eqn5}
\end{equation}

Combining the two conditions in \eqref{prop2_proof:eqn6} and \eqref{prop2_proof:eqn6} we get that if
$$|D| \geq  \max\Big\{\frac{16L'^{4}}{\kappa^2}, \frac{8L^2}{\nu^2}\Big\}\log\Big(\frac{4|\mathcal{H}_{\Phi}|}{\delta}\Big)$$

then with probability at least $1-\delta$ event $A\cap B$ occurs.

\end{proof}

\subsubsection{Property of least squares optimal solutions}
We first remind ourselves of a simple property of least squares minimization. 
Consider the least squares minimization setting, where $R(h) = \mathbb{E}[(Y-h(X))^2] $. 
\begin{equation}
\begin{split}
  &  \mathbb{E}[(Y-h(X))^2] = \mathbb{E}\Big[\big(Y-E\big[Y\big|X\big] + \mathbb{E}\big[Y\big|X\big]-h(X)\big)^2\Big] \\ & =  \mathbb{E}_{X}\Big[\mathbb{E}\big[\big(Y-E\big[Y\big|X\big]\big)^2|X\big]\Big] + \mathbb{E}_{X}\Big[\big(\mathbb{E}\big[Y\big|X\big]-h(X)\big)^2\Big]  \\ 
  & = \mathbb{E}_{X}\Big[\mathsf{Var}\big[Y\big|X\big]\Big]  + \mathbb{E}_{X}\Big[\big(\mathbb{E}\big[Y\big|X\big]-h(X)\big)^2\Big]  
  \end{split}
\end{equation}
In the above simplification, we use the law of total expectation.
Both the terms in the above equations are always greater than or equal to zero. The first term does not depend on $h$, which implies the minimization can focus on second term only.
\begin{equation}
\begin{split}
 \min_{h}  R(h) = \mathbb{E}_{X}\Big[\mathsf{Var}\big[Y\big|X\big]\Big]  + \min_{h}\mathbb{E}_{X}\Big[\big(\mathbb{E}\big[Y\big|X\big]-h(X)\big)^2\Big] 
  \end{split}
  \label{eqn:ls_opt}
\end{equation}

% \begin{equation}
% h^{*}\arg\min_{h}\mathbb{E}_{X}[(\mathbb{E}[Y|X]-h(X))^2] = 0 \implies\; \text{if}\; d\mathbb{P}_X>0 ,\; h(X) = E[Y|X] 
%   \label{eqn:ls_opt1}
% \end{equation}

Assume that $\mathbb{P}$ has full support over $\mathcal{X}$. Define $\forall x \in \mathcal{X}, h^{*}(x) = \mathbb{E}[Y|X=x]$. 
 $\forall h, R(h)\geq \mathbb{E}_{X}\Big[\mathsf{Var}\big[Y\big|X\big]\Big]$. Since $R(h^{*})=\mathbb{E}_{X}\Big[\mathsf{Var}\big[Y\big|X\big]\Big]$. Therefore, 

\begin{equation}
    h^{*} \in \arg\min_{h}\mathbb{E}_{X}\Big[\big(\mathbb{E}\big[Y\big|X\big]-h(X)\big)^2\Big] 
     \label{eqn: ls_opt1}
\end{equation}
 Moreover, we conclude that $h^{*}$ is the unique minimizer. Observe that  $\mathbb{E}_{X}\Big[\big(\mathbb{E}\big[Y\big|X\big]-h(X)\big)^2\Big] $ is zero for $h=h^{*}$. From Theorem 1.6.6 in \cite{ash2000probability}, it follows that any other minimizer is same as $h^{*}$ except over a set of measure zero.

Recall the definition of $m(x)=\mathbb{E}[Y^{e}|X^{e}=x]$
\begin{lemma} \label{lemma4: unique_ls}
Let $\ell$ be the square loss. If Assumption \ref{assm4: invariant_cexp} holds and $m\in \mathcal{H}_{\Phi}$, then $m$ uniquely solves expected risk minimization $m \in \arg\min_{\Phi \in \mathcal{H}_{\Phi}} R(\Phi)$ and also uniquely solves IRM (\eqref{eqn: IRM_cons_swap}).
\end{lemma}

\begin{proof}
$R(\Phi) = \sum_{e\in \mathcal{E}_{tr}} \pi^{e}R^{e}(\Phi)$.
From Assumption \ref{assm4: invariant_cexp} and the observation in \eqref{eqn: ls_opt1}, it follows that the unique optimal solution to expected risk minimization for each $R^{e}$ is $m$. Therefore, $m$ also minimizes the weighted combination $R$. 
% From Assumption \ref{assm5: invariant_realizable}, it follows that since $m(\cdot) \in \mathcal{H}_{\Phi}$

To show the latter part of the Lemma, if we can show that  $m \in \mathcal{S}^{\mathsf{IV}}(\epsilon)$, then the rest of the proof follows from the previous part as we already showed $m$ is a minimizer among all the functions in $\mathcal{H}_{\Phi}$ and $\mathcal{S}^{\mathsf{IV}}(\epsilon) \subseteq \mathcal{H}_{\Phi}$.

Suppose $m \not\in \mathcal{S}^{\mathsf{IV}}(\epsilon)$. This implies there exists at least one environment for which $\|\nabla_{w|w=1.0} R^{e}(w\cdot m)\|^2 > 0 \implies \nabla_{w|w=1.0} R^{e}(w\cdot m) \not=0 $. As a result $\exists$ $w$ in the neighborhood of $w=1.0$ where $R^{e}(w\cdot m) < R^{e}(m)$ (if such a point does not exist and all the points in the neigbohood of $w=1.0$ are greater than or equal to $R^{e}(m)$ that would make $\nabla_{w|w=1.0} R^{e}(w\cdot m) =0 $, which would be a contradiction). Therefore, $R^{e}(w\cdot m) < R^{e}(m)$. However, this is a contradiction as we know that $m$ is the unique optimizer for each environment. Hence, $m\not\in \mathcal{S}^{\mathsf{IV}}(\epsilon)$ cannot be true and thus $m \in \mathcal{S}^{\mathsf{IV}}(\epsilon)$. This completes the proof.

\end{proof}

Next, we prove Proposition \ref{prop4: emr_irm_real} from the main body of the manuscript. 

\begin{proposition} \label{prop4: emr_irm_real_append}Let $\ell$ be the square loss.  For every $\nu>0, \epsilon>0$ and $\delta\in (0,1)$, if $\mathcal{H}_{\Phi}$ is a finite  hypothesis class, $m\in \mathcal{H}_{\Phi}$, Assumptions \ref{assm3: bounded loss and gradient}, \ref{assm4: invariant_cexp} hold, and

$\bullet$  if the  number of samples $|D|$ is greater than
$\max\big\{\frac{8L^2}{\nu^2} \log(\frac{4|\mathcal{H}_{\Phi}|}{\delta}),\frac{16L'^{4}}{\epsilon^2}\log(\frac{2}{\delta}) \big\}$, then  with a probability at least $1-\delta$, every solution $\hat{\Phi}$ to EIRM (\eqref{eqn: EIRM_cons_swap}) satisfies  $R(m) \leq R(\hat{\Phi}) \leq R(m) + \nu$. If also $\nu<\tilde{\kappa}$, then $\Phi^{\dagger} = m$. 

% If  Assumptions \ref{assm3: bounded loss and gradient}, \ref{assm4: invariant_cexp} and \ref{assm5: invariant_realizable} hold, $\mathcal{H}_{\Phi}$ be a finite hypothesis classes and 

$\bullet$ if the number of samples $|D|$ is greater than
 $\frac{8L^2}{\nu^2}\log(\frac{2|\mathcal{H}_{ \Phi}|}{\delta})$, then with a probability at least $1-\delta$, every solution $\Phi^{\dagger}$ to ERM satisfies $R(m) \leq R(\Phi^{\dagger}) \leq R(m) + \nu$.  If also $\nu<\tilde{\kappa}$, then $\Phi^{\dagger} = m$. 
\end{proposition}

\begin{proof}
We first cover the second part of the Proposition. From Proposition \ref{prop3: scomp_erm}, we know that the output of ERM will satisfy $R(\Phi^{+}) \leq R(\Phi^{\dagger}) \leq R(\Phi^{+}) + \nu$. In this case from Lemma \ref{lemma4: unique_ls}, it follows that $\Phi^{+} = m$. From the definition of $\tilde{\kappa}$ and the fact that $\nu <\tilde{\kappa}$ implies that $\Phi^{\dagger} = m$. We now move to the first part of the Proposition.

For EIRM we will derive a tighter bound on sample complexity than the one in Proposition \ref{prop2: scomp_irm} since we can now use the Assumption \ref{assm4: invariant_cexp}. Observe that $\forall e \in \mathcal{E}_{tr}, \; \nabla_{w| w=1.0}R^{e}(w\cdot m) = 0$ (see the proof of Lemma \ref{lemma4: unique_ls}). Therefore, $R^{'}(m) =0$.

Define an event $A$: $D$ is $\frac{\nu}{2}$-representative w.r.t $\mathcal{X}$, $\mathcal{H}_{\Phi}$, loss $\ell$ and distribution $\bar{\mathbb{P}}$. 

Define an event $B$: $\tilde{D}$ is such that $|\hat{R}^{'}(m) - R^{'}(m)| \leq \frac{\epsilon}{2}$.   Since $R^{'}(m)=0$,  
$|\hat{R}^{'}(m) - R^{'}(m)| \leq \frac{\epsilon}{2} \implies |\hat{R}^{'}(m) | \leq \frac{\epsilon}{2} \implies \hat{R}^{'}(m)  \leq \frac{\epsilon}{2} $.
 Therefore, $m \in \hat{\mathcal{S}}^{\mathsf{IV}}(\epsilon)$.
 
If $A\cap B$ occurs, then $R(m) \leq R(\hat{\Phi}) \leq R(m) + \nu$; we justify claim next.
 Suppose $\hat{\Phi}$ solves \eqref{eqn: EIRM_cons_swap}. If event $A$ occurs, then from $\frac{\nu}{2}$-representative condition  we know that $R(\hat{\Phi}) -\frac{\nu}{2}    \leq \hat{R}(\hat{\Phi})$. From optimality of $\hat{\Phi}$ it follows that  $\hat{R}(\hat{\Phi}) \leq \hat{R}(m)$ (event $B$ $\implies m \in \hat{\mathcal{S}}^{\mathsf{IV}}(\epsilon)$). Moreover, from $\frac{\nu}{2}$-representative property, we  conclude that $\hat{R}(m) \leq R(m) + \frac{\nu}{2}$. We combine these conditions as follows. 

\begin{equation}
R(\hat{\Phi}) -\frac{\nu}{2}    \leq \hat{R}(\hat{\Phi}) \leq \hat{R}(m) \leq R(m) + \frac{\nu}{2}
\end{equation}
 From the above we have $R(m) \leq R(\hat{\Phi}) \leq R(m) + \nu$. Recall the definition of $\tilde{\kappa}$ and since $\nu<\tilde{\kappa}$ $\implies$ $\hat{\Phi} =m$.
 
Next, we bound the probability of success.
$P(A\cap B) = 1- P(A^{c} \cup B^{c}) \geq 1- P(A^{c}) - P(B^{c})$. If we can bound $P(A^{c}) \leq \frac{\delta}{2}$ and $P(B^{c}) \leq \frac{\delta}{2}$, then we know the probability of success is at least $1-\delta$.

We write $$\mathbb{P}(A) = \mathbb{P}\Big(\Big\{D: \forall h \in \mathcal{H}_{\Phi}, |\hat{R}(h) - R(h)| \leq \frac{\nu}{2} \Big\}\Big) = 1-\mathbb{P}\Big(\Big\{D: \exists h \in \mathcal{H}_{\Phi}, |\hat{R}(h) - R(h)| > \frac{\nu}{2} \Big\}\Big)  $$

% \begin{equation}
%     \begin{split}
%         & \mathbb{P}\Big(\Big\{D: \exists h \in \mathcal{H}_{ \Phi}, |\hat{R}(h) - R(h)| >\frac{\nu}{2} \Big\}\Big)   = \mathbb{P}(\cup_{h \in \mathcal{H}_{ \Phi}}\{|\hat{R}(h) - R(h)| >\frac{\nu}{2}\}) \\
%         & \leq \sum_{h\in \mathcal{H}_{\Phi}}\mathbb{P}(|\hat{R}(h) - R(h)| > \frac{\nu}{2} )
%     \end{split}
% \end{equation}

% % $\mathbb{P}(\{D: \exists h \in \mathcal{H}_{ \Phi}, |\hat{R}(h) - R(h)| >\frac{\nu}{2} \})  = \mathbb{P}(\cup_{h \in \mathcal{H}_{ \Phi}}) \leq \sum_{h\in \mathcal{H}_{\Phi}}\mathbb{P}(|\hat{R}(h) - R(h)| > \frac{\nu}{2} )$

% The loss function is bounded $|\ell(\Phi(\cdot), \cdot) | \leq L$. From Hoeffding's inequality in Lemma \ref{lemma1: hoeffding} it follows that  $$\mathbb{P}(|\hat{R}(h) - R(h)| > \frac{\nu}{2} ) \leq 2 \exp(-\frac{\frac{|D|\nu^2}{4}}{2L^2})$$
% Using this expression in the union bound above we get the condition that
% $$2|\mathcal{H}_{ \Phi}| \exp(-\frac{\frac{|D|\nu^2}{4}}{2L^2}) \leq \frac{\delta}{2} $$

% The loss function is bounded $|\ell(\Phi(\cdot), \cdot) | \leq L$. From Hoeffding's inequality in Lemma \ref{lemma1: hoeffding} it follows that  $$\mathbb{P}(|\hat{R}(h) - R(h)| > \frac{\nu}{2} ) \leq 2 \exp(-\frac{\frac{|D|\nu^2}{4}}{2L^2})$$
% Using this expression in the union bound above we get the condition that
% $$2|\mathcal{H}_{ \Phi}| \exp(-\frac{\frac{|D|\nu^2}{4}}{2L^2}) \leq \frac{\delta}{2} $$
From \eqref{prop2_proof:eqn6} if the condition \begin{equation}|D| \geq 
\frac{8L^2}{\nu^2}\log\Big(\frac{4|\mathcal{H}_{ \Phi}|}{\delta}\Big)
\label{prop4_proof: eqn1}
\end{equation}is true, then 
is true, then event $A^{c}$ occurs with probability at most $\frac{\delta}{2}$.

We write $\mathbb{P}(B) = \mathbb{P}\Big(\Big\{\tilde{D}:, |\hat{R}^{'}(m) - R^{'}(m)| \leq \frac{\epsilon}{2} \Big\}\Big) = 1-\mathbb{P}\Big(\Big\{\tilde{D}: |\hat{R}^{'}(m) - R^{'}(m)| > \frac{\epsilon}{2} \Big\}\Big)  $. 
The  gradient of loss function is bounded $|\frac{\partial \ell(\Phi(\cdot), \cdot)}{\partial w}|_{w=1.0} | \leq L^{'}$. From Hoeffding's inequality in Lemma \ref{lemma1: hoeffding} it follows that  

\begin{equation}\mathbb{P}\Big(\Big\{\tilde{D}: |\hat{R}^{'}(h) - R^{'}(h)| > \frac{\epsilon}{2}\Big\}  ) \leq 2 \exp(-\frac{|D|\epsilon^2}{16L'^4}\Big)
\label{prop4_proof: eqn3}
\end{equation}
We bound the above  \eqref{prop4_proof: eqn3} by $\frac{\delta}{2}$ to get

\begin{equation}
2 \exp(-\frac{|D|\epsilon^2}{16L'^4})\leq \frac{\delta}{2}  \implies |D| \geq  \frac{16L'^4}{\epsilon^2}\log(\frac{4}{\delta}) 
\label{prop4_proof: eqn2}
\end{equation}

Combining the two conditions \eqref{prop4_proof: eqn1} and \eqref{prop4_proof: eqn2}, 
$$|D| \geq  \max \Big\{ \frac{8L^2}{\nu^2}\log(\frac{4|\mathcal{H}_{\Phi}|}{\delta}), \frac{16L'^{4}}{\epsilon^2}\log(\frac{4}{\delta}) \Big\}$$
This ensures $P(A\cap B) \geq 1-\delta$.  This completes the proof. 

\end{proof}

Before stating the proof of Proposition \ref{prop5: irm_cfd_chd}, we will prove an intermediate proposition. 

For clarity, we will restate the result (Theorem 9 from \cite{arjovsky2019invariant}) next. 

\begin{assumption} \textbf{Linear general position.}
\label{assm7: linear_gen_posn}
 A set of training environments $\mathcal{E}_{tr}$ is said to lie in a linear general position of degree $r$ for some $r\in \mathbb{N}$ if $|\mathcal{E}_{tr}| > n-r + n/r$ and for all non-zero $x\in \mathbb{R}^{n}$ 
\begin{equation}
    \mathsf{dim}\Big(\mathsf{span}\Big\{\Sigma^e x-c^e\Big\}_{e\in \mathcal{E}_{tr}}\Big) > n-r
\end{equation}
where $\mathsf{span}$ is the linear span, $\mathsf{dim}$ is the dimension, and recall $n$ is dimension of $X^{e}$. This assumption checks for diversity in the environments and holds almost everywhere \cite{arjovsky2019invariant}. 
\end{assumption}
\begin{proposition} \label{thm10_arjovsky} (Theorem 9 \cite{arjovsky2019invariant})
If Assumptions \ref{assm6: linear_model} and \ref{assm7: linear_gen_posn} (with $r=1$) hold and let  $\Phi \in \mathbb{R}^{ n \times 1}$ ($\Phi\not=0$), then \begin{equation}
  \Phi^{\mathsf{T}}  \mathbb{E}^{e}[X^{e}X^{e,\mathsf{T}}]\Phi = \Phi^{\mathsf{T}}\mathbb{E}^{e}[X^{e}Y^{e}]
  \label{thm1: eqn1}
\end{equation}
holds for all $e\in \mathcal{E}_{tr}$ iff $\Phi=\tilde{S}^{\mathsf{T}}\gamma$.  
\end{proposition}

% We can also derive that $\tilde{S}^{\mathsf{T}}\gamma$ solves the OOD problem provided the data follows the above generative model. We denote $\mathbb{E}[X^{e}X^{e, \mathsf{T}}]$ as $\Sigma_e$

% \begin{assumption}
% \begin{itemize}
% \item $\forall e \in \mathcal{E}_{tr}, \pi^{e} \geq \frac{\pi^{\mathsf{min}}}{|\mathcal{E}_{tr}|}$
% % \item $\exists \; p_{\mathsf{min}}>0$ such that 
% % $\|\tilde{S}^{\mathsf{T}}\gamma\|^2 \geq \min_{e \in \mathcal{E}_{tr}}\{\frac{2 p_{\mathsf{min}}}{\lambda_{\mathsf{min}}(\Sigma_{e})}\}$ 
% \item $\exists \; \omega>0$ such that 
% $\|\tilde{S}^{\mathsf{T}}\gamma\|^2 \geq 2\omega$, $\forall \Phi \in \mathcal{H}_{\Phi}, \|\Phi\|^2 \geq \omega$, and $\tilde{S}^{\mathsf{T}}\gamma \in \mathcal{H}_{\Phi}$ 
% \item  $\lambda_{\mathsf{min}}  = \min_{e\in \mathcal{E}_{tr}}\lambda_{\mathsf{min}}(\Sigma_{e})>0$
% \end{itemize}
% \end{assumption}

% Observe that $\frac{\frac{2}{\omega\lambda_{\mathsf{min}} }\sqrt{\frac{\epsilon |\mathcal{E}_{tr}|}{\pi^{\mathsf{min}} }}}{1-\frac{2}{\omega\lambda_{\mathsf{min}}}\sqrt{\frac{\epsilon |\mathcal{E}_{tr}|}{\pi^{\mathsf{min}}}}}$ decreases as $\epsilon$ goes to zero; select $\epsilon_0$ to be the point where $\frac{\frac{2}{\omega\lambda_{\mathsf{min}} }\sqrt{\frac{\epsilon |\mathcal{E}_{tr}|}{\pi^{\mathsf{min}} }}}{1-\frac{2}{\omega\lambda_{\mathsf{min}}}\sqrt{\frac{\epsilon |\mathcal{E}_{tr}|}{\pi^{\mathsf{min}}}}}$ is $1-\frac{1}{\sqrt{2}}$, 
 Next, we propose an $\epsilon$-approximation of Proposition \ref{thm10_arjovsky}.
 
% Define $\epsilon_0 = \frac{\pi^{\mathsf{min}}}{|\mathcal{E}_{tr}|} (\omega \lambda_{\mathsf{min}})^2 (12-8\sqrt{2})$.
% . 
\begin{proposition} \label{prop_int: inf_sam_irm_approx}Let $\ell$ be the square loss.
If Assumptions \ref{assm6: linear_model}, \ref{assm8:regularity conditions on envmts and hphi} and \ref{assumption_singular_value} hold , a $\Phi$ that satisfies the constraint in \eqref{eqn: IRM_cons_swap} \begin{equation}
    R^{'}(\Phi) \leq \epsilon
\end{equation}
satisfies $\|\Phi- \tilde{S}^{\mathsf{T}}\gamma\|^2 \leq \epsilon$, where $\epsilon\in (0,1)$.
% where $\alpha \in [\frac{1}{1+\frac{1}{2\omega\lambda_{\mathsf{min}}}\sqrt{\frac{\epsilon |\mathcal{E}_{tr}|}{\pi^{\mathsf{min}}}}},  \frac{1}{1-\frac{1}{2\omega\lambda_{\mathsf{min}}}\sqrt{\frac{\epsilon |\mathcal{E}_{tr}|}{\pi^{\mathsf{min}}}}}]$
\end{proposition}
\begin{proof}
% From Assumption 7, we can write that \eqref{thm1: eqn1} implies
 Let us start by simplifying $\nabla_{w|w=1.0} R^{e}(w\cdot\Phi)$, using square loss for $\ell$ and linear representation $\Phi\in \mathbb{R}^{n\times 1}$.
\begin{equation} 
\begin{split}
    \nabla_{w|w=1.0} R^{e}(w\cdot\Phi) & = \frac{\partial \mathbb{E}^{e}\big[\big(Y^{e}-w\cdot\Phi^{\mathsf{T}}X^{e}\big)^2\big]}{\partial w}\Big|_{w=1.0} = 2\mathbb{E}^{e}\big[(\Phi^{\mathsf{T}}X^{e})^2\big] - 2\mathbb{E}^{e}\big[\Phi^{\mathsf{T}}X^{e}Y^{e}\big] \\ 
    & = 2\Phi^{\mathsf{T}}\mathbb{E}^{e}\big[X^{e}X^{e, \mathsf{T}}\big]\Phi - 2\Phi^{\mathsf{T}} \mathbb{E}^{e}\big[Y^{e}X^{e}\big]
    \end{split}
    \label{thm1_proof: eqn}
\end{equation}

Plug the above \eqref{thm1_proof: eqn} in the condition $R^{'}(\Phi)\leq \epsilon$ to get
\begin{equation}
     R^{'}(\Phi) = \sum_{e}\pi^{e}\|\nabla_{w| w=1.0} R^{e}(w\cdot\Phi)\|^2 =  4\sum_{e}\pi^{e}\Big( \Phi^{\mathsf{T}}\mathbb{E}^{e}\big[X^{e}X^{e,\mathsf{T}}\big]\Phi - \Phi^{\mathsf{T}}\mathbb{E}^{e}\big[X^{e}Y^{e}\big] \Big)^2 \leq \epsilon
      \label{thm1_proof: eqn0}
\end{equation}

From the bound on $\pi^{e}$ in Assumption \ref{assm6: linear_model} it follows that 
\begin{equation}
\begin{split}
  & \sum_{e}\pi^{e}\|\nabla_{w| w=1.0} R^{e}(w\cdot\Phi)\|^2  \geq \pi^{\mathsf{min}}  \sum_{e}\|\nabla_{w| w=1.0} R^{e}(w\cdot\Phi)\|^2  = \\ &\pi^{\mathsf{min}}  \sum_{e}\Big( \Phi^{\mathsf{T}}\mathbb{E}^{e}\big[X^{e}X^{e,\mathsf{T}}\big]\Phi - \Phi^{\mathsf{T}}\mathbb{E}^{e}\big[X^{e}Y^{e}\big] \Big)^2 =\pi^{\mathsf{min}} \sum_{e}\Big(\Phi^{\mathsf{T}}(\Sigma^e (\Phi-\tilde{S}^{\mathsf{T}}\gamma) - c^e)\Big)^2 \\ 
  & = \pi^{\mathsf{min}} \sum_{e}\Big((u + \tilde{S}^{\mathsf{t}}\gamma)^{\mathsf{T}}(\Sigma^e u  - c^e)\Big)^2 = \pi^{\mathsf{min}} \|(u + \tilde{S}^{\mathsf{t}}\gamma)^{\mathsf{T}}Q(u) \|^2,
   \label{thm1_proof: eqn1_n}
  \end{split}
\end{equation}
where $u = \Phi - \tilde{S}^{\mathsf{t}}\gamma$. We use the above equation \eqref{thm1_proof: eqn1_n} and combine with equation \eqref{thm1_proof: eqn0} to get
\begin{equation}
\begin{split}
  &  4\pi^{\mathsf{min}} \|(u + \tilde{S}^{\mathsf{t}}\gamma)^{\mathsf{T}}Q(u) \|^2 \leq \epsilon   \\
  &  \lambda_{\mathsf{min}}(Q(u)Q(u)^{\mathsf{T}})\|u + \tilde{S}^{\mathsf{t}}\gamma \|^2 \leq \frac{\epsilon}{4\pi^{\mathsf{min}}} \\ 
\end{split}
\end{equation}

We divide the analysis into two cases a) $\|u\|^2 \geq  \epsilon$ and b) $\|u\|^2< \epsilon$.  Consider case a). We use the bound on minimum eigenvalue in Assumption \ref{assumption_singular_value} to get 

\begin{equation}
\begin{split}
      &  \frac{1}{4 \pi^{\mathsf{min}}\omega}  \|u + \tilde{S}^{\mathsf{t}}\gamma \|^2  \leq \frac{\epsilon}{4\pi^{\mathsf{min}}} \\ 
  &  \|u + \tilde{S}^{\mathsf{t}}\gamma \|^2  \leq \omega \epsilon 
 \end{split}
\end{equation}

Further since $\epsilon <1$ we get

% Further from Assumption \ref{assumption_singular_value} we can use the bound on $\eta \geq \frac{1}{4\pi^{\mathsf{min}}\omega}$ to  simplify the above 
\begin{equation}
     \|u + \tilde{S}^{\mathsf{t}}\gamma \|^2 \leq \omega\epsilon \implies \|\Phi\|^2 \leq \omega
\end{equation}

From the inductive bias assumption (Assumption \ref{assm8:regularity conditions on envmts and hphi}), we know that $ \|\Phi\|^2 > \omega$. Therefore, the only case that is left is $\|u\|^2< \epsilon$, which implies $\|\Phi- \tilde{S}^{\mathsf{T}}\gamma\|^2 \leq \epsilon$. This completes the proof.

\end{proof}

Before proving Proposition \ref{prop5: irm_cfd_chd}, we first establish that Assumptions \ref{assm6: linear_model} and \ref{assm8:regularity conditions on envmts and hphi} are sufficient to ensure that Assumption \ref{assm3: bounded loss and gradient} holds and we can thus use bounds $L$ and $L'$ defined in Assumption \ref{assm3: bounded loss and gradient}. 

\textbf{Proving Assumption \ref{assm3: bounded loss and gradient} for square loss $\ell$ from Assumptions \ref{assm6: linear_model},  \ref{assm8:regularity conditions on envmts and hphi}.} 
From  Assumption \ref{assm6: linear_model}, we have

\begin{equation}
\begin{split}
&Z^{e} = (Z_1^{e},Z_2^{e})\\
  &X^{e}= SZ^e\\    
  &\|X^{e}\| \leq \|S\|\|Z^{e}\| 
 \end{split}
  \label{ass3_ass68}
\end{equation}

From Assumption \ref{assm6: linear_model},  $\|S\|$ is bounded and $\|Z^{e}\|$ is bounded $\implies$ $\|X^{e}\|$ is bounded as well (from \eqref{ass3_ass68}). Therefore, there exists $X^{\mathsf{sup}}<\infty$, s.t. $\|X^{e}\|\leq X^{\mathsf{sup}}$. 

\begin{equation}
\begin{split}
    &Y^{e} = (\tilde{S}^{\mathsf{T}}\gamma)^{\mathsf{T}}X^{e}+\varepsilon^{e} \\
    &| Y^{e} | \leq  \|\tilde{S}^{\mathsf{T}}\gamma\|\|X^{e}\| + |\varepsilon^{e}|\implies \\
    & |Y^{e}| \leq \sqrt{\Omega}X^{\mathsf{sup}} + \varepsilon^{\mathsf{sup}} 
\end{split}
\label{ass3_ass68_1}
\end{equation}

In the last step of \eqref{ass3_ass68_1}, we use $\|\tilde{S}^{\mathsf{T}}\gamma\|^2 \leq \Omega$ (Assumption \ref{assm8:regularity conditions on envmts and hphi}),  $\|X^{e}\| \leq X^{\mathsf{sup}}$ derived above (\eqref{ass3_ass68}), and $|\varepsilon^{e}| \leq \varepsilon^{\mathsf{sup}}$ (Assumption \ref{assm6: linear_model}).  Therefore, $Y^{e}$ is bounded and there exists a $K$ such that $|Y^{e}| \leq K \leq \sqrt{\Omega}X^{\mathsf{sup}} + \varepsilon^{\mathsf{sup}}$. Therefore, $\forall \Phi \in \mathcal{H}_{\Phi}$ and for all $X^{e},Y^{e}$ sampled from the model in Assumption \ref{assm6: linear_model} we have

\begin{equation}
   \ell(\Phi(X^{e}),Y^{e}) = (Y^{e}-\Phi^{\mathsf{T}}X^{e})^2 \leq (K+ \sqrt{\Omega}X^{\mathsf{sup}})^2
   \label{ass3_ass68_2}
\end{equation}

\begin{equation}
    \begin{split}
        &  \Big|\frac{\partial \ell(w\cdot\Phi^{\mathsf{T}}X,Y)}{\partial w}\Big|_{w=1.0}\Big| = \big|\Phi^{\mathsf{T}}X( \Phi^{\mathsf{T}}X -Y)\big| \leq (\sqrt{\Omega}X^{\mathsf{sup}})(\sqrt{\Omega}X^{\mathsf{sup}}+K)
    \end{split}
    \label{ass3_ass68_3}
\end{equation}

From \eqref{ass3_ass68_2}, we conclude that $\ell(\Phi(\cdot), \cdot)$ is bounded and there exists an $L$ such that $|\ell(\Phi(\cdot), \cdot)| \leq L \leq(K+ \sqrt{\Omega}X^{\mathsf{sup}})^2 $. 
From \eqref{ass3_ass68_3}, we conclude that $\frac{\partial \ell(w.\Phi^{\mathsf{T}}X,Y)}{\partial w}\Big|_{w=1.0}$ is bounded and there exists an $L'$ such that  $\Big|\frac{\ell(w.\Phi(\cdot),\cdot)}{\partial w}\Big|_{w=1.0}\Big|\leq L^{'} \leq (\sqrt{\Omega}X^{\mathsf{sup}})(\sqrt{\Omega}X^{\mathsf{sup}}+K)$. 

Next, we prove Proposition \ref{prop5: irm_cfd_chd} from the main body of the manuscript. 
\begin{proposition} \label{prop5: irm_cfd_chd_append} 
Let $\ell$ be the square loss. Given $\epsilon\in (0,1)$  and a $\delta \in (0,1)$, if Assumptions  \ref{assm6: linear_model},  \ref{assm8:regularity conditions on envmts and hphi}, \ref{assumption_singular_value} hold and if the number of data points $|D|$ is greater than $ \frac{16L'^{4}}{\epsilon^2}\log\big(\frac{2|\mathcal{H}_{\Phi}|}{\delta}\big)$, then  with a probability at least $1-\delta$, every solution $\hat{\Phi}$ to EIRM (\eqref{eqn: EIRM_cons_swap} with $\frac{\epsilon}{2}$) satisfies  $\|\hat{\Phi}- (\tilde{S}^{\mathsf{T}}\gamma\big)\|^2 \leq  \epsilon$.
% Let $\ell$ be the square loss. For every  $\epsilon\in (0,\epsilon_{\mathsf{th}})$ and $\delta \in (0,1)$, if Assumptions  \ref{assm6: linear_model},  \ref{assm8:regularity conditions on envmts and hphi}, and \ref{assumption_singular_value} hold and if the number of data points $|D|$ is greater than $ \frac{16L'^{4}}{\epsilon^2}\log\big(\frac{2|\mathcal{H}_{\Phi}|}{\delta}\big)$, then  with a probability at least $1-\delta$, every solution $\hat{\Phi}$ to EIRM (\eqref{eqn: EIRM_cons_swap} with $\frac{\epsilon}{2}$ and $\epsilon<1$) satisfies  $\|\hat{\Phi} - \tilde{S}^{\mathsf{T}}\gamma\|^2 \leq \epsilon $.
\end{proposition}

\begin{proof}

Define an event $A$: $\{\tilde{D}: \forall \Phi \in \mathcal{H}_{\Phi}, |\hat{R}^{'}(\Phi)-R^{'}(\Phi)|\leq \frac{\epsilon}{2}\}$. If event $A$ happens, then 
\begin{equation}
\begin{split}
    &\hat{R}^{'}(\Phi) \leq \frac{\epsilon}{2}  \\
    & \hat{R}^{'}(\Phi) - R^{'}(\Phi) + R^{'}(\Phi) \leq \frac{\epsilon}{2} \\
    &  R^{'}(\Phi) \leq \frac{\epsilon}{2} + |\hat{R}^{'}(\Phi) - R^{'}(\Phi)  | \\
     & R^{'}(\Phi)\leq \epsilon
    \end{split}
    \label{proof: prop_irm_cfd_child_eqn1}
\end{equation}

If event $A$ happens, then  every solution $\hat{\Phi}$ to EIRM (\eqref{eqn: EIRM_cons_swap}) satisfies \eqref{proof: prop_irm_cfd_child_eqn1}. We can now use  Proposition \ref{prop_int: inf_sam_irm_approx}. If the Assumptions \ref{assm6: linear_model},  \ref{assm8:regularity conditions on envmts and hphi}, \ref{assumption_singular_value} hold and event $A$ happens, then for all the output of EIRM \eqref{eqn: EIRM_cons_swap}   $\|\hat{\Phi}- \tilde{S}^{t}\gamma\|^2<\epsilon$. Hence, all that remains to be shown is that event $A$ occurs with a probability at least $1-\delta$. Next, we show that if $|D| \geq \frac{16L'^{4}}{\epsilon^2}\log(\frac{2|\mathcal{H}_{\Phi}|}{\delta})$, then with probability $1-\delta$ event $A$ happens. 

Now we need to bound the probability $\mathbb{P}(A)$. We will find an upper bound on the failure probability using Hoeffding's inequality (Lemma \ref{lemma1: hoeffding}) and the bound $L'$ (derived in \eqref{ass3_ass68_3})  as follows. We redo the same analysis as was done in \eqref{prop2_proof:eqn5} for reader's convenience. 

$$\mathbb{P}(A) = 1-\mathbb{P}\Big(\Big\{\tilde{D}: \exists \Phi \in \mathcal{H}_{\Phi}, |\hat{R}^{'}(\Phi)-R^{'}(\Phi)|> \frac{\epsilon}{2}\Big\}\Big) = \mathbb{P}\Big(\bigcup_{\Phi \in \mathcal{H}_{\Phi}}\Big\{\tilde{D}: |\hat{R}^{'}(\Phi)-R^{'}(\Phi)|> \frac{\epsilon}{2}\Big\}\Big) $$ 
\begin{equation}
\begin{split}
& \mathbb{P}\Big(\bigcup_{\Phi \in \mathcal{H}_{\Phi}}\Big\{\tilde{D}: |\hat{R}^{'}(\Phi)-R^{'}(\Phi)|> \frac{\epsilon}{2}\Big\}\Big) \leq \sum_{\Phi\in \mathcal{H}_{\Phi}}\mathbb{P}\Big(\{\tilde{D}: |\hat{R}^{'}(\Phi)-R^{'}(\Phi)|> \frac{\epsilon}{2}\Big\}\Big)    
\leq 2|\mathcal{H}_{\Phi}| e^{-\frac{\epsilon^{2}|D|}{16L'^{4}}}\\
&2|\mathcal{H}_{\Phi}|e^{-\frac{\epsilon^{2}|D|}{16L'^{4}}}\leq \delta \implies \mathbb{P}(A^{c}) \leq \delta \\
& |D| \geq \frac{16L'^{4}}{\epsilon^2}\log\Big(\frac{2|\mathcal{H}_{\Phi}|}{\delta}\Big) \implies \mathbb{P}(A^{c}) \leq \delta
\end{split}
\end{equation}
Hence, we know that if $|D| \geq \frac{16L'^{4}}{\epsilon^2}\log\Big(\frac{2|\mathcal{H}_{\Phi}|}{\delta}\Big)$, then with probability $1-\delta$ event $A$ happens.

The proof characterized the property of $\hat{\Phi}$ that satisfies $\hat{R}^{'}(\hat{\Phi}) \leq \epsilon$. But how do we know such a $\hat{\Phi}$ exists; we show that a $\hat{\Phi}$ always exists. 
Consider a $\Phi$ that satisfiess the following.
\begin{equation}
\begin{split}
    & R^{'}(\Phi) \leq \frac{\epsilon}{2}  \\
     & \hat{R}^{'}(\Phi) - \hat{R}^{'}(\Phi) + R^{'}(\Phi) \leq \frac{\epsilon}{2} \\
      & \hat{R}^{'}(\Phi) \leq \frac{\epsilon}{2} + |\hat{R}^{'}(\Phi) - R^{'}(\Phi)  | \\
      & \hat{R}^{'}(\Phi)\leq \epsilon \; (\text{follows from event} \; A)
    \end{split}
    \label{proof: prop_irm_cfd_chd_eqn1}
\end{equation}

We know that $R^{'}(\tilde{S}^{\mathsf{T}}\gamma) = 0$. From \eqref{proof: prop_irm_cfd_chd_eqn1}, $\tilde{\mathcal{S}}^{\mathsf{T}}\gamma \in \mathcal{H}_{\Phi}$ satisfies $\hat{R}^{'}(\Phi)\leq \epsilon$. 

% Therefore, if event $A$ happens solutions to $ R^{'}(\Phi) \leq \frac{\epsilon}{2}$ are subset solutions of  $\hat{R}^{'}(\Phi) \leq \frac{3\epsilon}{4}$ which in turn is a subset of 
% $ R^{'}(\Phi) \leq \epsilon$. If $\tilde{S}^{t}\gamma \in \mathcal{H}_{\Phi}$, then $\tilde{S}^{t}\gamma \in  R^{'}(\Phi) \leq \frac{\epsilon}{2}$, which implies it is also contained in $\hat{R}^{'}(\Phi) \leq \frac{3\epsilon}{4}$. This shows that if we assume $\mathcal{H}_{\Phi}$ feasibility is guaranteed with probability $1-\delta$ as well. 

\end{proof}

\subsection{Extensions: Polynomial models, infinite hypothesis classes, binary-classification}

\subsubsection{Polynomial model}

In the main body of the work, we did not cover the polynomial model in detail due to space limitations. In this section, we study the polynomial model and show how the results for the linear model can be generalized to this case. We first begin by stating the model itself.

\begin{assumption} \label{pol_model}
\begin{equation}
    \begin{split}
    &    e \sim \mathsf{Categorical}(\pi^{e}), \pi^{e}>0 \forall e \in \mathcal{E}_{tr} \\
    &    Y^{e} = \gamma^{\mathsf{T}}\zeta_{p}^{c}\big(Z_1^{e}\big) + \varepsilon^{e}, \varepsilon^{e} \perp Z_1^{e},  \mathbb{E}[\varepsilon^{e}] = 0, \mathbb{E}[(\varepsilon^{e})^2] = \sigma^2, |\varepsilon^{e}| \leq \varepsilon^{\mathsf{sup}}   \\ 
      &  X^{e} = S(Z_1^{e}, Z_2^{e})
    \end{split}
\end{equation}

Assume that $Z_1^e$ component of $S$ is invertible, i.e. $\exists \tilde{S}$ such that $\tilde{S}(S(Z_1^e,Z_2^e))=Z_1^e$ and also $\tilde{S}^{t}\gamma \not =0 $. In the above $\zeta_p^{a}$ is a polynomial feature map of degree $p$ defined as $\zeta_p^{a}:\mathbb{R}^{a} \rightarrow \mathbb{R}^{a^{'}}$, where $a$ denotes the dimension of the input to the map $\zeta_p^{a}$, 
$\zeta_p^{a}(W) = [W, W\otimes W,\dots,  (W \otimes W ... \text{p times} \otimes W)] = [(W^{\otimes i})_{i=1}^{p}]$ and $\otimes$ is the Kronecker product. Also, $a^{'} = \sum_{i=1}^{p}a^{i}$. $\forall e \in \mathcal{E}_{tr}, \pi^{e} \geq \frac{\pi^{\mathsf{min}}}{|\mathcal{E}_{tr}|}$.  The support of distribution of $Z^{e}=[Z_1^e,Z_2^e]$, $\mathbb{P}^{e}_{Z^{e}}$, is bounded and the operator norm of $S$, $\|S\|=\sigma_{\mathsf{max}}(S)$ ($\sigma_{\mathsf{max}}(S)$ is maximum singular value of $S$), is also bounded.
\end{assumption}

Define $Z^{e}=(Z_1^e, Z_2^{e})$ and say $Z_1^e \in \mathbb{R}^c$ and $Z_2^{e}\in \mathbb{R}^{d}$.

Can we directly use the analysis from the linear case? We cannot directly use the polynomial map for the features as we also need to find an appropriate transformation of the matrix $S$, which preserves the linear relationship between the transformed features and the transformed variables $Z$. We carry out this exercise below.

Define $\bar{S} = \mathsf{diag}\Big[(S^{\otimes i})_{i=1}^{p}\Big]$, where $\bar{S}$ is a block diagonal matrix with diagonal matrices defining it given as $\Big\{S, S^{\otimes2}\dots , S^{\otimes p} \Big\}$. 

Define $\bar{X^{e}} = \zeta_p^{n}\big(X^{e}\big)$, where $\zeta_p^{n}(X^{e})$ is the polynomial feature map of degree $p$ of $n$ dimensional input $X^{e}$.
Similarly,  define $\bar{Z_1^{e}} = \zeta_p^{c}(Z_1^{e})$, where $\zeta_p^{c}(Z_1^{e})$ is the polynomial feature map of degree $p$ of $c$ dimensional input $Z_1^{e}$ and define $\bar{Z^{e}} = \zeta_p^{c+d}(Z^{e})$, where $\zeta_p^{c+d}(Z^{e})$ is the polynomial feature map of degree $p$ of $c+d$ dimensional input $Z^{e}=(Z_1^e, Z_2^{e})$.  Observe that each component of $\bar{Z_1^{e}}$ is also in $\bar{Z}$.

From the model we know that $X^{e} = S Z^{e}$. We would like to remind the reader of the mixed product property of tensors. 
Consider matrices $A\in \mathbb{R}^{i\times j}$, $B\in \mathbb{R}^{k\times l}$ $C\in \mathbb{R}^{j\times p}$, $D\in \mathbb{R}^{l\times q}$. 
\begin{equation}
    (A\otimes B)(C\otimes D) = (AC)\otimes (BD)
\end{equation}

In the expressions that follow, we exploit the mixed-product property of Kronecker product stated above.

\begin{equation}
    \begin{split}
&        \bar{S}\bar{Z^{e}} = \mathsf{diag}\Big[(S^{\otimes i})_{i=1}^{p}\Big] \Big[(Z^{e,\otimes i})_{i=1}^{p}\Big]  = \Big[(SZ^{e})^{\otimes i})_{i=1}^{p}\Big] = \Big[(X^{e,\otimes i})_{i=1}^{p}\Big]  = \zeta_p^{n}(X^{e}) = \bar{X}^{e}
    \end{split}
    \label{poly_proof_eqn1}
\end{equation}

Define $\bar{\tilde{S}} = \mathsf{diag}\Big[(\tilde{S}^{\otimes i})_{i=1}^{p}\Big]$. 
\begin{equation}
    \begin{split}
        \bar{\tilde{S}}X^{e} & =  \bar{\tilde{S}}\bar{S}\bar{Z}^{e}  = \mathsf{diag}\Big[(\tilde{S}^{\otimes i})_{i=1}^{p}\Big]\mathsf{diag}\Big[(S^{\otimes i})_{i=1}^{p}\Big]\Big[(Z^{e,\otimes i})_{i=1}^{p}\Big]  \\ 
        & \mathsf{diag}\Big[((\tilde{S}S)^{\otimes i})_{i=1}^{p}\Big]\Big[(Z^{e,\otimes i})_{i=1}^{p}\Big]  = \Big[((\tilde{S}SZ^{e})^{\otimes i})_{i=1}^{p}\Big] = \Big[((Z_1^{e})^{\otimes i})_{i=1}^{p}\Big] = \zeta_p^{c}(Z_1^{e}) = \bar{Z_1^{e}}
    \end{split}
     \label{poly_proof_eqn2}
\end{equation}

The dimensionality of $\bar{X}^{e}$ is $n^{'} = \sum_{i=1}^{p} n^{i} = \frac{n^{p+1}-n}{n-1} $. 

\begin{assumption} \label{assm_pol:regularity conditions on envmts and hphi} \textbf{Inductve bias.} $\mathcal{H}_{\Phi}$ is a  finite set of linear models (bounded) parametrized by $\Phi \in \mathbb{R}^{n'}$.  $\bar{\tilde{S}}^{\mathsf{T}}\gamma\in \mathcal{H}_{\Phi}$.   $\exists \; \omega>0, \Omega>0,$ $\forall \Phi \in \mathcal{H}_{\Phi},\omega \leq \|\Phi\|^2 \leq \Omega$. 
\end{assumption}

We next compute the norm of $\bar{S}$ in terms of the norm of $S$.  Recall we are using operator norm defined as $\|S\| = \sigma_{\mathsf{max}}(S)$.
$\|\bar{S}\|= \Big\|\mathsf{diag}\Big[(S^{\otimes i})_{i=1}^{p}\Big]\Big\|$. Since $\bar{S}$ is a diagonal matrix $\|\bar{S}\| = \max_{i=1,..,p}\{\|S^{\otimes i }\|\}$. Also, note that$ \|S^{\otimes i }\| = \|S\|^{i}$ \cite{laub2005matrix}. Therefore, $\|\bar{S}\| = \max_{i=1....,p}\{\|S\|^{i}\}$. Hence, if $\|S\|$ is bounded, $\|\bar{S}\|$ is also bounded. 

Also, $\|\bar{Z^{e}}\|^2 = \sum_{i}\|Z^{e,\otimes i}\|^2$. Observe that $\|Z^{e,\otimes i}\| = \|Z^{e}\|^{i}$. Hence, if $\|Z^{e}\|$ is bounded, $\|\bar{Z^{e}}\|$ is also bounded. Since $\bar{X^{e}} =\bar{S}\bar{Z^{e}}$. We can conclude that $\|\bar{X^e}\|$ is also bounded. We can now follow the same line of reasoning as in \eqref{ass3_ass68}, \eqref{ass3_ass68_1},\eqref{ass3_ass68_2}, \eqref{ass3_ass68_3} to conclude that the loss and the gradient of the loss are bounded. 

We rewrite the above model in Assumption \ref{pol_model} as a linear model in terms of the transformed features.

\begin{equation}
    \begin{split}
    &    e \sim \mathsf{Categorical}(\pi^{e}), \pi^{e}>0 \forall e \in \mathcal{E} \\
    &    Y^{e} = \gamma^{t}\bar{Z_1} + \varepsilon^{e}, \varepsilon^{e} \perp \bar{Z_1}^{e},  \mathbb{E}[\varepsilon^{e}] = 0, \mathbb{E}[(\varepsilon^{e})^2] = \sigma^2, |\varepsilon^{e}| \leq \varepsilon^{\mathsf{sup}}    \\ 
      &  \bar{X}^{e} = \bar{S}\bar{Z}^{e}
    \end{split}
\end{equation}

We showed above in \eqref{poly_proof_eqn2} that $\bar{Z_1^{e}}$ defined in \eqref{poly_proof_eqn1} component of $\bar{Z^{e}}$ is invertible, $\bar{\tilde{S}}\bar{S}\bar{Z^{e}} = \bar{Z_1^{e}}$. We have also shown above that support of $\bar{Z}^{e}$ is bounded and the norm of $\bar{S}$ is bounded and as a result $\|\bar{X^e}\|$ and the loss and the gradient of the loss (conditions in Assumption \ref{assm3: bounded loss and gradient} are satisfied) are bounded. We adapt the Assumption \ref{assumption_singular_value} for the polynomial case below.

% \begin{assumption}
% \textbf{Linear general position of training environments.}
% \label{assm: linear_gen_posn_pol}
% A set of training environments $\mathcal{E}_{tr}$ is said to lie in a linear general position of degree $r$ for some $r\in \mathbb{N}$ if $|\mathcal{E}_{tr}| > n^{'}-r + n^{'}/r$ and for non-zero $x\in \mathbb{R}^{n^{'}}$ and 

% \begin{equation}
%     \mathsf{dim}\Big(\mathsf{span}\Big\{\mathbb{E}^{e}[\bar{X}^{e}\bar{X}^{e,\mathsf{T}}]x-\mathbb{E}^{e}[\bar{X}^{e,\mathsf{T}}\varepsilon^{e}]\Big\}_{e\in \mathcal{E}_{tr}}\Big) > n^{'}-r
% \end{equation}
% \end{assumption}
We denote $\mathbb{E}[\bar{X}^{e} \bar{X}^{e, \mathsf{T}}]=\bar{\Sigma}_{e}$

% \begin{assumption}
% \label{assm:reg_pol}
% For all the environments $e\in \mathcal{E}_{tr}$, $\bar{\Sigma}_{e}$ is positive definite. 
% \end{assumption}

For each environment $e\in \mathcal{E}_{tr}$, define  $\bar{c}^e = \mathbb{E}^{e}[\bar{X}^{e}\varepsilon^{e}]$. Also, define a matrix $\bar{Q}(x)$, where each column of the matrix $Q(x)$ is $\bar{\Sigma}^e x - \bar{c}^e$.

\begin{assumption}
\label{assumption_singular_value_pol}
For all $\|x\|^2\geq \epsilon >0$, where $\epsilon$ is the slack in the IRM penalty in \eqref{eqn: IRM_cons_swap}, the minimum singular value of $\bar{Q}(x)$, $\lambda_{\mathsf{min}}(\bar{Q}(x) \bar{Q}(x)^{\mathsf{T}}) \geq \frac{1}{4\pi^{\mathsf{min}}\omega} $.
\end{assumption}

% Define the minimum eigenvalue over all the matrices $\bar{\Sigma}_{e}$ as $\bar{\lambda}_{\mathsf{min}} = \min_{e\in \mathcal{E}_{tr}}\lambda(\bar{\Sigma}_{e})$. Define $\bar{\epsilon}_{\mathsf{th}} = \frac{24-16\sqrt{2}}{3} \frac{\pi^{\mathsf{min}}}{|\mathcal{E}_{tr}|} (\omega \bar{\lambda}_{\mathsf{min}})^2 $.

From the analysis in this section, we  see that we have been able to construct a linear model identical to \ref{assm6: linear_model}, where the role of $X^{e}$, $Z^{e}$, $S$, $\tilde{S}$, $\Sigma^{e}$, $\lambda_{\mathsf{min}}$, $\epsilon_{\mathsf{th}}$  is taken by $\bar{X^{e}}$, $\bar{Z^{e}}$, $\bar{S}$,  $\bar{\tilde{S}}$,  $\bar{\Sigma}^{e}$, $\bar{\lambda}_{\mathsf{min}}$, $\bar{\epsilon}_{\mathsf{th}}$. Now we are ready to use the result already proven for linear model and state the next Proposition in terms of the parameters for the polynomial model.

% \begin{assumption} \label{assm11:regularity conditions on envmts and hphi} \textbf{Regularity conditions on environments and $\mathcal{H}_{\Phi}$}
% \begin{itemize}
% \item $\forall e \in \mathcal{E}_{tr}, \pi^{e} \geq \frac{\pi^{\mathsf{min}}}{|\mathcal{E}_{tr}|}$
% % \item $\exists \; p_{\mathsf{min}}>0$ such that 
% % $\|\tilde{S}^{\mathsf{T}}\gamma\|^2 \geq \min_{e \in \mathcal{E}_{tr}}\{\frac{2 p_{\mathsf{min}}}{\lambda_{\mathsf{min}}(\Sigma_{e})}\}$ 
% \item  $\bar{\lambda}_{\mathsf{min}}  = \min_{e\in \mathcal{E}_{tr}}\lambda_{\mathsf{min}}(\bar{\Sigma}_{e})>0$
% \item $\mathcal{H}_{\Phi}$ is a family of linear models parametrized by $\Phi \in \mathbb{R}^{n^{'}}$  and $\bar{\tilde{S}}^{\mathsf{T}}\gamma \in \mathcal{H}_{\Phi}$ 
% \item $\exists \; \omega>0$ such that 
% $\|\bar{\tilde{S}}^{\mathsf{T}}\gamma\|^2 \geq 2\omega$, $\forall \Phi \in \mathcal{H}_{\Phi}, \|\Phi\|^2 \geq \omega$
% \end{itemize}
% \end{assumption}

% \begin{proposition}
% Let $\ell$ be the square loss. 
% If Assumptions \ref{pol_model}, 
% \end{proposition}

\begin{proposition}  Let $\ell$ be the square loss. Given $\epsilon\in (0,1)$  and a $\delta \in (0,1)$,  if Assumptions \ref{pol_model}, \ref{assm_pol:regularity conditions on envmts and hphi}, \ref{assumption_singular_value_pol} hold and if the number of data points $|D|$ is greater than $ \frac{16L'^{4}}{\epsilon^2}\log\big(\frac{2|\mathcal{H}_{\Phi}|}{\delta}\big)$, then  with a probability at least $1-\delta$, every solution $\hat{\Phi}$ to EIRM (\eqref{eqn: EIRM_cons_swap}) satisfies  $\|\hat{\Phi} - \bar{\tilde{S}}^{\mathsf{T}}\gamma\|^2\leq \epsilon$. 
% where  $\alpha \in [\frac{1}{1+\tau \sqrt{\epsilon}},  \frac{1}{1-\tau\sqrt{\epsilon}}]$. 
\end{proposition}

\subsubsection{Infinite Hypothesis Classes}

In the work so far we have assumed that the hypothesis class $\mathcal{H}_{\Phi}$ is finite. In this section, we discuss infinite hypothesis class extensions. Before we do that we state an important result on covering numbers that we will use soon. 

\begin{lemma}\label{lemma:covering_number} \cite{shalev2014understanding} Define a set $\mathcal{A}= \{a\in \mathbb{R}^{k}, \|a\|^2 \leq A^{\mathsf{sup}}\}$.  Covering number  for $\eta$-cover of  $\mathcal{A}$ given as $N_{\eta}(\mathcal{A})$ is bounded as 
\begin{equation}
   N_{\eta}(\mathcal{A}) \leq \Big(\frac{2\sqrt{A^{\mathsf{sup}} k}}{\eta}\Big)^{k}
    \label{prop5: eqn1}
\end{equation}
\end{lemma}

\textbf{A. Infinite Hypothesis Class: Confounders and Anti-causal variables}

In this section, we seek to extend Proposition \ref{prop5: irm_cfd_chd} to infinite hypothesis classes. 

We restate the Assumption \ref{assm8:regularity conditions on envmts and hphi} for linear models.
\begin{assumption} \label{assm_infinite:regularity conditions on envmts and hphi} \textbf{Inductve bias.} $\mathcal{H}_{\Phi}$ is a  set of linear models (bounded) parametrized by $\Phi \in \mathbb{R}^{n}$ $\mathcal{H}_{\Phi} = \{ \Phi\in \mathbb{R}^{n}, \; 0 <\omega \leq \|\Phi\|^2 \leq \Omega\}$.   $\tilde{S}^{\mathsf{T}}\gamma \in \mathcal{H}_{\Phi}$. 
% \begin{itemize}
% \item , 
% \item 
% \end{itemize}
\end{assumption}
Note that the only difference between Assumption \ref{assm8:regularity conditions on envmts and hphi} and Assumption \ref{assm_infinite:regularity conditions on envmts and hphi} is that the hypothesis class is not required to be finite anymore.

We already established in \eqref{ass3_ass68_1}, \eqref{ass3_ass68_2}, \eqref{ass3_ass68_3} that from Assumptions \ref{assm6: linear_model} and \ref{assm8:regularity conditions on envmts and hphi}, we can show that the conditions in Assumption \ref{assm3: bounded loss and gradient} hold, i.e., loss and the gradient of the loss are bounded, and also $X^{e}, Y^{e}$ are bounded . The same conclusion follows from Assumptions \ref{assm6: linear_model} and Assumption \ref{assm_infinite:regularity conditions on envmts and hphi}. Hence, for the rest of this section, we can state that $\|X^{e}\| \leq X^{\mathsf{sup}}$, $|Y^{e}| \leq K$, the square loss $\ell(\Phi(\cdot),\cdot)$ is bounded by $L$ and $\frac{\partial \ell(w\cdot (\Phi(\cdot),\cdot)}{\partial w}\Big|_{w=1.0}$ is bounded by $L'$. In the next lemma, we aim to show that if Assumption \ref{assm6: linear_model} and \ref{assm_infinite:regularity conditions on envmts and hphi} hold, then $R^{'}(\Phi)$ is Lipschitz continuous. 

% \begin{assumption} \label{assm13:regularity conditions on envmts and hphi} \textbf{Regularity conditions on environments and $\mathcal{H}_{\Phi}$}
% \begin{itemize}
% \item $\forall e \in \mathcal{E}_{tr}, \pi^{e} \geq \frac{\pi^{\mathsf{min}}}{|\mathcal{E}_{tr}|}$
% % \item $\exists \; p_{\mathsf{min}}>0$ such that 
% % $\|\tilde{S}^{\mathsf{T}}\gamma\|^2 \geq \min_{e \in \mathcal{E}_{tr}}\{\frac{2 p_{\mathsf{min}}}{\lambda_{\mathsf{min}}(\Sigma_{e})}\}$ 
% \item  $\lambda_{\mathsf{min}}  = \min_{e\in \mathcal{E}_{tr}}\lambda_{\mathsf{min}}(\Sigma_{e})>0$
% \item $\mathcal{H}_{\Phi}$ is  a set of linear models parametrized by $\Phi \in \mathbb{R}^{n}$ and $\tilde{S}^{\mathsf{T}}\gamma \in \mathcal{H}_{\Phi}$, 
% \item $\exists \; \omega>0$ such that 
% $\|\tilde{S}^{\mathsf{T}}\gamma\|^2 \geq 2\omega$, $\forall \Phi \in \mathcal{H}_{\Phi}, P\geq \|\Phi\|^2 \geq \omega$
% \item $\forall x \in \mathcal{X}, \|x\| \leq U$
% \end{itemize}
% \end{assumption}

% Assumption \ref{assm14: lipschitz loss and gradient} is implied by the conditions on hypothesis class made in the Assumption \ref{assm_infinite:regularity conditions on envmts and hphi} and we show that next.

% \begin{lemma}
% If $H_{\Phi}$ is a linear hypothesis class with $\Phi \in \mathbb{R}^{n}$, $\|\Phi\|^2 \leq P$. and $\|x\| \leq U$, then $R^{'}(\Phi)$ is Lipschitz continous with a Lipschitz constant $C^{'}\leq 2PU(PU+K) (2PU+K)U$. 
% \end{lemma}

\begin{lemma}
If Assumption \ref{assm6: linear_model} and \ref{assm_infinite:regularity conditions on envmts and hphi} hold, then $R^{'}(\Phi)$ is Lipschitz continuous in $\Phi$.
\end{lemma}
\begin{proof}
The output of the model $|\Phi^{\mathsf{T}}X| \leq \|\Phi\|\|X\| \leq \sqrt{\Omega}X^{\mathsf{sup}} $ (From Cauchy-Schwarz). 

\begin{equation}
    \begin{split}
         R^{'}(\Phi) &= \sum_{e}\pi^{e}\mathbb{E}^{e}\Big[\Phi^{\mathsf{T}}X^{e}\big( \Phi^{\mathsf{T}}X^{e} -Y^{e}\big)\Big]^2 \\
        |R^{'}(\Phi_1) - R^{'}(\Phi_2)|  &= \Big|\sum_{e}\pi^{e}\Big(\mathbb{E}^{e}\Big[\Phi_1^{\mathsf{T}}X^{e}\big( \Phi_1^{\mathsf{T}}X^{e} -Y^{e}\big)\Big]^2 -\mathbb{E}^{e}\Big[\Phi_1^{\mathsf{T}}X^{e}\big( \Phi_2^{\mathsf{T}}X^{e} -Y^{e}\big)\Big]^2\Big)\Big| \\
        & \leq \sum_{e}\pi^{e}\Big|\mathbb{E}^{e}\Big[\Phi_1^{\mathsf{T}}X^{e}\big( \Phi_1^{\mathsf{T}}X^{e} -Y^{e}\big)\Big]^2 -\mathbb{E}^{e}\Big[\Phi_2^{\mathsf{T}}X^{e}( \Phi_2^{\mathsf{T}}X^{e} -Y^{e})\Big]^2\Big|
    \end{split} 
    \label{proof: lemma_lipsch1}
\end{equation}

We bound each term in the summation in the \eqref{proof: lemma_lipsch1} above 

\begin{equation}
    \begin{split}
    & \Big|\mathbb{E}^{e}\Big[\Phi_1^{\mathsf{T}}X^{e}\big( \Phi_1^{\mathsf{T}}X^{e} -Y^{e}\big)\Big]^2 -\mathbb{E}^{e}\Big[\Phi_2^{\mathsf{T}}X^{e}( \Phi_2^{\mathsf{T}}X^{e} -Y^{e})\Big]^2\Big| \\
                & =\Big| \Big(\mathbb{E}^{e}\Big[\Phi_1^{\mathsf{T}}X^{e}\big( \Phi_1^{\mathsf{T}}X^{e} -Y^{e}\big)\Big] -\mathbb{E}^{e}\Big[\Phi_2^{\mathsf{T}}X^{e}\big( \Phi_2^{\mathsf{T}}X^{e} -Y^{e}\big)\Big]\Big)\Big(\mathbb{E}^{e}\Big[\Phi_1^{\mathsf{T}}X^{e}\big( \Phi_1^{\mathsf{T}}X^{e} -Y^{e}\big)\Big] + \\ &
               \;\;\;\;\; \mathbb{E}^{e}\Big[\Phi_2^{\mathsf{T}}X^{e}\big( \Phi_2^{\mathsf{T}}X^{e} -Y^{e}\big)\Big]\Big) \Big| \\ 
        &\leq \Big| \mathbb{E}^{e}\Big[\Phi_1^{\mathsf{T}}X^{e}\big( \Phi_1^{\mathsf{T}}X^{e} -Y^{e}\big)\Big] -\mathbb{E}^{e}\Big[\Phi_2^{\mathsf{T}}X^{e}\big( \Phi_2^{\mathsf{T}}X^{e} -Y^{e}\big)\Big]
        \Big| 2\sqrt{\Omega}X^{\mathsf{sup}}(\sqrt{\Omega}X^{\mathsf{sup}}+K) \\ 
        & \leq \Big| \mathbb{E}^{e}\Big[\Phi_1^{\mathsf{T}}X^{e}\big( \Phi_1^{\mathsf{T}}X^{e} -Y^{e}\big)\Big] -\mathbb{E}^{e}\Big[\Phi_2^{\mathsf{T}}X^{e}\big( \Phi_1^{\mathsf{T}}X^{e} -Y^{e}\big)\Big] + \mathbb{E}^{e}\Big[\Phi_2^{\mathsf{T}}X^{e}\big( \Phi_1^{\mathsf{T}}X^{e} -Y^{e}\big)\Big] - \\  &\;\;\;\;\;\mathbb{E}^{e}\Big[\Phi_2^{\mathsf{T}}X^{e}( \Phi_2^{\mathsf{T}}X^{e} -Y^{e})\Big]\Big| 2\sqrt{\Omega}X^{\mathsf{sup}}(\sqrt{\Omega}X^{\mathsf{sup}}+K)  \\ 
        & = \Big| \mathbb{E}^{e}\Big[\big(\Phi_1^{\mathsf{T}}X^{e}-\Phi_2^{\mathsf{T}}X^{e}\big)\big( \Phi_1^{\mathsf{T}}X^{e} -Y^{e}\big)\Big] + \mathbb{E}^{e}\Big[\Phi_2^{\mathsf{T}}X^{e}\big( \Phi_1^{\mathsf{T}}X^{e} -\Phi_2^{\mathsf{T}}X^{e}\big)\Big] \Big| 2\sqrt{\Omega}X^{\mathsf{sup}}(\sqrt{\Omega}X^{\mathsf{sup}}+K)  \\
        & \leq \Big| \mathbb{E}^{e}\Big[\big(\Phi_1^{\mathsf{T}}X^{e}-\Phi_2^{\mathsf{T}}X^{e}\big)\big( \Phi_1^{\mathsf{T}}X^{e} -Y^{e}\big)\Big] \Big| 2\sqrt{\Omega}X^{\mathsf{sup}}(\sqrt{\Omega}X^{\mathsf{sup}}+K) + \\ & \;\;\;\;\;\Big|\mathbb{E}^{e}\Big[\Phi_2^{\mathsf{T}}X^{e}\big( \Phi_1^{\mathsf{T}}X^{e} -\Phi_2^{\mathsf{T}}X^{e}\big)\Big] \Big| 2\sqrt{\Omega}X^{\mathsf{sup}}(\sqrt{\Omega}X^{\mathsf{sup}}+K) 
    \end{split}
    \label{proof: lemma_lipsch2}
\end{equation}

We bound each term in the last line in \eqref{proof: lemma_lipsch2}  
\begin{equation}
    \begin{split}
   & \Big| \mathbb{E}^{e}\Big[\big(\Phi_1^{\mathsf{T}}X^{e}-\Phi_2^{\mathsf{T}}X^{e}\big)\big( \Phi_1^{\mathsf{T}}X^{e} -Y^{e}\big)\Big] \Big| \\ 
   & \leq \mathbb{E}^{e}\Big[\Big|\big(\Phi_1^{\mathsf{T}}X^{e}-\Phi_2^{\mathsf{T}}X^{e}\big)\Big| \Big| \big( \Phi_1^{\mathsf{T}}X^{e} -Y^{e}\big)\Big|\Big] \\
   & \leq (\sqrt{\Omega}X^{\mathsf{sup}}+K) \mathbb{E}^{e}\big[\big|(\Phi_1-\Phi_2)^{\mathsf{T}}X^{e})\big|\big] \\ 
   & \leq  (\sqrt{\Omega}X^{\mathsf{sup}}+K) \mathbb{E}^{e}\big[\|\Phi_1 - \Phi_2\|X^{\mathsf{sup}}\big] \;\; (\text{Cauchy-Scwarz})\\  
   & \leq (\sqrt{\Omega}X^{\mathsf{sup}}+K) X^{\mathsf{sup}}\|\Phi_1- \Phi_2\|
    \end{split}
      \label{proof: lemma_lipsch3}
\end{equation}

\begin{equation}
    \begin{split}
   & \Big|\mathbb{E}^{e}\Big[\Phi_2^{\mathsf{T}}X^{e}\big( \Phi_1^{\mathsf{T}}X^{e} -\Phi_2^{\mathsf{T}}X^{e}\big)\Big] \Big| \\ 
   & \leq \mathbb{E}^{e}\Big[\Big|\Phi_2^{\mathsf{T}}X^{e}\Big|\Big|\big( \Phi_1^{\mathsf{T}}X^{e} -\Phi_2^{\mathsf{T}}X^{e}\big)\Big|\Big] \\
   & \leq \sqrt{\Omega}X^{\mathsf{sup}}\mathbb{E}^{e}\Big[\Big|(\Phi_1-\Phi_2)^{\mathsf{T}}X^{e})\Big|\Big]\;\; (\text{Cauchy-Scwarz}) \\ 
   & \leq \sqrt{\Omega}X^{\mathsf{sup}}\mathbb{E}^{e}\big[\|\Phi_1 - \Phi_2\|X^{\mathsf{sup}}\big]\\  
   & \leq\sqrt{\Omega}(X^{\mathsf{sup}})^2 \|\Phi_1- \Phi_2\|
    \end{split}
    \label{proof: lemma_lipsch4}
\end{equation}

Substituting \eqref{proof: lemma_lipsch3} and \eqref{proof: lemma_lipsch4} in \eqref{proof: lemma_lipsch2} to get
\begin{equation}
|R^{'}(\Phi_1) - R^{'}(\Phi_2)| \leq 2\sqrt{\Omega}(X^{\mathsf{sup}})^2(\sqrt{\Omega}X^{\mathsf{sup}}+K) (2\sqrt{\Omega}X^{\mathsf{sup}}+K) \|\Phi_1 - \Phi_2\|
\end{equation}
Therefore, $R^{'}$ is Lipschitz with a constant $C^{'} \leq2\sqrt{\Omega}(X^{\mathsf{sup}})^2(\sqrt{\Omega}X^{\mathsf{sup}}+K) (2\sqrt{\Omega}X^{\mathsf{sup}}+K)$.

\end{proof}

% \textbf{Bound on $L$ from Assumption \ref{assm13:regularity conditions on envmts and hphi}} 

% \begin{equation}
%   \ell(\Phi(X),Y) = (Y-\Phi^{\mathsf{T}}X)^2 \leq (K+ PU)^2
% \end{equation}
% Therefore $L\leq (K+PU)^2$

% \textbf{Bound on $L'$ from Assumption \ref{assm13:regularity conditions on envmts and hphi}} 
% \begin{equation}
%     \begin{split}
%         &  |\frac{\ell(w.\Phi^{\mathsf{T}}X,Y)}{\partial w}|_{w=1.0}| = |\Phi^{\mathsf{T}}X( \Phi^{\mathsf{T}}X -Y)| \leq (PU)(PU+K)
%          \leq 2(PU)^2((PU)^2+K^2)
%     \end{split}
% \end{equation}

% Therefore $L'\leq (PU)(PU+K)$
We just showed that $R^{'}$ is Lipschitz continuous and we set its Lipschitz constant as $C^{'}$.
\begin{proposition}
 Let $\ell$ be the square loss. For every $\epsilon \in (0,\epsilon^{\mathsf{th}})$ and $\delta\in (0,1)$,  if Assumptions \ref{assm6: linear_model}, \ref{assumption_singular_value}, \ref{assm_infinite:regularity conditions on envmts and hphi} hold and if the number of samples $|D|$ is greater than  $\frac{32L'^4}{\epsilon^2} \Big[n\log\Big(\frac{16C^{'}\sqrt{\Omega n}}{\epsilon}\Big) + \log\Big(\frac{2}{\delta}\Big)\Big] $, then  with a probability at least $1-\delta$ every solution $\hat{\Phi}$ to EIRM (\eqref{eqn: EIRM_cons_swap} with $\frac{\epsilon}{2}$) satisfies  $\|\hat{\Phi} - \tilde{S}^{\mathsf{T}}\gamma\|^2 \leq \epsilon$.
\end{proposition}
%  [\frac{1}{1+\frac{1}{\omega\lambda_{\mathsf{min}}}\sqrt{\frac{3\epsilon |\mathcal{E}_{tr}|}{2\pi^{\mathsf{min}}}}},  \frac{1}{1-\frac{1}{\omega\lambda_{\mathsf{min}}}\sqrt{\frac{3\epsilon |\mathcal{E}_{tr}|}{2\pi^{\mathsf{min}}}}}] =

\begin{proof}

Following the proof of Proposition \ref{prop5: irm_cfd_chd}, our goal is to compute the probability of event $A$: $\{\forall \Phi \in \mathcal{H}_{\Phi}, |\hat{R}^{'}(\Phi)-R^{'}(\Phi)|\leq \frac{\epsilon}{2}\}$. 
We  construct a minimum cover of size $N_{\eta}(\mathcal{H}_{\Phi})$ (See Lemma \ref{lemma:covering_number}) with points $\mathcal{C} =  \{\Phi_j\}_{j=1}^{b}$.

Compute the probability of failure at one point $\Phi_j$ in the cover 
\begin{equation}
    \mathbb{P}\Big[\Big\{\tilde{D}: |R^{'}(\Phi_j) - \hat{R}^{'}(\Phi_j) |>\frac{\epsilon}{4}\Big\}\Big] < 2e^{-\frac{\epsilon^2 |D|}{32L'^{4}}}
\end{equation}

We use union bound to bound the probability of failure over the cover $\mathcal{C}$ as follows 
\begin{equation}
    \mathbb{P}\Big[\Big\{\tilde{D}: \max_{\Phi_j\in \mathcal{C}}|R^{'}(\Phi_j) - \hat{R}^{'}(\Phi_j) |>\frac{\epsilon}{4}\Big\}\Big] < 2N_{\eta}(\mathcal{H}_{\Phi})e^{-\frac{\epsilon^2 |D|}{32L'^4}}
     \label{proof: inf_hyp_0}
\end{equation}

Now consider any $\Phi \in \mathcal{H}_{\Phi}$ and suppose $\Phi_{j}$ is nearest point to it in the cover.

\begin{equation} 
\begin{split}
 &|R^{'}(\Phi) - \hat{R}^{'}(\Phi) | =  |R^{'}(\Phi) - R^{'}(\Phi_j) +R^{'}(\Phi_j) -  \hat{R}^{'}(\Phi_j )+ \hat{R}^{'}(\Phi_j )- \hat{R}^{'}(\Phi) | \\ & \leq |R^{'}(\Phi) - R^{'}(\Phi_j)| + |R^{'}(\Phi_j) -  \hat{R}^{'}(\Phi_j )| + |\hat{R}^{'}(\Phi_j )- \hat{R}^{'}(\Phi)| \leq   |R^{'}(\Phi_j) - \hat{R}^{'}(\Phi_j) | + 2\eta C^{'} 
 \end{split}
\end{equation}
In the above simplification, we exploited the Lipschitz continuity of $R^{'}$. Therefore, for each $\Phi \in \mathcal{H}_{\Phi}$
\begin{equation} 
\begin{split}
 &|R^{'}(\Phi) - \hat{R}^{'}(\Phi) |  \leq  \max_{\Phi_j\in \mathcal{C}} |R^{'}(\Phi_j) - \hat{R}^{'}(\Phi_j) | + 2\eta C^{'} \\ 
  &\max_{\Phi \in \mathcal{H}_{\Phi}}|R^{'}(\Phi) - \hat{R}^{'}(\Phi) |  \leq  \max_{\Phi_j\in \mathcal{C}} |R^{'}(\Phi_j) - \hat{R}^{'}(\Phi_j) | + 2\eta C^{'} 
 \end{split}
 \label{proof: inf_hyp_1}
\end{equation}
 Set $\eta = \frac{\epsilon}{8C^{'}}$ in \eqref{proof: inf_hyp_1}  and from \eqref{proof: inf_hyp_0}  with probability at least $1-N_{\eta}(\mathcal{H}_{\Phi})2e^{-\frac{\epsilon^2 |D|}{32L'^4}}$ 
\begin{equation}
     \max_{\Phi\in \mathcal{H}_{\Phi}} |R^{'}(\Phi) - \hat{R}^{'}(\Phi)  | \leq \epsilon/2\; \text{(since $\max_{\Phi_j\in \mathcal{C}}|R^{'}(\Phi_j) - \hat{R}^{'}(\Phi_j)  | \leq \epsilon/4$ )}
\end{equation}

We  bound $N_{\eta}(\mathcal{H}_{\Phi})2e^{-\frac{\epsilon^2 |D|}{32L'^4}}\leq \delta$ and solve for bound on $|D|$ to get 

\begin{equation} 
\begin{split}
  &  |D| \geq \frac{8L'^4}{\frac{\epsilon^2}{4}} \log\Big(\frac{2N_{\eta}(\mathcal{H}_{\Phi})}{\delta}\Big)\; \text{(Use Lemma \ref{lemma:covering_number})} \\ 
 &     |D| \geq \frac{32L'^4}{\epsilon^2} \Big[n\log\Big(\frac{16C^{'}\sqrt{\Omega n}}{\epsilon}\Big) + \log\Big(\frac{2}{\delta}\Big)\Big]
    \end{split}
    \label{proof: inf_hyp_2}
\end{equation}

Therefore, if condition in \eqref{proof: inf_hyp_2} holds, then event $A$ occurs and following the same argument as in the proof of Proposition \ref{prop5: irm_cfd_chd} the proof is complete.

\end{proof}

\textbf{B. Infinite hypothesis class: Lipschitz continuous functions}

In this section, we seek to extend Proposition \ref{prop2: scomp_irm} and \ref{prop4: emr_irm_real} to infinite hypothesis class of Lipschitz continuous functions that we formally define next.
Define a map $\Phi:\mathcal{P} \times \mathcal{X}\rightarrow \mathbb{R}$ 
from the parameter space $\mathcal{P}$ and the feature space $\mathcal{X}$ to reals. Each $p\in \mathcal{P}$ is a possible choice for the representation $\Phi(p,\cdot)$. Consider neural networks as an example, $\mathcal{P}$ represents the set of the values the weights of the network can take. 

\begin{assumption}\label{assm15: parametrized_model}
$\Phi:\mathcal{P} \times \mathcal{X}\rightarrow \mathbb{R}$ is a a Lipschitz continuous function (with Lipschitz constant say $Q$).

$\mathcal{P}\subset \mathbb{R}^{k}$ is closed and bounded, thus there exists a $P<\infty$ such that  $\forall p \in \mathcal{P},\; \|p\|^2\leq P$. 

$\mathcal{X}\subset \mathbb{R}^{n}$ is closed and bounded, thus there exists a $X^{\mathsf{sup}}<\infty$ such that $\forall x \in \mathcal{X},\; \|x\|\leq X^{\mathsf{sup}}$.

$\mathcal{Y}\subset \mathbb{R}$ is closed and bounded, thus there  exists a $K<\infty$ such that $\forall y \in \mathcal{Y},\; |y|\leq K$ 
\end{assumption}

\begin{assumption}
\label{assm14: lipschitz loss and gradient} \textbf{Lipschitz loss and gradient of loss.}
 $R(\Phi)$ is Lipschitz with a constant $C$, $R^{'}(\Phi)$ is Lipchitz with a constant $C^{'}$. 
\end{assumption}

% Following Assumption \ref{assm15: parametrized_model}, the conditions in Assumption \ref{assm3: bounded loss and gradient} and Assumption \ref{assm14: lipschitz loss and gradient} are satisfied.

\textbf{From  Assumption \ref{assm15: parametrized_model} derive the conditions in Assumption \ref{assm3: bounded loss and gradient} and Assumption \ref{assm14: lipschitz loss and gradient}}
 $\Phi:\mathcal{P} \times \mathcal{X}\rightarrow \mathbb{R}$ is a  continuous function defined over closed and bounded domain $\mathcal{P} \times \mathcal{X}$ (domain is compact) and as a result $\Phi$ is bounded say by $M$.

Consider square loss $\ell(\Phi(p,X), Y) = (Y-\Phi(p,X))^2 \leq  (M+K)^2$. Hence, there exists an  $L$ such that
 $|\ell(\Phi(\cdot), \cdot)| \leq L \leq (M+K)^2$. 

$\frac{\partial \ell(w.\Phi(\cdot), \cdot)}{\partial w}\Big|_{w=1.0}$ is bounded: $\Big|\frac{\partial \ell(w.\Phi(\cdot), \cdot)}{\partial w}\Big|_{w=1.0}\Big| = |(Y-\Phi(p,X))\Phi(p,X)| \leq (K+M)M$. Hence, there exists an $L'$ such that 
$|\frac{\partial \ell(w.\Phi(\cdot), \cdot)}{\partial w}\Big|_{w=1.0}| \leq L' \leq (K+M)M$
\begin{lemma}
If Assumption \ref{assm15: parametrized_model}  holds, then $R(\Phi(p,\cdot))$ and $R^{'}(\Phi(p,\cdot))$ are Lipschitz continuous in $p$.
\end{lemma}
$R$ is Lipschitz:
\begin{equation}
    \begin{split}
        |R(\Phi(p,\cdot)-R(\Phi(q,\cdot))|  & = \Big|\sum_{e}\pi^{e}\Big(\mathbb{E}^{e}[(\Phi(p,X^e)-Y^e)^2] - \mathbb{E}^{e}[(\Phi(q,X^e)-Y^e)^2]\Big)\Big| \\ 
        &\leq \sum_{e}\pi^{e}\Big|\mathbb{E}^{e}\Big[\big(\Phi(p,X^e)-Y^e\big)^2\Big] - \mathbb{E}^{e}\Big[\big(\Phi(q,X^e)-Y^e\big)^2\Big]\Big| \\
        & =\sum_{e}\pi^{e}\Big|\mathbb{E}^{e}\Big[\Phi(p,X^e)-\Phi(q,X^e)\Big]\mathbb{E}^{e}\Big[\Phi(p,X^{e}) + \Phi(q,X^{e}) -2Y\Big]\Big| \\ 
        & \leq \sum_{e}\pi^{e}\mathbb{E}^{e}\Big[\Big|\Phi(p,X^e)-\Phi(q,X^e)\Big|\Big]\mathbb{E}^{e}\Big[\Big|\Phi(p,X^{e}) + \Phi(q,X^{e}) -2Y\Big|\Big]\\
         &  \leq \sum_{e}\pi^{e}\mathbb{E}^{e}\Big[\Big|\Phi(p,X^e)-\Phi(q,X^e)\Big|\Big]2(M+K)  \\
         & \leq \|p-q\|2(M+K)Q
    \end{split}
\end{equation}
Therefore, $R$ is Lipschitz with a constant $C\leq 2(M+K)Q$

$R^{'}$ is Lipschitz:

\begin{equation}
    \begin{split}
        |R^{'}(\Phi(p, \cdot)) - R^{'}(\Phi(q, \cdot))|  &= \Big|\sum_{e}\pi^{e}\Big(\mathbb{E}^{e}\Big[\Phi(p,X^{e})\big( \Phi(p,X^{e}) -Y^{e}\big)\Big]^2 -\mathbb{E}\Big[\Phi(q,X^{e})\big( \Phi(q,X^{e}) -Y^{e}\big)\Big]^2\Big)\Big| \\
        & \leq \sum_{e}\pi^{e}\Big|\mathbb{E}^{e}\big[\Phi(p,X^{e})( \Phi(p,X^{e}) -Y^{e})]^2 -\mathbb{E}[\Phi(q,X^{e})( \Phi(q,X^{e})-Y^{e})\big]^2\Big|
    \end{split} 
    \label{proof: lemma_lipsch1_genn}
\end{equation}

We bound each term in the summation in the \eqref{proof: lemma_lipsch1} above 

\begin{equation}
    \begin{split}
                & \Big| \Big(\mathbb{E}^{e}\Big[\Phi(p,X^{e})\big( \Phi(p,X^{e}) -Y^{e}\big)\Big] -\mathbb{E}^{e}[\Phi(q,X^{e})\big( \Phi(q,X^{e}) -Y^{e}\big)\Big]\Big)\Big(\mathbb{E}^{e}\Big[\Phi(p,X^{e})\big( \Phi(p,X^{e}) -Y^{e}\big)\Big] + \\ &
               \;\;\;\;\; \mathbb{E}^{e}\Big[\Phi(q,X^{e})\big( \Phi(q,X^{e}) -Y^{e}\big)\Big]\Big) \Big| \\ 
        &\leq \Big| \Big(\mathbb{E}^{e}\Big[\Phi(p,X^{e})\big( \Phi(p,X^{e}) -Y^{e}\big)\Big] -\mathbb{E}^{e}\Big[\Phi(q,X^{e})\big( \Phi(q,X^{e})-Y^{e}\big)\Big]\Big)\Big| 2M(M+K) \\ 
        & \leq \Big| \Big(\mathbb{E}^{e}\Big[\Phi(p,X^{e})\big(\Phi(p,X^{e}) -Y^{e}\big)\Big] -\mathbb{E}^{e}\Big[\Phi(q,X^{e})\big(\Phi(p,X^{e})-Y^{e}\big)\Big] + \mathbb{E}^{e}\Big[\Phi(q,X^{e})\big( \Phi(p,X^{e}) -Y^{e}\big)\Big] - \\  &\;\;\;\;\mathbb{E}^{e}\Big[\Phi(q,X^{e})\big( \Phi(q,X^{e}) -Y^{e}\big)\Big]\Big)\Big| 2M(M+K) \\ 
        & = \Big| \Big(\mathbb{E}^{e}\Big[(\Phi(p,X^{e})-\Phi(q,X^{e})\big( \Phi(p,X^{e}) -Y^{e}\big)\Big]\Big) + \mathbb{E}^{e}\Big[\Phi(q,X^{e})\big( \Phi(p,X^{e}) -\Phi(q,X^{e})\big)\Big] \Big| 2M(M+K) \\
        & \leq \Big| \Big(\mathbb{E}^{e}\Big[\big(\Phi(p,X^{e})-\Phi(q,X^{e})\big)\big( \Phi(p,X^{e}) -Y^{e}\big)\Big]\Big) \Big|2M(M+K) +\\ &\;\;\;\; \Big|\mathbb{E}^{e}\Big[\Phi(q,X^{e})\big( \Phi(p,X^{e})-\Phi(q,X^{e})\big)\Big] \Big| 2M(M+K) 
    \end{split}
    \label{proof: lemma_lipsch2_genn}
\end{equation}

We bound each term in the last line in \eqref{proof: lemma_lipsch2_genn}  
\begin{equation}
    \begin{split}
   & \Big| \big(\mathbb{E}^{e}\big[(\Phi(p,X^{e})-\Phi(q,X^{e}))( \Phi(p,X^{e}) -Y^{e})\big]\big) \Big| \\ 
   & \leq \mathbb{E}^{e}\big[\big|(\Phi(p,X^{e})-\Phi(q,X^{e}))\big| \big| ( \Phi(p,X^{e}) -Y^{e})\big|\big] \\
   & \leq  (M+K)Q\|p - q\|
    \end{split}
      \label{proof: lemma_lipsch3_gen}
\end{equation}

\begin{equation}
    \begin{split}
   & \big|\mathbb{E}\big[\Phi(q,X^{e})( \Phi(p,X^{e}) -\Phi(q,X^{e})\big] \big| \\ 
   & \leq \mathbb{E}\big[\big|\Phi(q,X^{e})||( \Phi(p,X^{e}) -\Phi(q,X^{e})\big|] \\
   & \leq  MQ\|p - q\|
    \end{split}
    \label{proof: lemma_lipsch4_gen}
\end{equation}

Substituting \eqref{proof: lemma_lipsch3_gen} and \eqref{proof: lemma_lipsch4_gen} in \eqref{proof: lemma_lipsch2_genn} to get
\begin{equation}
|R^{'}(\Phi(p, \cdot)) - R^{'}(\Phi(q,\cdot))| \leq 2M(M+K) (2M+K)Q \|p - q\|
\end{equation}
Therefore, $R^{'}$ is Lipschitz with a constant $C^{'} \leq 2M(M+K) (2M+K)Q$.

\textbf{C. EIRM: Sample complexity with no distributional assumptions}

In this section, we discuss the extension of Proposition \ref{prop2: scomp_irm} to the infinite hypothesis class case.
Consider the problem in \eqref{eqn: EIRM_cons_swap} and replace the $\epsilon$ with $\epsilon + \kappa$.  Define distance between the two sets $\mathcal{S}^{\mathsf{IV}}(\epsilon)$ and its approximation $ \mathcal{S}^{\mathsf{IV}}(\epsilon+\kappa)$  as follows

\begin{equation}
    \mathsf{dis}(\kappa)  = \max_{g \in \mathcal{S}^{\mathsf{IV}}(\epsilon+\kappa)} \min_{h \in \mathcal{S}^{\mathsf{IV}}(\epsilon)} \mathsf{d}(g,h)
\end{equation}

where $\mathsf{d}(g,h)$ is some metric that measures the distance between functions $g$ and $h$. Observe that if $a\leq b$, $\mathsf{dis}(a) \leq \mathsf{dis}(b)$. 

\begin{assumption} \label{assm: dis}
$ \lim_{k\rightarrow 0} \mathsf{dis}(\kappa) =0 $
\end{assumption}
Define $$D^{*}= \max\Big\{\frac{32L^2}{\nu^2} \Big[k\log\Big(\frac{16C\sqrt{Pk}}{\nu}\Big) + \log\Big(\frac{2}{\delta}\Big)\Big], \frac{8L'^4}{\kappa^2} \Big[k\log\Big(\frac{8C'\sqrt{Pk}}{\kappa}\Big) + \log\Big(\frac{2}{\delta}\Big)\Big] \Big\} $$
\begin{proposition} For every $\nu>0$ and $\delta \in (0,1)$, if Assumption \ref{assm15: parametrized_model}, \ref{assm: dis} hold, then $\exists\; \kappa>0$ such that if the  number of samples $|D|$ is greater than  $D^{*}$,   then with a probability at least $1-\delta$,  every solution $\hat{\Phi}$ to EIRM (replace $\epsilon$ with $\epsilon + \kappa$ in\eqref{eqn: EIRM_cons_swap})  in $\mathcal{S}^{\mathsf{IV}}(\epsilon+2\kappa)$ and $|R(\hat{\Phi}) - R(\Phi^{*})| \leq \nu$, where $\Phi^{*}$ is a solution of IRM in \eqref{eqn: IRM_cons_swap}.
\end{proposition}

\begin{proof}
We divide the proof in two parts. 

Define an event $A$: $\{\tilde{D}: \forall p \in \mathcal{P}, |R^{'}(\Phi(p,\cdot))- \hat{R}^{'}(\Phi(p,\cdot))| \leq \kappa\}$. 

In the first half, we will show that if event $A$ occurs, then $\mathcal{S}^{\mathsf{IV}}(\epsilon) \subseteq \hat{\mathcal{S}}^{\mathsf{IV}}(\epsilon+\kappa) \subseteq \mathcal{S}^{\mathsf{IV}}(\epsilon+2\kappa) $ and then bound the probability of $A$ not occuring.

\begin{equation}
    \begin{split}
         & R^{'}(\Phi(p,\cdot)) \leq \epsilon \implies  R^{'}(\Phi(p,\cdot))- \hat{R}^{'}(\Phi(p,\cdot)) + \hat{R}^{'}(\Phi(p,\cdot)) \leq \epsilon \implies \\& \hat{R}^{'}(\Phi(p,\cdot)) \leq \epsilon + |R^{'}(\Phi(p,\cdot)) - \hat{R}^{'}(\Phi(p,\cdot))| \implies  
        \hat{R}^{'}(\Phi(p,\cdot)) \leq \epsilon +\kappa \\ 
    \end{split}
\end{equation}

Therefore, $\mathcal{S}^{\mathsf{IV}}(\epsilon) \subseteq \hat{\mathcal{S}}^{\mathsf{IV}}(\epsilon+\kappa)$

\begin{equation}
    \begin{split}
        & \hat{R}^{'}(\Phi(p,\cdot)) \leq \epsilon+\kappa \implies  \hat{R}^{'}(\Phi(p,\cdot))- R^{'}(\Phi(p,\cdot)) + R^{'}(\Phi(p,\cdot)) \leq \epsilon +\kappa \implies \\ &R^{'}(\Phi(p,\cdot)) \leq \epsilon + \kappa+ |R^{'}(\Phi(p,\cdot)) - \hat{R}^{'}(\Phi(p,\cdot))| \implies 
        R^{'}(\Phi(p,\cdot)) \leq \epsilon +2\kappa \\ 
    \end{split}
\end{equation}
Therefore, $\hat{\mathcal{S}}^{\mathsf{IV}}(\epsilon+\kappa) \subseteq \mathcal{S}^{\mathsf{IV}}(\epsilon+2\kappa)$

We bound the probability of  event $A$ not occurring.
Using the covering number (from Lemma \ref{lemma:covering_number}) we  construct a minimum cover of size $b= N_{\eta}(\mathcal{P})$ with points $\mathcal{C}_1 =  \{p_j\}_{j=1}^{b}$.

Compute the probability of failure at one point $p_j$ in the cover 
\begin{equation}
    \mathbb{P}\Big[\Big\{\tilde{D}: |R^{'}(\Phi(p_j, \cdot)) - \hat{R}^{'}(\Phi(p_j,\cdot)) |>\frac{\kappa}{2}\Big\}\Big] < 2e^{-\frac{\kappa^2 |D|}{8L'^{4}}}
\end{equation}

We use union bound to bound the probability of the failure over the entire cover $\mathcal{C}_{1}$ as 
\begin{equation}
    \mathbb{P}\Big[\Big\{\tilde{D}: \max_{p_j\in \mathcal{C}_1}|R^{'}(\Phi(p_j,\cdot)) - \hat{R}^{'}(\Phi(p_j,\cdot)) |>\frac{\kappa}{2} \Big\}\Big] < N_{\eta}(\mathcal{P})2e^{-\frac{\kappa^2 |D|}{8L'^4}}
     \label{proof: inf_hyp_0_eirm_disag}
\end{equation}

Now consider any $p \in \mathcal{P}$ and suppose $p_j$ is nearest point to  it in the cover.

\begin{equation} 
\begin{split}
 &|R^{'}(\Phi(p,\cdot)) - \hat{R}^{'}(\Phi(p,\cdot)) | = \\& |R^{'}(\Phi(p,\cdot)) - R^{'}(\Phi(p_j, \cdot)) +R^{'}(\Phi(p_j,\cdot)) -  \hat{R}^{'}(\Phi(p_j,\cdot) )+ \hat{R}^{'}(\Phi(p_j,\cdot) )- \hat{R}^{'}(\Phi(p,\cdot)) | \\ &\leq   |R^{'}(\Phi(p_j,\cdot)) - \hat{R}^{'}(\Phi(p_j,\cdot)) | + 2\eta C^{'} 
 \end{split}
\end{equation}
In the above simplficiation, we exploit the Lipschitz continuity of $R^{'}$.

Therefore
\begin{equation} 
\begin{split}
 \forall p \in \mathcal{P}\; &|R^{'}(\Phi(p,\cdot)) - \hat{R}^{'}(\Phi(p,\cdot)) |  \leq  \max_{p_j\in \mathcal{C}_1} |R^{'}(\Phi(p_j,\cdot)) - \hat{R}^{'}(\Phi(p_j,\cdot)) | + 2\eta C^{'} \\ 
  &\max_{p\in \mathcal{P}}|R^{'}(\Phi(p,\cdot)) - \hat{R}^{'}(\Phi(p,\cdot)) |  \leq  \max_{p_j\in \mathcal{C}_1} |R^{'}(\Phi(p_j,\cdot)) - \hat{R}^{'}(\Phi(p_j,\cdot)) | + 2\eta C^{'}
 \end{split}
 \label{proof: inf_hyp_1_eirm_disag}
\end{equation}
 Set $\eta = \frac{\kappa}{4C^{'}}$ in \eqref{proof: inf_hyp_1_eirm_disag}  and from \eqref{proof: inf_hyp_0_eirm_disag}  with probability at least $1-N_{\eta}(\mathcal{P})2e^{-\frac{\kappa^2 |D|}{8L'^4}}$ 
\begin{equation}
     \max_{p\in \mathcal{P}} |R^{'}(\Phi(p,\cdot)) - \hat{R}^{'}(\Phi(p,\cdot))  | \leq \kappa\; \text{(since $\max_{p_j\in \mathcal{C}_1}|R^{'}(\Phi(p_j,\cdot)) - \hat{R}^{'}(\Phi(p_j,\cdot))  | \leq \kappa/2$ )}
\end{equation}

We bound $N_{\eta}(\mathcal{P})2e^{-\frac{\kappa^2 |D|}{8L'^4}}\leq \delta$ and solve for bound on $|D|$ to get 

\begin{equation} 
\begin{split}
  &  |D| \geq \frac{8L'^4}{\kappa^2} \log\Big(\frac{2N_{\eta}(\mathcal{P})}{\delta}\Big)\\ 
 &     |D| \geq \frac{8L'^4}{\kappa^2} \Big[k\log\Big(\frac{8C'\sqrt{Pk}}{\kappa}\Big) + \log\Big(\frac{2}{\delta}\Big)\Big]
    \end{split}
    \label{proof: inf_hyp_2_eirm_disag}
\end{equation}

Therefore, if condition in \eqref{proof: inf_hyp_2_eirm_disag} holds, then event $A$ occurs. If event $A$ occurs, then $\mathcal{S}^{\mathsf{IV}}(\epsilon) \subseteq \hat{\mathcal{S}}^{\mathsf{IV}}(\epsilon+\kappa) \subseteq \mathcal{S}^{\mathsf{IV}}(\epsilon+2\kappa) $.

$\Phi(p^{*},\cdot)$ is a solution to IRM \eqref{eqn: IRM_cons_swap} and it  satisfies $\forall \Phi(p,\cdot) \in \mathcal{S}^{\mathsf{IV}}(\epsilon)$ $ R(\Phi(p^{*})) \leq R(\Phi(p,\cdot)$. 

Define an event $B$: $\{D: \forall p\in \mathcal{P}, |R(\Phi(p,\cdot))- \hat{R}(\Phi(p,\cdot))| \leq \frac{\nu}{2}\}$. 

If event $B$ occurs, then for a solution $\Phi(\hat{p},\cdot)$ of  \eqref{eqn: EIRM_cons_swap} (where $\epsilon$ is replaced with $\epsilon + \kappa$) satisfies

\begin{equation}
\begin{split}
& R(\Phi(\hat{p},\cdot)) - \frac{\nu}{2} \leq \hat{R}(\Phi(\hat{p},\cdot) \leq \hat{R}(\Phi(p^{*}, \cdot)) \leq R(\Phi(p^{*}, \cdot)) + \frac{\nu}{2} \\ 
& R(\Phi(\hat{p},\cdot)) \leq R(\Phi(p^{*}, \cdot)) + \nu
\end{split}
\label{proof: eirm_disag_eq1}
\end{equation} 
% We denote $\Phi(p^{*},\cdot)$ as $\Phi^{*}$ and $\Phi(\hat{p},\cdot)$ as $\hat{\Phi}$.

Using the covering number (from Lemma \ref{lemma:covering_number}) we  construct a minimum cover of size $b= N_{\eta}(\mathcal{P})$ with points $\mathcal{C}_1 =  \{p_j\}_{j=1}^{b}$.

Let us bound the probability of failure at one point $p_j$ in the cover 
\begin{equation}
    \mathbb{P}\Big[|R(\Phi(p_j,\cdot)) - \hat{R}(\Phi(p_j,\cdot)) |>\frac{\nu}{4}\Big] < 2e^{-\frac{\nu^2 |D|}{32L^{2}}}
\end{equation}

We use union bound to bound the probability defined as 
\begin{equation}
    \mathbb{P}\Big[\Big\{D: \max_{p_j\in \mathcal{C}_1}|R(\Phi(p_j,\cdot)) - \hat{R}(\Phi(p_j,\cdot)) |>\frac{\nu}{4}\Big\}\Big] < N_{\eta}(\mathcal{P})2e^{-\frac{\nu^2 |D|}{32L^2}}
     \label{proof: inf_hyp_3_eirm_disag}
\end{equation}

Now consider any $p \in \mathcal{P}$ and suppose $p_{j}$ is nearest point to it in the cover.

\begin{equation} 
\begin{split}
 &|R(\Phi(p,\cdot)) - \hat{R}(\Phi(p,\cdot)) | = \\ & |R(\Phi(p,\cdot)) - R(\Phi(p_j,\cdot)) +R(\Phi(p_j,\cdot)) -  \hat{R}(\Phi(p_j,\cdot) )+ \hat{R}(\Phi(p_j,\cdot) )- \hat{R}(\Phi(p,\cdot)) | \\ &\leq   |R(\Phi_j) - \hat{R}(\Phi_j) | + 2\eta C 
 \end{split}
\end{equation}

Therefore, for each $p \in \mathcal{P}$
\begin{equation} 
\begin{split}
 &|R(\Phi(p,\cdot)) - \hat{R}(\Phi(p_j,\cdot)) |  \leq  \max_{p_j\in \mathcal{C}_1} |R(\Phi(p_j,\cdot)) - \hat{R}(\Phi(p_j,\cdot)) | + 2\eta C \\ 
  &\max_{p\in \mathcal{P}}|R(\Phi(p,\cdot)) - \hat{R}(\Phi(p,\cdot)) |  \leq  \max_{p_j\in \mathcal{C}_1} |R(\Phi(p_j,\cdot)) - \hat{R}(\Phi(p_j,\cdot)) | + 2\eta C
 \end{split}
 \label{proof: inf_hyp_4_eirm_disag}
\end{equation}
 Set $\eta = \frac{\nu}{8C}$ in \eqref{proof: inf_hyp_3_eirm_disag}  and from \eqref{proof: inf_hyp_4_eirm_disag}  with probability at least $1-N_{\eta}(\mathcal{P})2e^{-\frac{\nu^2 |D|}{32L'^4}}$ 
\begin{equation}
     \max_{p\in \mathcal{P}} |R(\Phi(p,\cdot)) - \hat{R}^{'}(\Phi(p,\cdot))  | \leq \nu/2\; \text{(since $\max_{p_j\in \mathcal{C}_1}|R^{'}(\Phi(p_j,\cdot)) - \hat{R}^{'}(\Phi(p_j,\cdot))  | \leq \nu/4$ )}
\end{equation}

We  bound $N_{\eta}(\mathcal{P})2e^{-\frac{\nu^2 |D|}{32L^2}}\leq \delta$ and solve for bound on $|D|$ to get 

\begin{equation} 
\begin{split}
  &  |D| \geq \frac{32L^2}{\nu^2} \log\Big(\frac{2N_{\eta}(\Phi)}{\delta}\Big)\\ 
 &     |D| \geq \frac{32L^2}{\nu^2} \Big[k\log\Big(\frac{16C\sqrt{Pk}}{\nu}\Big) + \log(\frac{2}{\delta})\Big]
    \end{split}
    \label{proof: inf_hyp_5_eirm_disag}
\end{equation}

Therefore, if condition in \eqref{proof: inf_hyp_5_eirm_disag} holds, then with probability at least $1-\frac{\delta}{2}$ event $A$ occurs.

Also, if 
\begin{equation}
  |D| \geq \max\Big\{\frac{32L^2}{\nu^2} \Big[k\log\Big(\frac{16C\sqrt{Pk}}{\nu}\Big) + \log\Big(\frac{2}{\delta}\Big)\Big], \frac{8L'^4}{\kappa^2} \Big[k\log\Big(\frac{8C'\sqrt{Pk}}{\kappa}\Big) + \log\Big(\frac{2}{\delta}\Big)\Big] \Big\}
\end{equation}
then the event $A\cap B$ occurs with at least $1-\delta$ probability.

Project $\Phi(\hat{p},\cdot) \in \hat{\mathcal{S}}^{\mathsf{IV}}(\epsilon + \kappa)$ on $\mathcal{S}^{\mathsf{IV}}(\epsilon)$, i.e., find the closest function in terms of the metric $\mathsf{dis}$, to obtain $\Phi(\tilde{\hat{p}},\cdot)$. If event $A$ occurs, then $\hat{p} \in \mathcal{S}^{\mathsf{IV}}(\epsilon + 2\kappa)$. The distance $\|\hat{p}-\tilde{\hat{p}}\| \leq \mathsf{dis}(2\kappa)$. 
\begin{equation}
    \begin{split}
        |R(\Phi(\hat{p},\cdot)) - R(\Phi(\tilde{\hat{p}},\cdot))| \leq C\|\hat{p}-\tilde{\hat{p}}\| \leq C\mathsf{dis}(2\kappa) 
    \end{split}
\end{equation}
We choose $\kappa_0$ such that $\kappa<\kappa_0$ (use Assumption \ref{assm: dis}) such that 
$C\mathsf{dis}(2\kappa)\leq \nu$. 

\begin{equation}
    \begin{split}
       R(\Phi(p^{*}, \cdot) \leq R(\Phi(\tilde{\hat{p}},\cdot)) \leq  R(\Phi(\hat{p},\cdot))  + \nu
    \end{split}
    \label{proof: eirm_disag_eq2}
\end{equation}
Therefore, by combining \eqref{proof: eirm_disag_eq1} amd \eqref{proof: eirm_disag_eq2}, we can conclude that if event $A\cap B$ occurs, then $\Phi(\hat{p},\cdot) \in \mathcal{S}^{\mathsf{IV}}(\epsilon + 2\kappa)$ and  $|   R(\Phi(p^{*},\cdot)) - R(\Phi(\hat{p},\cdot))|\leq \nu$. From the conditions on $|D|$, we know that $A\cap B$ occurs with probability $1-\delta$. We substitute $\Phi(p^{*},\cdot)$ as $\Phi^{*}$ and $\Phi(\hat{p},\cdot)$ as $\hat{\Phi}$ and this completes the proof.  
\end{proof}

\textbf{D. OOD Performance: Covariate shift case}

In this section, we discuss the extension of Proposition \ref{prop4: emr_irm_real} to the infinite hypothesis class case. 

Define $D_1^{*} = \Big\{ \frac{32L^2}{\nu^2} \Big[k\log\Big(\frac{32C\sqrt{Pk}}{\nu}\Big) + \log\Big(\frac{4}{\delta}\Big)\Big], \frac{16L'^{4}}{\epsilon^2}\log\Big(\frac{4}{\delta}\Big) \Big\}$.

Define $D_2^{*} = \frac{32L^2}{\nu^2} \Big[k\log\Big(\frac{16C\sqrt{Pk}}{\nu}\Big) + \log\Big(\frac{2}{\delta}\Big)\Big]$.

\begin{proposition}
If  Assumptions  \ref{assm4: invariant_cexp}, \ref{assm15: parametrized_model}  hold, $m\in \mathcal{H}_{\Phi}$ and if the  number of samples $|D|$ is greater than $D_1^{*}$, then  with a probability at least $1-\delta$, every solution $\hat{\Phi}$ to EIRM satisfies $R(m) \leq R(\hat{\Phi}) \leq R(m)+\nu$.

If  Assumptions  \ref{assm4: invariant_cexp}, \ref{assm15: parametrized_model} hold, $m\in \mathcal{H}_{\Phi}$ and  if the  number of samples $|D|$ is greater than $D_{2}^{*}$,  then  with a probability at least $1-\delta$  every solution $\Phi^{\dagger}$ of ERM satisfies  $R(m) \leq R(\Phi^{\dagger}) \leq R(m)+\nu$.

\end{proposition}

\begin{proof}
We begin with the first part. 
Following the proof of Proposition \ref{prop5: irm_cfd_chd}, our goal is to compute the probability of event $A$: $\{\forall p \in \mathcal{P}, |\hat{R}(\Phi(p,\cdot))-R(\Phi(p,\cdot))|\leq \frac{\nu}{2}\}$. 
Using the covering number (from Lemma \ref{lemma:covering_number}) we construct a minimum cover of size $b= N_{\eta}(\mathcal{P})$ with points $\mathcal{C} =  \{p_j\}_{j=1}^{b}$.

Compute the probability of failure at one point $p_j$ in the cover 
\begin{equation}
    \mathbb{P}\Big[\Big\{D: |R(\Phi(p_j,\cdot)) - \hat{R}(\Phi(p_j,\cdot)) |>\frac{\nu}{4}\Big\}\Big] < 2e^{-\frac{\nu^2 |D|}{32L^{2}}}
\end{equation}

We use union bound to bound the probability of failure over the cover $\mathcal{C}$
\begin{equation}
    \mathbb{P}\Big[\Big\{D: \max_{j}|R(\Phi(p_j,\cdot)) - \hat{R}(\Phi(p_j,\cdot)) |>\frac{\nu}{4}\Big\}\Big] < N_{\eta}(\mathcal{P})2e^{-\frac{\nu^2 |D|}{32L^2}}
     \label{proof: inf_hyp_0_erm}
\end{equation}

Now consider any $p \in \mathcal{P}$ and suppose $p_{j}$ is nearest point to it in the cover.

\begin{equation} 
\begin{split}
 &|R(\Phi(p,\cdot)) - \hat{R}(\Phi(p,\cdot)) | = \\&  |R(\Phi(p,\cdot)) - R(\Phi(p_j,\cdot)) +R(\Phi(p_j,\cdot)) -  \hat{R}(\Phi(p_j,\cdot) )+ \hat{R}(\Phi(p_j,\cdot) )- \hat{R}(\Phi(p,\cdot)) | \\ &\leq   |R(\Phi(p_j,\cdot) - \hat{R}(\Phi(p_j,\cdot)) | + 2\eta C 
 \end{split}
\end{equation}
In the above simplification, we used the Lipschitz continuity of $R$.
Therefore
\begin{equation} 
\begin{split}
\forall p \in \mathcal{P},\; &|R(\Phi(p,\cdot)) - \hat{R}(\Phi(p,\cdot)) |  \leq  \max_{p_j\in \mathcal{C}} |R(\Phi(p_j,\cdot)) - \hat{R}(\Phi(p_j,\cdot)) | + 2\eta C \\ 
  &\max_{p \in \mathcal{P}}|R(\Phi(p,\cdot)) - \hat{R}(\Phi(p,\cdot)) |  \leq  \max_{p_j\in \mathcal{C}} |R(\Phi(p_j,\cdot)) - \hat{R}(\Phi(p_j,\cdot)) | + 2\eta C
 \end{split}
 \label{proof: inf_hyp_1_erm}
\end{equation}
 Set $\eta = \frac{\nu}{8C}$ in \eqref{proof: inf_hyp_1_erm}  and from \eqref{proof: inf_hyp_0_erm}  with probability at least $1-N_{\eta}(\mathcal{P})2e^{-\frac{\nu^2 n}{32L'^4}}$ 
\begin{equation}
     \max_{p\in \mathcal{P}} |R^{'}(\Phi(p,\cdot)) - \hat{R}^{'}(\Phi(p,\cdot))  | \leq \nu/2\; \text{(since $\max_{p_j\in \mathcal{C}}|R^{'}(\Phi(p_j,\cdot)) - \hat{R}^{'}(\Phi(p_j,\cdot))  | \leq \nu/4$ )}
\end{equation}

We  bound $N_{\eta}(\mathcal{P})2e^{-\frac{\nu^2 |D|}{32L^2}}\leq \delta$ and solve for bound on $|D|$ to get 

\begin{equation} 
\begin{split}
  &  |D| \geq \frac{8L^2}{\frac{\nu^2}{4}} \log\Big(\frac{N_{\eta}(\mathcal{P})}{\delta}\Big)\; \text{(Use Lemma \ref{lemma:covering_number})}\\ 
 &     |D| \geq \frac{32L^2}{\nu^2} \Big[k\log\Big(\frac{16C\sqrt{Pk}}{\nu}\Big) + \log\Big(\frac{2}{\delta}\Big)\Big]
    \end{split}
    \label{proof: inf_hyp_2_erm}
\end{equation}

Therefore, if condition in \eqref{proof: inf_hyp_2_erm} holds, then event $A$ occurs. Observe that the optimal solution $p^{*}$ for expected risk minimization $p^{*}\in \arg\min_{p\in \mathcal{P}} R(\Phi(p,\cdot))$ satisfies $\Phi(p^{*}, \cdot) = m(\cdot)$ (From Lemma \ref{lemma4: unique_ls} and $m\in \mathcal{H}_{\Phi}$).
If event $A$ occurs, then from same argument used in \eqref{eqn : iota-rep-opt-cond}, a solution $p^{+}\in \mathcal{P}$ of ERM satisfies 
$R(\Phi(p^{*},\cdot)) \leq R(\Phi(p^{+},\cdot)) \leq R(\Phi(p^{*},\cdot)) + \nu$ is true.

We now move to the second part. From the first part of the proof, we  conclude that when 
\begin{equation}
|D| \geq \frac{32L^2}{\nu^2} \Big[k\log\Big(\frac{16C\sqrt{Pk}}{\nu}\Big) + \log(\frac{4}{\delta})\Big]
\label{prop4_proof: eqn1_eirm}
\end{equation}
 with a probability at least $1-\frac{\delta}{2}$ event
$A$ occurs. 

Define an event $B$: $\tilde{D}$ is such that $|\hat{R}^{'}(m) - R^{'}(m)| \leq \frac{\epsilon}{2}$.   Since $R^{'}(m)=0$,  
$|\hat{R}^{'}(m) - R^{'}(m)| \leq \frac{\epsilon}{2} \implies |\hat{R}^{'}(m) | \leq \frac{\epsilon}{2} \implies \hat{R}^{'}(m)  \leq \frac{\epsilon}{2} $.
 Therefore, $m \in \hat{\mathcal{S}}^{\mathsf{IV}}(\epsilon)$.

 We write $\mathbb{P}(B) = \mathbb{P}(\{\tilde{D}: |\hat{R}^{'}(m) - R^{'}(m)| \leq \frac{\epsilon}{2} \}) = 1-\mathbb{P}(\{\tilde{D}: |\hat{R}^{'}(m) - R^{'}(m)| > \frac{\epsilon}{2} \})  $.  The  gradient of loss function is bounded $|\frac{\partial \ell(\Phi(\cdot), \cdot)}{\partial w}|_{w=1.0} | \leq L^{'}$. From Hoeffding's inequality in Lemma \ref{lemma1: hoeffding} it follows that  
$$\mathbb{P}(|\tilde{D}: \hat{R}^{'}(h) - R^{'}(h)| > \frac{\epsilon}{2}  ) \leq 2 \exp\Big(-\frac{|D|\epsilon^2}{16L'^4}\Big)$$

$$2 \exp\Big(-\frac{|D|\epsilon^2}{16L'^4}\Big)\leq \frac{\delta}{2} $$

\begin{equation}
|D| \geq  \frac{16L'^4}{\epsilon^2}\log\Big(\frac{4}{\delta}\Big) 
\label{prop4_proof: eqn2_eirm}
\end{equation}

Combining the two conditions \eqref{prop4_proof: eqn1_eirm} and \eqref{prop4_proof: eqn2_eirm}, 
$$|D| \geq  \max \Big\{ \frac{32L^2}{\nu^2} \Big[k\log\Big(\frac{32C\sqrt{Pk}}{\nu}\Big) + \log\Big(\frac{4}{\delta}\Big)\Big], \frac{16L'^{4}}{\epsilon^2}\log\Big(\frac{4}{\delta}\Big) \Big\}$$
This ensures $P(A\cap B) \geq 1-\delta$.  If event $A\cap B$ occurs, then we follow the same justification as in the proof of Proposition \ref{prop4: emr_irm_real} to claim that a solution $\hat{p}\in \mathcal{P}$ to EIRM \eqref{eqn: EIRM_cons_swap} satisfies  $R(\Phi(p^{*},\cdot)) \leq R(\Phi(\hat{p},\cdot)) \leq R(\Phi(p^{*},\cdot)) + \nu$ 

This completes the proof. 

\end{proof}

\subsubsection{Extensions to binary classification (cross-entropy)}
In the main body of the manuscript, we focused on regression (square-loss). In this section, we discuss the results that can be extended to binary classification (cross-entropy) loss. We will not go in the order in which the results were introduced in the manuscript but in an order that makes for easier exposition for the classification case. 

We begin by showing how to extend Proposition \ref{prop4: emr_irm_real} to binary classification (cross-entropy).
Recall that the entropy of a distribution $\mathbb{P}_X$ is $H(\mathbb{P}) = -\mathbb{E}_{\mathbb{P}}\big[\log(d\mathbb{P}) \big]$. Recall that the cross entropy of $\mathbb{Q}$ relative to $\mathbb{P}$ is  $H(\mathbb{P},\mathbb{Q}) = -\mathbb{E}_{\mathbb{P}}\big[\log(d\mathbb{Q}) \big] = H(P) + \mathsf{KL}(d\mathbb{P}\|d\mathbb{Q})$. 
The cross entropy loss $\ell$ for binary classification when using a predictor $f: \mathcal{X}\rightarrow [0,1]$ ($f(X^{e})$ is the probability of label $1$ conditional on $X^{e}$) is given as $\ell\big(f(X^{e}),Y^{e}\big) = Y^{e}\log\big(f(X^{e})\big) + (1-Y^{e})\log\big(1-f(X^{e})\big)$. For the discussion below $\mathbb{Q}(Y^{e}|X^{e})$ is defined in terms of $f$ as follows $\mathbb{Q}(Y^{e}=1|X^{e})= f(X^{e})$ $\big(\mathbb{Q}(Y^{e}=0|X^{e})= 1-f(X^{e})\big)$. 
\begin{equation}
\begin{split}
    R^{e}(f) & = \mathbb{E}^{e}\big[ \ell(Y^{e}, f\big(X^{e}\big)) \big]  \\ 
    & = \mathbb{E}^{e}\Big[Y^{e}\log\big(f(X^{e}\big)\big) + (1-Y^{e})\log\big(1-f\big(X^{e}\big)\big)\Big] \\ 
    & = \mathbb{E}^{e}\Big[\mathbb{E}^{e}\big[Y^{e}|X^{e}\big]\log\big(f\big(X^{e}\big)\big) + \big(1-\mathbb{E}\big[Y^{e}|X^{e}\big]\big)\log\big(1-f\big(X^{e}\big)\big)\Big] \\
    & = \mathbb{E}^{e}\Big[ \mathbb{P}(Y^{e}=1|X^{e})\log\big(f\big(X^{e}\big)\big) + \mathbb{P}(Y^{e}=0|X^{e})\log\big(1-f\big(X^{e}\big)\big)\Big] \\
  &  = \mathbb{E}^{e}\Big[H\big(\mathbb{P}(Y^{e}|X^{e}), \mathbb{Q}(Y^{e}|X^{e})\big)\Big] \\
  & = \mathbb{E}^{e}\big[H\big(\mathbb{P}(Y^{e}|X^{e})\big) + \mathsf{KL}\big(\mathbb{P}(Y^{e}|X^{e}) \| \mathbb{Q}(Y^{e}|X^{e}\big)\big] \\ 
 & = \mathbb{E}^{e}\Big[H\big(\mathbb{P}(Y^{e}|X^{e})\big)\Big] + \mathbb{E}^{e}\Big[\mathsf{KL}\big(\mathbb{P}(Y^{e}|X^{e}) \| \mathbb{Q}(Y^{e}|X^{e}\big)\Big]
\end{split}
\label{cross_entropy_decomposition}
\end{equation}

From the above it is clear that $\mathbb{Q}(Y^{e}|X^{e}\big) = \mathbb{P}(Y^{e}|X^{e}\big)$ minimizes the risk in an individual environment.

\begin{assumption} \textbf{Invariance w.r.t all the features.}
\label{assm4: invariant_cexp_binary}
For all $e,o\in \mathcal{E}_{all}$ and for all $x\in \mathcal{X}$,
$\mathbb{E}[Y^{e}|X^{e}=x]  =\mathbb{E}[Y^{o}|X^{o}=x]$.  $X^{e}\sim \mathbb{P}^{e}_{X^{e}}$ and  $\forall e\in \mathcal{E}_{all}$ support of $\mathbb{P}^{e}_{X^{e}}$ is equal to  $\mathcal{X}$.
\end{assumption}
Observe that in the binary-classification setting the above assumption amounts to equating the conditional probabilities $\mathbb{P}(Y^{e}|X^{e})$ and $\mathbb{P}(Y^{o}|X^{o})$.

Recall that map $m$ (from \eqref{m_inv_model}) simplifies to  $\forall x \in \mathcal{X}$
\begin{equation} 
m(x) = \mathbb{E}^{e}[Y^{e}|X^{e}=x] = \mathbb{P}(Y^{e}=1|X^{e}=x)
\label{m_inv_model_append}
\end{equation}

If Assumption \ref{assm4: invariant_cexp_binary} holds, then from cross-entropy decomposition in \eqref{cross_entropy_decomposition} it is clear that $m$ solves the OOD problem (as it is optimal w.r.t each environment). It is also the unique minimizer. We can justify it based on the same argument presented in \eqref{eqn:ls_opt}. Suppose there was another optimizer which was different from $m$ over a set with a non-zero measure. Over such a set the the KL divergence term inside \eqref{cross_entropy_decomposition} will be greater than zero, thus making the second term in \eqref{cross_entropy_decomposition} positive thus contradicting the optimality. This shows $m$ is the unique optimizer. The rest of the arguments presented in the proof of Proposition \ref{prop4: emr_irm_real} carry over to this case.  Therefore,  Proposition \ref{prop4: emr_irm_real} extends to the cross-entropy loss. 

 Note that Proposition \ref{prop2: scomp_irm}'s proof was agnostic to loss type and only used boundedness, which holds for both cross-entropy as long as the probability output are in the strict interior of $[0,1]$ defined by $[\mathsf{p}_{\mathsf{min}},\mathsf{p}_{\mathsf{max}}]\subset [0,1]$. We could not generalize Proposition \ref{prop5: irm_cfd_chd} to cross-entropy loss and that is left as future work.

Next, we move to showing how Proposition \ref{prop1: ood} can be generalized to binary classification. 

\begin{assumption} \textbf{Existence of an invariant representation.}
\label{assm1:ood_cond_envmt_binary}
$\exists\; \Phi^{*}:\mathcal{X}\rightarrow \mathcal{Z}$ such that $\forall e,o \in \mathcal{E}_{all}$ and $\forall x \in \mathcal{X}$, $ \mathbb{E}[Y^{e}|\Phi^{*}(x)]=\mathbb{E}[Y^{o}|\Phi^{*}(x)]$.
\end{assumption}

Recall $m$ defined in \eqref{m_inv_model}, $\forall z \in \Phi^{*}(\mathcal{X})$
\begin{equation} 
m(z) = \mathbb{E}^{e}[Y^{e}|Z^{e}=z] = \mathbb{P}(Y^{e}=1|X^{e}=z)
% \label{m_inv_model_append}
\end{equation}

Define a composite predictor $w\circ \Phi^{*}$. Substituting $f=w\circ \Phi^{*}$ in \eqref{cross_entropy_decomposition} we get the following.  For the discussion below, a distribution $\mathbb{R}(Y^{e}|X^{e})$ is defined in terms of $w\circ \Phi^{*}$ as follows $\mathbb{R}(Y^{e}=1|X^{e})= w\circ \Phi^{*}(X^{e})$ ($\mathbb{R}(Y^{e}=0|X^{e})= 1-w\circ \Phi^{*}(X^{e})$). 
\begin{equation}
\begin{split}
    R^{e}(w\circ \Phi^{*})  = \mathbb{E}^{e}\Big[H\big(\mathbb{P}(Y^{e}|Z^{e})\big)\Big] + \mathbb{E}^{e}\Big[\mathsf{KL}\big(\mathbb{P}(Y^{e}|Z^{e}) \| \mathbb{R}(Y^{e}|\big(Z^{e}\big)\big)\Big]
\end{split}
\label{cross_entropy_decomposition1}
\end{equation}
If all the data is transformed by $\Phi^{*}$, then from the above decomposition \eqref{cross_entropy_decomposition1} it is clear that $\mathbb{R}(Y^{e}|Z^{e}) = \mathbb{P}(Y^{e}|Z^{e})$ is the optimal predictor for each environment. Hence, $w^{*}(Z^{e}) = \mathbb{P}(Y^{e}=1|Z^{e})$ is the best choice for $w$.  

\begin{assumption}
\label{assm2:ood_cond_envmt_binary}
\textbf{Existence of an environment where the invariant representation is sufficient.}
$\exists$ an environment $e \in \mathcal{E}_{all}$ such that $Y^{e} \perp X^{e} | Z^{e}$, where $Z^{e} = \Phi^{*}(X^{e})$.
\end{assumption}

We derive a relationship as follows for the environment $q$ satisfying Assumption \ref{assm2:ood_cond_envmt_binary}. 
% \begin{equation}
% \mathbb{P}(Y^{e},X^{e}|Z^{e}) = \mathbb{P}(Y^{e}|Z^{e})\mathbb{P}(X^{e}|Z^{e})  \\ 
% \end{equation}
% \begin{equation}
% \mathbb{P}(Y^{e},X^{e}|Z^{e}) = \mathbb{P}(Y^{e}|Z^{e},X^{e})\mathbb{P}(X^{e}|Z^{e})  \\ 
% \end{equation}
\begin{equation}
    \mathbb{P}(Y^{q}|Z^{q},X^{q}) = \mathbb{P}(Y^{q}|Z^{q}) \; \text{(follows from conditional independence in Assumption \ref{assm2:ood_cond_envmt_binary})}
    \label{cross_entropy_sec_eq2}
\end{equation}
Also, note that since $Z^{q}=\Phi^{*}(X^{q})$ we have 
\begin{equation}
     \mathbb{P}(Y^{q}|Z^{q},X^{q}) = \mathbb{P}(Y^{q}|X^{q})
         \label{cross_entropy_sec_eq3}
\end{equation}

From \eqref{cross_entropy_sec_eq2} and \eqref{cross_entropy_sec_eq3} we have 
\begin{equation}
    \mathbb{P}(Y^{q}|X^{q}) = \mathbb{P}(Y^{q}|Z^{q})
    \label{cross_entropy_sec_eq4}
\end{equation}
We use \eqref{cross_entropy_sec_eq4} in the cross entropy decomposition from \eqref{cross_entropy_decomposition}
\begin{equation}
\begin{split}
R^{q}(f) = \mathbb{E}^{q}\Big[H\big(\mathbb{P}(Y^{q}|Z^{q})\big)\Big] + \mathbb{E}^{q}\Big[\mathsf{KL}\big(\mathbb{P}(Y^{q}|Z^{q}) \| \mathbb{Q}(Y^{q}|X^{q}\big)\Big]
\end{split}
\end{equation}
Recall  $\mathbb{Q}(Y^{q}=1|X^{q})= f(X^{q})$ ($\mathbb{Q}(Y^{q}=0|X^{q})= 1-f(X^{q})$). Also, recall $w^{*}(Z^{q}) = \mathbb{P}(Y^{q}=1|Z^{q})$. From the above it is clear that $f = w^{*}\circ \Phi^{*}$ is the optimal predictor for environment $q$. 

\begin{equation}
    R^{q}(w^{*}\circ \Phi^{*}) = \mathbb{E}^{q}\Big[H\big(\mathbb{P}(Y^{q}|X^{q})\big)\Big]
\end{equation}
The expected conditional entropy for environment $e$ is defined as  $\bar{H}^{e} = \mathbb{E}^{e}\Big[H\big(\mathbb{P}(Y^{e}|Z^{e})\big)\Big]$ is the risk achieved by $w^{*}\circ \Phi^{*}$. Also, $\bar{H}^{e}$ measures the amount of noise in the environment. This is much like the variance that remains in the least squares minimization. In the next assumption, we state that the noise in all the environments is bounded above. We also assume that one of the environments which achieves the maximum noise level is environment $q$, which satisfies Assumption \ref{assm2:ood_cond_envmt_binary}. 

\begin{assumption}
\label{assm3: ood-binary}
$\forall e \in \mathcal{E}_{all}, \bar{H}^{e} \leq \bar{H}^{\mathsf{sup}},\bar{H}^{q}=\bar{H}^{\mathsf{sup}}$
\end{assumption}
Therefore, 
\begin{equation}
    R^{q}(w^{*}\circ \Phi^{*}) = \bar{H}^{\mathsf{sup}}
\end{equation}
From \eqref{cross_entropy_decomposition1} for all the environments 
\begin{equation}
    R^{e}(w^{*}\circ \Phi^{*}) = \bar{H}^{e}
\end{equation}
Observe that $\max_{e\in \mathcal{E}_{all}} R^{e}(w^{*}\circ \Phi^{*}) =\bar{H}^{\mathsf{sup}}$. 
From the above assumption it is clear that for all predictors $f:\mathcal{X}\rightarrow [0,1]$. 
\begin{equation}
\begin{split}
&    \forall f, \max_{e\in \mathcal{E}_{all}} R^{e}(f) \geq R^{q}(f) \geq \bar{H}^{\mathsf{sup}} \\ 
 &   \min_{f} \max_{e\in \mathcal{E}_{all}} R^{e}(f)  \geq \bar{H}^{\mathsf{sup}}
\end{split}
\end{equation}
 Since $\max_{e\in \mathcal{E}_{all}} R^{e}(w^{*}\circ \Phi^{*}) =\bar{H}^{\mathsf{sup}}$, we  conclude that $w^{*}\circ \Phi^{*}$ is the predictor that solves the OOD problem in \eqref{eqn1: ood}. This completes the extension of Proposition \ref{prop1: ood} to cross-entropy.

\subsubsection{On the Biasedness of ERM}

Consider the model in Assumption \ref{assm6: linear_model}.  For each environment $e\in \mathcal{E}_{tr}$, define a vector $\rho^{e} = \mathbb{E}^{e}\big[\varepsilon^{e}X^{e}\big]$. Define a matrix $\bar{\rho}$ with $\rho^{e}$ as column vectors $\bar{\rho} = [\rho^{1},\dots, \rho^{|\mathcal{E}_{tr}|}]$. Define a vector $\bar{\pi} = [\pi^{1}, \dots, \pi^{|\mathcal{E}_{tr}|}]$, where recall from Assumption \ref{assm6: linear_model} $\pi^{o}$ is probability a point comes from environment $o$.
% Observe that when each $\rho^{e}$ is zero (such as in the case when $\varepsilon^{e}\perp X^{e}$) the solution of ERM is not biased as the gradient is zero.
%  In the case when the correlation vectors are non-zero, we argue that the gradient defined above in \eqref{erm_bias_eqn2} will be non-zero thus establishing that  $\tilde{\mathcal{S}}^{\mathsf{T}}\gamma$ is not the optimal solution to ERM. 
% Define a matrix $\bar{\rho}$ with $\rho^{e}$ as column vector.
\begin{proposition}
\label{erm_bias}
If Assumption \ref{assm6: linear_model} holds, $\mathcal{H}_{\Phi}$ is a linear hypothesis class with parameter $\Phi$ and if the rank of $\bar{\rho}$ is at least one, then ERM is asymptotically biased, i.e., even with infinite data ERM will not achieve the desired solution $\tilde{\mathcal{S}}^{\mathsf{T}}\gamma$, except over a set of measure zero of probability distributions $\bar{\pi}$. 
\end{proposition}
\begin{proof}
% We elaborate on why ERM framework will generally lead to biased solutions (asymptotic solution is not equal to the desired OOD model) for models with confounders or anti-causal variables (such as the one in Assumption \ref{assm6: linear_model}).

Consider the case when ERM has access to infinite data, i.e., we are solving the expected risk minimization problem  stated as $\min_{\Phi\in \mathcal{H}_{\Phi}}R(\Phi)$. We will consider the linear model in Assumption \ref{assm6: linear_model} and assume $\mathcal{H}_{\Phi}$ linear hypothesis class parametrized by $\Phi\in \mathbb{R}^{n}$. We simplify the $\nabla_{\Phi}R(\Phi)$ for the square loss below

\begin{equation}
\begin{split}
    \nabla_{\Phi}R(\Phi) = \sum_{e\in \mathcal{E}_{tr}}\pi^{e}\mathbb{E}^{e}\Big[\big(Y^{e}-\Phi^{\mathsf{T}}X^{e}\big)X^{e}\Big]
\end{split}
\label{erm_bias_eqn1}
\end{equation}
We compute the gradient for $\Phi= \tilde{\mathcal{S}}^{\mathsf{T}}\gamma$ as
\begin{equation}
    \begin{split}
        & \nabla_{\Phi|\Phi=\tilde{\mathcal{S}}^{\mathsf{T}}\gamma}R(\Phi) =  \sum_{e\in \mathcal{E}_{tr}}\pi^{e}\mathbb{E}^{e}\Big[\big(Y^{e}-\gamma^{\mathsf{T}}\tilde{\mathcal{S}}X^{e}\big)X^{e}\Big]\; \text{(Use Assumption \ref{assm6: linear_model})} \\ 
        & = \sum_{e\in \mathcal{E}_{tr}}\pi^{e}\mathbb{E}^{e}\big[\varepsilon^{e}X^{e}\big] 
    \end{split}
    \label{erm_bias_eqn2}
\end{equation}
Recall $\rho^{e} = \mathbb{E}^{e}\big[\varepsilon^{e}X^{e}\big]$ and $\bar{\rho} = [\rho^{1},\dots, \rho^{|\mathcal{E}_{tr}|}]$. Recall $\bar{\pi} = [\pi^{1}, \dots, \pi^{|\mathcal{E}_{tr}|}]$. Setting the gradient defined in \eqref{erm_bias_eqn2} to zero and using the above matrix notation we get
\begin{equation}
    \bar{\rho}\bar{\pi} = 0, \mathbf{1}^{\mathsf{T}}\bar{\pi}=1, \bar{\pi}\geq 0
    \label{erm_bias_eqn3}
\end{equation}
 If $\bar{\pi}$ satisfies \eqref{erm_bias_eqn3}, then ERM is unbiased, else it is not.  Consider the set of vectors in the probability simplex $\{\bar{\pi}\;|\;\mathbf{1}^{\mathsf{T}}\bar{\pi}=1, \bar{\pi}\geq 0\}$ and define a uniform probability distribution over it. Since rank of $\bar{\rho}>0$ at least one of the columns of $\bar{\rho}$ is non-zero. As a result a uniform random draw from this set of probablity distributions would have zero probability of satisfying $\bar{\rho}\bar{\pi} = 0$. Therefore, $\tilde{\mathcal{S}}^{\mathsf{T}}\gamma$ is not the optimal solution to ERM and thus the solution of ERM would be biased away from $\tilde{\mathcal{S}}^{\mathsf{T}}\gamma$.
 
 \end{proof}

 In Proposition \ref{prop5: irm_cfd_chd}, we had assumed Assumptions \ref{assm6: linear_model}, \ref{assm8:regularity conditions on envmts and hphi}, \ref{assumption_singular_value} hold. If we also assume that rank of $\bar{\rho}$ is at least one, the Proposition \ref{prop5: irm_cfd_chd} continues to hold.  If for at least one $e \in \mathcal{E}_{tr}$, $\mathbb{E}^{e}\big[\varepsilon^{e}X^{e}\big]$ is non-zero, then the rank of $\bar{\rho}$ is at least one. From proof of Theorem 10 in \cite{arjovsky2019invariant}, linear general position continues to hold  except over a set of covariance matrices with measure zero even when one of the $\mathbb{E}^{e}\big[\varepsilon^{e}X^{e}\big]$ is non-zero. 
 
 Also, in the above Proposition \ref{erm_bias}, we only required that rank of $\bar{\rho}$ is at least one. However, if we make the additional assumptions  \ref{assm6: linear_model},  \ref{assm8:regularity conditions on envmts and hphi}, \ref{assumption_singular_value}, the result of the above Proposition continues to hold. Therefore, if Assumption  \ref{assm6: linear_model}, \ref{assm8:regularity conditions on envmts and hphi}, \ref{assumption_singular_value} hold and rank of $\bar{\rho}$ is at least one, then ERM is asymptotically biased and IRM can be within $\sqrt{\epsilon}$ neighborhood of the ideal solution with the sample complexity shown in Proposition \ref{prop5: irm_cfd_chd}.

 \bibliographystyle{apalike}
\bibliography{erm_irm_arxiv_june4}

\end{document}